\documentclass[11pt]{article}
\usepackage{bbm}
\usepackage{enumerate}
\usepackage[OT1]{fontenc}
\usepackage[authoryear, round, colon]{natbib}
\usepackage{amsmath, amsfonts,amsthm,amssymb,xspace,bm, verbatim,dsfont}
\usepackage[usenames]{color}
\usepackage{multirow}
\usepackage{mathrsfs}
\usepackage{fullpage}
\usepackage{algorithm}
\usepackage{algorithmic}
\usepackage{graphicx}
\usepackage{subfigure}
\usepackage[colorlinks,
			linkcolor=blue,
			anchorcolor=blue,
			citecolor=blue
			]{hyperref}

\long\def\comment#1{}

\newcommand{\bel}{\begin{eqnarray}\label}
	\newcommand{\eel}{\end{eqnarray}}
\newcommand{\bes}{\begin{eqnarray*}}
	\newcommand{\ees}{\end{eqnarray*}}

\def\##1\#{\begin{align}#1\end{align}}
\def\$#1\${\begin{align*}#1\end{align*}}

\let\tilde\widetilde

\newcommand{\ab}{\mathbf{a}}

\newcommand{\ub}{\mathbf{u}}
\newcommand{\vb}{\mathbf{v}}

\newcommand{\xb}{\mathbf{x}}
\newcommand{\yb}{\mathbf{y}}
\newcommand{\zb}{\mathbf{z}}

\newcommand{\Ib}{\mathbf{I}}

\newcommand{\Mb}{\mathbf{M}}
\newcommand{\Nb}{\mathbf{N}}

\newcommand{\Ub}{\mathbf{U}}
\newcommand{\Vb}{\mathbf{V}}

\newcommand{\Xb}{\mathbf{X}}

\newcommand{\cB}{\mathcal{B}}
\newcommand{\cC}{\mathcal{C}}
\newcommand{\cD}{\mathcal{D}}
\newcommand{\cE}{\mathcal{E}}

\newcommand{\cM}{\mathcal{M}}
\newcommand{\cN}{\mathcal{N}}

\newcommand{\cR}{\mathcal{R}}
\newcommand{\cS}{{\mathcal{S}}}

\newcommand{\RR}{\mathbb{R}}
\newcommand{\SSS}{\mathbb{S}}

\newcommand{\bbeta}{\bm{\beta}}

\newcommand{\bzeta}{\bm{\zeta}}

\newcommand{\bmu}{\bm{\mu}}

\newcommand{\bGamma}{\bm{\Gamma}}

\newcommand{\bSigma}{\bm{\Sigma}}

\newcommand{\Gammab}{\mathbf{\Gamma}}

\newcommand{\Sigmab}{\mathbf{\Sigma}}

\newcommand{\argmax}{\mathop{\mathrm{argmax}}}

\newcommand{\rank}{\mathop{\mathrm{rank}}}
\newcommand{\supp}{\mathop{\mathrm{supp}}}

\def \ddd{{\rm d}}
\def\supp{\mathop{\text{supp}}}
\def\rank{\mathrm{rank}}

\global\long\def\Y{\bm{Y}}

\global\long\def\Z{\bm{Z}}

\global\long\def\x{\bm{x}}

\global\long\def\y{\bm{y}}

\global\long\def\e{\bm{e}}

\global\long\def\u{\bm{u}}
\global\long\def\v{\bm{v}}

\global\long\def\bbeta{\bm{\beta}}

\newtheoremstyle{mytheoremstyle} 
{\topsep}                    
{\topsep}                    
{\normalfont}                   
{}                           
{\bfseries}                   
{.}                          
{.5em}                       
{}  

\theoremstyle{mytheoremstyle}

\ifx\BlackBox\undefined
\newcommand{\BlackBox}{\rule{1.5ex}{1.5ex}}  
\fi

\ifx\QED\undefined
\def\QED{~\rule[-1pt]{5pt}{5pt}\par\medskip}
\fi

\ifx\proof\undefined
\newenvironment{proof}{\par\noindent{\bf Proof\ }}{\hfill\BlackBox\\[2mm]}
\fi

\ifx\theorem\undefined
\newtheorem{theorem}{Theorem}
\fi
\ifx\example\undefined
\newtheorem{example}[theorem]{Example}
\fi
\ifx\property\undefined
\newtheorem{property}{Property}
\fi
\ifx\lemma\undefined
\newtheorem{lemma}[theorem]{Lemma}
\fi
\ifx\proposition\undefined
\newtheorem{proposition}[theorem]{Proposition}
\fi
\ifx\remark\undefined
\newtheorem{remark}[theorem]{Remark}
\fi
\ifx\corollary\undefined
\newtheorem{corollary}[theorem]{Corollary}
\fi
\ifx\definition\undefined
\newtheorem{definition}[theorem]{Definition}
\fi
\ifx\conjecture\undefined
\newtheorem{conjecture}[theorem]{Conjecture}
\fi
\ifx\fact\undefined
\newtheorem{fact}[theorem]{Fact}
\fi
\ifx\claim\undefined
\newtheorem{claim}[theorem]{Claim}
\fi
\ifx\assumption\undefined
\newtheorem{assumption}[theorem]{Assumption}
\fi
\ifx\condition\undefined
\newtheorem{condition}{Condition}
\fi
\numberwithin{equation}{section}
\numberwithin{theorem}{section}

\begin{document}

\title{\huge Regularized EM Algorithms: A Unified Framework and Statistical Guarantees}

\author{
Xinyang Yi\\
{The University of Texas at Austin}\\
{yixy@utexas.edu}
\and
Constantine Caramanis\\
{The University of Texas at Austin}\\
{constantine@utexas.edu} }
\date{}

\maketitle

\begin{abstract}
 Latent variable models are a fundamental modeling tool in machine learning applications, but they present significant computational and analytical challenges. The popular EM algorithm and its variants, is a much used algorithmic tool; yet our rigorous understanding of its performance is highly incomplete. Recently, work in \cite{balakrishnan2014statistical} has demonstrated that for an important class of problems, EM exhibits linear local convergence. In the high-dimensional setting, however, the $M$-step may not be well defined. We address precisely this setting through a unified treatment using regularization. While regularization for high-dimensional problems is by now well understood, the iterative EM algorithm requires a careful balancing of making progress towards the solution while identifying the right structure (e.g., sparsity or low-rank). In particular, regularizing the $M$-step using the state-of-the-art high-dimensional prescriptions (e.g., \`a la \cite{wainwright2014structured}) is not guaranteed to provide this balance. Our algorithm and analysis are linked in a way that reveals the balance between optimization and statistical errors. We specialize our general framework to sparse gaussian mixture models, high-dimensional mixed regression, and regression with missing variables, obtaining statistical guarantees for each of these examples.
\end{abstract}

\section{Introduction}
In this paper, we give general conditions and an analytical framework for the convergence of the EM method for high-dimensional parameter estimation in latent variable models. We specialize these conditions to several problems of interest, including high-dimensional sparse and low-rank mixed regression, sparse gaussian mixture models, and regression with missing covariates. As we explain below, the key problem in the high-dimensional setting is the $M$-step. A natural idea is to modify this step via appropriate regularization, yet choosing the appropriate sequence of regularizers is a critical problem. As we know from the theory of regularized M-estimators (e.g., \cite{wainwright2014structured}) the regularizer should be chosen proportional to the target estimation error. For EM, however, the target estimation error changes at each step.

The main contribution of our work is technical: we show how to perform this iterative regularization. We show that the regularization sequence must be chosen so that it converges to a quantity controlled by the ultimate estimation error. In existing work, the estimation error is given by the relationship between the population and empirical $M$-step operators, but the $M$-operator is not well defined in the high-dimensional setting. Thus a key step, related both to our algorithm and its convergence analysis, is obtaining a different characterization of statistical error for the high-dimensional setting.

\subsection*{Background and Related Work}
EM (e.g., \cite{dempster1977maximum, mclachlan2007algorithm}) is a general algorithmic approach for handling latent variable models (including mixtures), popular largely because it is typically computationally highly scalable, and easy to implement. On the flip side, despite a fairly long history of studying EM in theory (e.g., \cite{wu1983convergence, tseng2004analysis, mclachlan2007algorithm}), very little has been understood about general statistical guarantees until recently. Very recent work in \cite{balakrishnan2014statistical} establishes a general local convergence theorem (i.e., assuming initialization lies in a local region around true parameter) and statistical guarantees for EM, which is then specialized to obtain near-optimal rates for several specific {\em low-dimensional} problems -- low-dimensional in the sense of the classical statistical setting where the samples outnumber the dimension. A central challenge in extending EM (and as a corollary, the analysis in \cite{balakrishnan2014statistical}) to the high-dimensional regime is the $M$-step. On the algorithm side, the $M$-step will not be stable (or even well-defined in some cases) in the high-dimensional setting. To make matters worse, any analysis that relies on showing that the finite-sample $M$-step is somehow ``close'' to the $M$-step performed with infinite data (the population-level $M$-step) simply cannot apply in the high-dimensional regime. Recent work in \cite{wang2014high} treats high-dimensional EM using a truncated $M$-step. This works in some settings, but also requires specialized treatment for every different setting, precisely because of the difficulty with the $M$-step.

In contrast to work in \cite{wang2014high}, we pursue a high-dimensional extension via regularization. The central challenge, as mentioned above, is in picking the sequence of regularization coefficients, as this must control the optimization error (related to the special structure of $\bbeta^{\ast}$), as well as the statistical error. Finally, we note that for finite mixture regression, St{\"a}dler et al.\cite{stadler2010ℓ} consider an $\ell_1$ regularized EM algorithm for which they develop some asymptotic analysis and oracle inequality. However, this work doesn't establish the theoretical properties of local optima arising from regularized EM. Our work addresses this issue from a local convergence perspective by using a novel choice of regularization.

\noindent {\bf Notation:} Let $\ub = (u_1,u_2,\ldots,u_p)^{\top} \in \RR^p$ be a vector and $\Mb = [M_{i,j}] \in \RR^{p_1\times p_2}$ be a matrix. The $\ell_q$ norm of $\ub$ is defined as $\|\ub\|_p = (\sum_{i=1}^p |u_i|^q)^{1/q}$. We use $\|\Mb\|_{*}$ to denote the nuclear norm of $\Mb$ and $\|\Mb\|_{2}$ to denote its spectral norm.  We use $\odot$ to denote the Hadamard product between two vectors, i.e., $\ub\odot\vb = (u_1v_1,u_2v_2,\ldots,u_pv_p)^{\top}$. A $p$-by-$p$ identity matrix is denoted as $\Ib_p$. We use capital letter (e.g., $X$) to denote random variable, vector and matrix. For a sub-Gaussian (sub-exponential) random variable $X$, we use $\|X\|_{\psi_2}$ ($\|X\|_{\psi_1}$) to denote its Orlicz norm (see \citet{vershynin2010introduction} for detailed definitions). For two functions $f(n)$ and $g(n)$, we use $f(n) \lesssim g(n)$ to represent $f(n) \leq Cg(n)$ for some absolute constant $C > 0$. In parallel, we use $f(n) \gtrsim g(n)$ to represent $f(n) \geq C'g(n)$ for some absolute constant $C' > 0$. For any differentiable function $f: \RR^p \rightarrow \RR$, we use $\nabla f$ to denote its gradient.

The rest of our paper is organized as follows. We present our regularized EM algorithm, including the precise sequence of regularization, and discuss its applications to several example models in Section \ref{sec:algorithm}. The specific examples to which we show our results apply, are sparse gaussian mixture models, sparse or low-rank mixed regression, and regression with missing covariates. In section \ref{sec:theory}, we establish our analytical framework and show the main theory, i.e., computational and statistical guarantees of the regularized EM algorithm. Then, by applying our main theory, we establish several near optimal statistical rate of those aforementioned models in section \ref{sec:app}. Section \ref{sec:sim} demonstrates our results through numerical examples. We outline the proof of our main result in section \ref{proof:thm:main}. The detailed proofs of other results and multiple technical lemmas are deferred to the appendix.

\section{Regularized EM Algorithm}
\label{sec:algorithm}
In this section, we first present a general regularized EM algorithm in which a convex regularizer is used to enforce certain type of structure. Then we turn to revisit three well known latent variable models and show how the proposed algorithm can be applied to high dimensional parameter estimation in these models.
\subsection{Algorithm}

Before introducing our approach, we first review the classic EM algorithm. Let $\Y, \Z$ be random variables taking values in $\mathcal{Y}$,$\mathcal{Z}$. Suppose they have join distribution 
\[
f_{\bbeta}(\yb,\zb)
\]
depending on model parameter $\bbeta \subseteq \Omega$ where $\Omega$ is some parameter space in $\RR^{p}$. In latent variable models, it is common to assume we can only obtain samples from $\Y$ while $\Z$, called latent variable, can not be observed. Consider the marginal distribution of $\Y$ as 
\[
y_{\bbeta}(\yb) := \int_{\mathcal{Z}} f_{\bbeta}(\yb,\zb)d\zb.
\]
Given $n$ i.i.d. observations $\yb_1,\yb_2,\ldots,\yb_n$ of $\Y$, our goal is to estimate the model parameter $\bbeta$. We consider the maximum likelihood estimation: compute $\widehat{\bbeta} \in \Omega$ that maximizes the log likelihood function, namely,
\begin{equation}
\label{maximum_likelihood}
\widehat{\bbeta} = \arg \max_{\bbeta \in \Omega} h(\bbeta;\yb_1^n),
\end{equation}
where 
\[
h(\bbeta;\yb_1^n) := \frac{1}{n}\sum_{i=1}^{n}\log{y_{\bbeta}(\yb_i)}.
\]
In many settings, the objective function in \eqref{maximum_likelihood} is highly nonconvex, thereby it's computationally inefficient to solve it directly. Instead, we turn to a lower bound of $h(\bbeta;\y_1^n)$ which is more friendly to evaluate and optimize. Let $\kappa_{\beta}(\zb|\yb)$ denote the conditional distribution of $\Z$ given $\Y = \yb$. For any $\bbeta' \in \Omega$, we have
\begin{align}
\label{EMlowerbound}
h(\bbeta';\yb_1^n) & = \frac{1}{n}\sum_{i=1}^{n}\log{y_{\bbeta'}(\y_i)} = \frac{1}{n}\sum_{i=1}^n\log \int_{\mathcal{Z}}  f_{\bbeta'}(\yb_i,\zb) d\zb  \notag\\
& =  \frac{1}{n}\sum_{i=1}^n \log \int_{\mathcal{Z}}  \kappa_{\bbeta}(\zb|\yb_i) \frac{f_{\bbeta'}(\yb_i,\zb)}{\kappa_{\bbeta}(\zb|\yb_i)} d\zb  \overset{(a)}{\geq}  \frac{1}{n}\sum_{i=1}^n \int_{\mathcal{Z}} \kappa_{\bbeta}(\zb|\yb_i) \log    \frac{f_{\bbeta'}(\yb_i,\zb)}{\kappa_{\bbeta}(\zb|\yb_i)} d\zb  \notag\\
& =  \frac{1}{n}\sum_{i=1}^n \int_{\mathcal{Z}} \kappa_{\bbeta}(\zb|\yb_i) \log f_{\bbeta'} (\yb_i,\zb)d\zb - \int_{\mathcal{Z}} \kappa_{\bbeta}(\zb|\yb_i) \log \kappa_{\bbeta}(\zb|\yb_i) d\zb,
\end{align}
where $(a)$ follows from Jensen's inequality. The key idea of EM algorithm is to perform iterative maximization of the obtained lower bound \eqref{EMlowerbound}. We denote the first term in \eqref{EMlowerbound} as function $Q_n(\cdot|\cdot)$, i.e.,
\begin{equation}
\label{Q_sample}
Q_n(\bbeta'|\bbeta) := \frac{1}{n}\sum_{i=1}^n \int_{\mathcal{Z}} \kappa_{\bbeta}(\zb|\yb_i) \log f_{\bbeta'}(\yb_i,\zb)d\zb.
\end{equation}
One iteration of EM algorithm, mapping $\bbeta^{(t)}$ to $\bbeta^{(t+1)}$, consists of the following two steps:
\begin{itemize}
\item E-step: Compute function $Q_n(\bbeta|\bbeta^{(t)})$ given $\bbeta^{(t)}$.
\item M-step: $\bbeta^{(t+1)} \leftarrow \arg\max_{\bbeta \in \Omega}Q_n(\bbeta|\bbeta^{(t)})$.
\end{itemize}
It's convenient to introduce mapping $\cM_n:\Omega \rightarrow \Omega$ to denote the above algorithm
\begin{equation}
\label{M_sample}
\cM_n(\bbeta) := \arg \max_{\bbeta' \in \Omega} Q_n(\bbeta' | \bbeta).
\end{equation}
When $n \rightarrow \infty$, we define the population level $Q(\cdot | \cdot)$ function as 
\begin{equation}
\label{Q_population}
Q(\bbeta'|\bbeta) :=\int_{\mathcal{Y}}  y_{\bbeta^*}(\yb) \int_{\mathcal{Z}} \kappa_{\bbeta}(\zb|\yb) \log f_{\bbeta'} (\yb,\zb)d\zb d\yb.
\end{equation}
Similar to \eqref{M_sample}, we define the population level mapping $\cM:\Omega \rightarrow \Omega$ as 
\begin{equation}
\label{M_population}
\cM(\bbeta) = \arg \max_{\bbeta' \in \Omega} Q(\bbeta' | \bbeta).
\end{equation}

\begin{algorithm}[!htb]
\caption{ High Dimensional Regularized EM Algorithm}
\label{alg}
\begin{algorithmic}[1]
\INPUT  Samples $\{\yb_i\}_{i=1}^{n}$, regularizer $\cR$, number of iterations $T$, initial parameter $\bbeta^{(0)}$, initial regularization parameter $\lambda_n^{(0)}$, estimated statistical error $\Delta$, contractive factor $\kappa < 1$.
\FOR{$t = 1, 2, \ldots, T$}
\STATE {\bf Regularization parameter update:}
\begin{equation}
\label{lambda_update}
\lambda_n^{(t)}  \leftarrow \kappa\lambda_n^{(t-1)} + \Delta.
\end{equation}
\STATE {\bf E-step}: Compute function $Q_n(\cdot|\bbeta^{(t-1)})$ according to \eqref{Q_sample}.
\STATE {\bf Regularized M-step:} 
\[
\bbeta^{(t)} \leftarrow \arg \max_{\bbeta \in \Omega} Q_n(\bbeta|\bbeta^{(t-1)}) - \lambda_n^{(t)}\cdot \cR(\bbeta).
\]
\ENDFOR
\OUTPUT $\bbeta^{(T)}$.
\end{algorithmic}
\end{algorithm}

Generally, the classic EM procedure is not applicable to high dimensional regime where $n \ll p$: First, with insufficient number of samples,  $\cM_n(\bbeta)$ is usually far way from $\cM(\bbeta)$. In this case, even if the initial parameter is $\bbeta^*$, $\cM_n(\bbeta^*)$ is not a meaningful estimation of $\bbeta^*$. As an example, in Gaussian mixture models, the minimum estimation error $\|\cM_n(\bbeta^*) - \cM(\bbeta^*)\|$ can be much larger than signal strength $\|\bbeta^*\|$. Second, in some models, $\cM_n(\bbeta)$ is not even well defined. For instance, in mixture linear regression, solving \eqref{M_sample} involves inverting sample covariance matrix that is not full rank when $n < p$. (See Section \ref{sec:mlr} for detailed discussion.)

We now turn to our regularized EM algorithm that is designed to overcome the aforementioned high dimensionality challenges. In particular, we propose to replace the M-step with regularized maximization step. In detail, for some convex regularizer $\cR: \Omega \rightarrow \RR^{+}$ and user specified regularization regularization parameter $\lambda_n$, our regularized M-step is defined as:
\begin{equation}
\label{M_regularizer}
\cM_n^r(\bbeta) := \arg \max_{\bbeta' \in \Omega} Q_n(\bbeta'|\bbeta) - \lambda_n\cR(\bbeta').
\end{equation}
We present the details of our algorithm in Algorithm \ref{alg}. The role of $\cR$ is to enforce the solution to have a certain structure of the model parameter $\bbeta^*$. 

The choice of regularization parameter $\lambda_n^{(t)}$ plays an important role in controlling statistical and optimization error. As stated in \eqref{lambda_update}, the update of $\lambda_n^{(t)}$ involves a linear combination of old parameter $\lambda_n^{(t-1)}$ and the quantity $\Delta$. Then $\lambda_n^{(t)}$ takes the form
\[
\lambda_n^{(t)} = \kappa^t\lambda_n^{(0)} + \frac{1-\kappa^t}{1-\kappa} \Delta.
\]

\begin{figure}[ht]
	\centering
	\includegraphics[width=0.4\columnwidth]{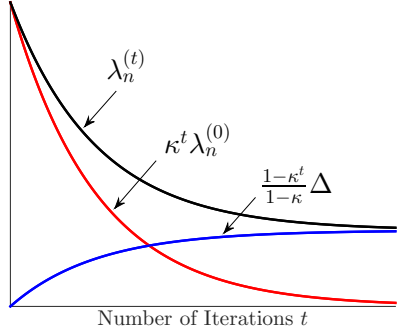}
	\caption{Illustration of regularization parameter update. We have that $\lambda_n^{(t)}$, represented by the black line, is the summation of $\kappa^t\lambda_n^{(0)}$ and $\frac{1-\kappa^t}{1-\kappa}\Delta$ represented by the red and blue lines respectively. }
	\label{fig:choice of lambda}
\end{figure}

As shown in Figure \ref{fig:choice of lambda}, $\lambda_n^{(t)}$ first decays geometrically from $\lambda_n^{(0)}$ and then it gradually approaches $\frac{1}{1-\kappa}\cdot\Delta$. Quantity $\Delta$ characterizes the target statistical error which depends on number of samples $n$, data dimension $p$ and some factor associated with concrete models. Usually, we have $\Delta = O(\sqrt{\log p/n})$, which vanishes when $n,p$ increase with a fixed ratio. To provide some intuitions of such choice, we first note that from theory of high dimensional regularized M-estimator \citet{wainwright2014structured}, suitable $\lambda_n$ should be proportional to the target estimation error. Analogous to our setting, we let $\lambda_n^{(t)}$ be proportional to $\|\cM_n^r(\bbeta^{(t)}) - \bbeta^*\|_2$ which is the estimation error in step $t$. Consider the following triangle inequality 
\[\|\cM_n^r(\bbeta^{(t)}) - \bbeta^*\|_2 \leq \|\cM_n^r(\bbeta^{(t)}) - \cM_n^r(\bbeta^*)\|_2 + \|\cM_n^r(\bbeta^*) - \bbeta^*\|_2.
\] 
Note that the second term $\|\cM_n^r(\bbeta^*)-\bbeta^*\|_2$ corresponds to quantity $\Delta$ since they both have the sense of final estimation error. The first term $\|\cM_n^r(\bbeta^{(t)}) - \cM_n^r(\bbeta^*)\|_2$, resulting from optimization error $\|\bbeta^{(t)} - \bbeta^*\|_2$, then corresponds to $\kappa\lambda_n^{(t-1)}$ in \eqref{lambda_update}. By setting $t = 1$, we observe that $\lambda_n^{(0)}$ is proportional to the initialization error. Consequently, we have $\lambda_n^{t} \geq \kappa\lambda^{(t-1)} + \Delta$. Inspired by the low-dimensional analysis of EM in \citet{balakrishnan2014statistical}, we expect the optimization error to decay geometrically, so we choose $\kappa \in (0,1)$. Beyond the intuition, we provide the rigorous analysis and detailed parameter update in Section \ref{sec:theory}.

\subsection{Example Models}
\label{sec:example models}

Now we introduce three well known latent variable models. For each model, we review the specific formulations of standard EM algorithm, discuss the extensions in high dimensional setting, and provide the implementations of high dimensional regularized EM iterations.  
\subsubsection{Gaussian Mixture Model}

We consider the balanced isotropic Gaussian mixture model (GMM) with two components where the distribution of random variables $(Y,Z) \in \RR^p\times\{-1,1\}$ is determined by
\[
\Pr\left(Y = \yb|Z = z\right) = \phi(\yb; z\cdot\bbeta^*, \sigma^2\Ib_p)
\]
and
\[
\Pr(Z = 1) = \Pr(Z = -1) = 1/2.
\]
Here we use $\phi(\cdot|\bmu,\Sigmab)$ to denote probability density function of $\cN(\bmu,\Sigmab)$. In this example, $Z$ is latent variable that indicates the cluster id of each sample. In this example, given $n$ i.i.d. samples $\{\yb_i\}_{i=1}^{n}$, function $Q_n(\cdot|\cdot)$ defined in \eqref{Q_sample} corresponds to 
\begin{equation}
\label{gaussian_mixutre_Qn}
Q_n^{GMM}(\bbeta'|\bbeta) = -\frac{1}{2n}\sum_{i=1}^n\left[w(\yb_i;\bbeta)\|\yb_i - \bbeta'\|_2^2 + (1 - w(\yb_i;\bbeta))\|\yb_i + \bbeta'\|_2^2\right],
\end{equation}
where 
\begin{equation}
\label{gaussian_mixutre_w}
w(\yb;\bbeta) := \frac{\exp{(-\frac{\|\yb - \bbeta\|_2^2}{2\sigma^2})}}{\exp{(-\frac{\|\yb - \bbeta\|_2^2}{2\sigma^2})} + \exp{(-\frac{\|\yb + \bbeta\|_2^2}{2\sigma^2})} }.
\end{equation}
Then we have that the standard EM update \eqref{M_sample} corresponds to 
\begin{equation}
\label{gaussian_mixture_Mn}
\cM_{n}(\bbeta) = \frac{2}{n}\sum_{i=1}^{n}w(\yb_i;\bbeta)\yb_i - \frac{1}{n}\sum_{i=1}^n \yb_i.
\end{equation}

In high dimensional regime, we assume $\bbeta^*$ is sparse. Formally, let $\cB_0(s;p) := \{\ub \in \RR^p: |\supp(\ub)| \leq s\}$, we have $\bbeta^* \in \cB_0(s;p)$. Naturally, we choose regularizer $\cR$ to be $\ell_1$ norm in order to recover the sparse structure. Consequently, our regularized EM iteration \eqref{M_regularizer} corresponds to 
\[
\cM_n^r(\bbeta) = \arg\max_{\bbeta' \in \RR^p} Q_n^{GMM}(\bbeta'|\bbeta)- \lambda_n\|\bbeta'\|_1.
\]
\subsubsection{Mixed Linear Regression} \label{sec:mlr}

Mixed linear regression (MLR), as considered in some recent work \citep{chaganty2013spectral, yi2013alternating, chen2014convex}, is the problem of recovering two or more linear vectors from mixed linear measurements. In the case of mixed linear regression with two symmetric and balanced components, response-covariate pair $(Y,X) \in \RR\times\RR^p$ is linked through
\[
Y = \langle X,~ Z\cdot \bbeta^*\rangle + W,
\]
where $W$ is noise term and $Z$ is latent variable that has Rademacher distribution over $\{-1,1\}$. We assume $X \sim \cN(0,\Ib_p)$, $W \sim \cN(0,\sigma^2)$. In this setting, with $n$ i.i.d. samples $\{y_i,\xb_i\}_{i=1}^n$ of pair $(Y,X)$, function $Q_n(\cdot|\cdot)$ then corresponds to 
\begin{equation} \label{mlr_Qn}
Q_n^{MLR}(\bbeta'|\bbeta) = -\frac{1}{2n}\sum_{i=1}^n\left[w(y_i,\xb_i;\bbeta)(y_i - \langle\xb_i,\bbeta'\rangle)^2  + (1 - w(y_i,\xb_i;\bbeta))(y_i + \langle\xb_i,\bbeta'\rangle)^2\right],
\end{equation}
where $w(y,\xb;\bbeta)$ is defined as 
\[
w(y,\xb;\bbeta) := \frac{\exp{(-\frac{(y - \langle\xb,\bbeta\rangle)^2}{2\sigma^2})}}{\exp{(-\frac{(y - \langle\xb,\bbeta\rangle)^2}{2\sigma^2})} + \exp{(-\frac{(y + \langle\xb,\bbeta\rangle)^2}{2\sigma^2})}}.
\]
 
The standard EM iteration  \eqref{M_sample}  corresponds to 
\begin{equation}
\label{eq:mlr:M_sample}
\cM_n(\bbeta) = \left(\sum_{i=1}^n \xb_i\xb_i^{\top}\right)^{-1}\left(\sum_{i=1}^n(2w(y_i,\xb_i;\bbeta)-1)y_i\xb_i\right).
\end{equation}

Note that \eqref{eq:mlr:M_sample} involves inverting sample covariance matrix. Therefore, in high dimensional setting $\cM_n(\bbeta)$ is not well defined since sample covariance matrix has rank much smaller than the ambient dimension. As discussed earlier, characterizing statistical error in terms of $\cM_n(\bbeta) - \cM(\bbeta)$ is not well suited to this case. 

Next we consider two kinds of structure about $\bbeta^*$ in order to deal with high dimensionality. First we assume $\bbeta^*$ is an $s$-sparse vector, i.e., $\bbeta^* \in \cB_0(s;p)$. Then by using $\ell_1$ regularizer, \eqref{M_regularizer} corresponds to 
\begin{align*}
 \cM_n^r(\bbeta)  = \arg\max_{\bbeta' \in \RR^p} Q_n^{MLR}(\bbeta'|\bbeta) - \lambda_n\|\bbeta'\|_1.
\end{align*}
Second we consider that the model parameter is a matrix $\Gammab^* \in \RR^{p_1\times p_2}$ with $\rank(\Gammab^*) = \theta \ll \min(p_1,p_2)$. We further assume $X \in \RR^{p_1\times p_2}$ is an i.i.d. Gaussian matrix, i.e., entries of $X$ are independent random variables with distribution $\cN(0,1)$. Note that in low dimensional case $n \gg p_1\times p_2$, there is no essential difference between assuming parameter is vector and matrix since we can always treat $X$ and $\Gammab^*$ as $(p_1 \times p_2)$-dimensional vectors. In high dimensional regime, low rank structure leads to different regularization. We choose $\cR$ to be nuclear norm to serve such structure. Consequently, given $n$ samples with form $\{y_i,\Xb_i\}_{i=1}^{n}$,  \eqref{M_regularizer} then  corresponds to 
\begin{align} \label{mlr:implementation:low rank}
 \cM_n^r(\Gammab) \notag = \arg\max_{\Gammab' \in \RR^{p_1\times p_2}} & -\frac{1}{2n}\sum_{i=1}^n\big[w( y_i,\Xb_i;\Gammab)(y_i - \langle\Xb_i,\Gammab'\rangle)^2  + \\ & (1 - w(y_i,\Xb_i;\Gammab))(y_i + \langle\Xb_i,\Gammab'\rangle)^2\big]- \lambda_n\|\Gammab'\|_*.
\end{align}
The standard low rank matrix recovery with a single component, including other sensing matrix designs beyond Gaussian matrix, has been studied extensively (e.g.,\citet{candes2009exact, recht2010guaranteed, candes2011tight, negahban2011estimation, jain2013low, chen2013exact, cai2015rop}). To the best of our knowledge, theoretical study of the mixed low rank matrix recover has not been considered in existing literature.

\subsubsection{Missing Covariate Regression}
As our last example, we consider the missing covariate regression (MCR) problem. To be the same as standard linear regression, $\{y_i,\xb_i\}_{i=1}^n$ are samples of $(Y,X)$ linked through $Y = \langle X,\bbeta^*\rangle + W$. However, we assume each entry of $\xb_i$ is missing independently with probability $\epsilon \in (0,1)$. Therefore, the observed covariate $\widetilde{\xb}_i$ takes the form
\[
\widetilde{x}_{i,j} = \begin{cases} x_{i,j} & \text{with probability}\; 1 - \epsilon \\ * & \text{otherwise} \end{cases}.
\] 
To ease notation, we introduce vector $\zb_i \in \{0,1\}^p$ to indicate the positions of missing entries, i.e., $z_{i,j} = 1$ if $x_{i,j}$ is missing. In this example, the E step involves computing the distribution of missing entries given current parameter guess $\bbeta$. Under Gaussian design $X \sim \cN(\bm{0},\Ib_p), W \sim \cN(0,\sigma^2)$, given observed covariate entries $(\bm{1} - \zb_i)\odot\xb_i$ and $y_i$, the conditional mean vector of $\widetilde{\xb}_i$ has form  
\begin{equation} \label{mcr:mean}
\bmu_{\bbeta}(y_i,\zb_i,\xb_i) := \mathbb{E}[\tilde{\xb_i}\big|\bbeta, y_i, (\bm{1} - \zb_i)\odot\xb_i] = (\bm{1} - \zb_i)\odot \xb_i + \frac{y_i - \langle\bbeta, (\bm{1}-\zb_i)\odot \xb_i \rangle}{\sigma^2 + \|\zb_i \odot \bbeta\|_2^2} \zb_i \odot\bbeta,
\end{equation}
and the conditional correlation matrix of $\tilde{\xb}_i$ has form
\begin{align} \label{mcr:covariance} 
& \Sigmab_{\bbeta}(y_i,\zb_i,\xb_i) := \mathbb{E}\left[\tilde{\xb}_i\tilde{\xb}_i^{\top}\big|\bbeta, y_i, (\bm{1} - \zb_i)\odot\xb_i\right] \notag\\ & = \bmu_{\bbeta}\bmu_{\bbeta}^{\top} + {\rm diag}(\zb_i) - \left( \frac{1}{\sigma^2+\|\zb_i\odot \bbeta\|_2^2}\right)(\zb_i\odot\bbeta)(\zb_i\odot\bbeta)^{\top}.
\end{align}
Consequently, $Q_n(\cdot|\cdot)$ corresponds to
\begin{equation} \label{mcr_Qn}
Q_n^{MCR}(\bbeta'|\bbeta) = \frac{1}{n}\sum_{i=1}^n \langle y_i \bmu_{\bbeta}(y_i,\zb_i,\xb_i), \bbeta' \rangle - \frac{1}{2}\bbeta^{\top} \Sigmab_{\bbeta}(y_i,\zb_i,\xb_i)\bbeta.
\end{equation}
The standard EM update corresponds to 
\[
\cM_n(\bbeta) = \left[\sum_{i=1}^{n} \Sigmab_{\bbeta}(y_i,\zb_i,\xb_i)\right]^{-1} \sum_{i=1}^{n}y_i \bmu_{\bbeta}(y_i,\zb_i,\xb_i).
\]
Note that $\Sigmab_{\bbeta}(y_i,\zb_i,\xb_i)$ has rank at most $O(\epsilon p)$ with high probability. When $\epsilon = O(1/p)$, the empirical covariance matrix is non-invertible when $n \ll p$. We now assume $\bbeta^* \in \cB_0(s;p)$. By leveraging $\ell_1$ regularization, one step update in Algorithm \ref{alg} corresponds to
\[
\cM_n^r(\bbeta) = \argmax_{\bbeta' \in \RR^p} Q_n^{MCR}(\bbeta'|\bbeta) - \lambda_n\|\bbeta'\|_1.
\] 
\section{General Computational and Statistical Guarantees}
\label{sec:theory}
We now turn to the theoretical analysis of high dimensional regularized EM algorithm. In Section \ref{sec:framework}, we set up a general analytical framework for regularized EM where the key ingredients are decomposable regularizer and several technical conditions about population based $Q(\cdot|\cdot)$ and sample based $Q_n(\cdot|\cdot)$. 
In Section \ref{sec:main results}, we first introduce a resampling version of Algorithm \ref{alg} and provide our main result (Theorem \ref{thm:main}) that characterizes both computational and statistical performance of the proposed variant of regularized EM algorithm. 

\subsection{Framework}
\label{sec:framework}
\subsubsection{Decomposable Regularizers}

Decomposable regularizer, as considered in a body of work (e.g., \cite{candes2007dantzig, negahban2009unified, wainwright2014structured, Chen2014Graph}), has been shown to be useful, both empirically and theoretically, for high dimensional structural estimation. It also plays an important role in our analytical framework. We begin with the assumption that $\cR: \RR^p \rightarrow \RR^{+}$ is a norm, thereby we have $\cR(\ub + \vb) \leq \cR(\ub) + \cR(\vb), \;\forall\; \ub,\vb \in \RR^p$. Consider a pair of subspaces  $(\cS, \overline{\cS})$  in $\RR^p$ such that $\cS \subseteq \overline{\cS}$. We denote the subspace orthogonal to $\overline{\cS}$ with respect to inner product $\big\langle \cdot,\cdot\big\rangle$ as $\overline{\cS}^{\bot}$, namely
\[
\overline{\cS}^{\bot} := \big\{\ub \in \Omega: \big\langle \ub,\vb\big\rangle, \;\forall\; \vb \in \overline{\cS}\}.
\]
\begin{definition} \label{def:decomposability}(Decomposability) 
Regularizer $\cR:\RR^p \rightarrow \RR^{+}$ is decomposable with respect to $(\cS, \overline{\cS})$ if 
\[
\cR(\ub + \vb) = \cR(\ub) + \cR(\vb), \; \text{for any} \; \ub \in \cS, \vb \in \overline{\cS}^{\bot}.
\]
\end{definition}
Usually the structure of model parameter $\bbeta^*$ can be characterized by specifying a subspace $\cS$ such that $\bbeta^* \in \cS$. The common use of regularizer is thus to penalize the compositions of solution that live outside $\cS$. As $\cR$ is a norm, for $\vb \in \overline{\cS}^{\bot}$, we always have $\cR(\bbeta^* + \vb) \leq \cR(\bbeta^*) + \cR(\vb)$. Consequently, decomposable regularizers actually make such penalty as much as possible by achieving the upper bound. We are interested in bounding the estimation error in some norm $\|\cdot\|$. The following quantity is critical in connecting $\cR$ to $\|\cdot\|$.
\begin{definition}(Subspace Compatibility Constant)
	\label{def:subspace compatibility constant}
For any subspace $\cS \subseteq \RR^p$, a given regularizer $\cR$ and some norm $\|\cdot\|$, the subspace compatibility constant of $\cS$ with respect to $\cR, \|\cdot\|$ is given by
\[
\Psi(\cS) := \sup_{\ub \in \cS\setminus\{\bm{0}\}} \frac{\cR(\ub)}{\|\ub\|}.
\]
\end{definition}
Standardly, the dual norm of $\cR$ is defined as $
\cR^*(\vb) : = \sup_{\cR(\ub) \leq 1} \big\langle \ub, \vb\big\rangle$.
To simplify notation, we let $\|\ub\|_{\cR} := \cR(\ub)$ and $\|\ub\|_{\cR^*} := \cR^*(\ub)$.
\subsubsection{Conditions on $Q(\cdot|\cdot)$}
Next, we review three technical conditions, originally proposed by \cite{balakrishnan2014statistical}, about population level $Q(\cdot|\cdot)$ function. Recall that $\Omega \subseteq \RR^p$ is the basin of attraction. It is well known that performance of EM algorithm is sensitive to initialization. Analyzing Algorithm \ref{alg} with any initial point is not desirable in this paper. Our theory is developed with focus on a $r$-neighbor region round $\bbeta^*$ that is defined as $\cB(r;\bbeta^*) := \big\{\ub \in \Omega,  \|\ub - \bbeta^*\| \leq r\big\}$.
 
We first assume that $Q(\cdot|\bbeta^*)$ is {\sl self consistent} as stated below.   
\begin{condition} (Self Consistency)
\label{condition: self_consistency} Function $Q(\cdot|\bbeta^*)$ is self consistent, namely 
\[
\bbeta^* = \arg \max_{\bbeta \in \Omega} Q(\bbeta|\bbeta^*).
\]
\end{condition}

It is usually assumed that $\bbeta^*$ maximizes the population log likelihood function. Under this condition, Condition \ref{condition: self_consistency} is always satisfied by following the classical theory of EM algorithm \citet{mclachlan2007algorithm}.

Basically, we require $Q(\cdot|\bbeta)$ is differentiable over $\Omega$ for any $\bbeta \in \Omega$. We assume the function $Q(\cdot|\cdot)$ satisfies a certain strongly concavity condition and is smooth over $\Omega$.
\begin{condition} (Strong Concavity and Smoothness ($\gamma,\mu, r$))
\label{condition: strong concavity-smooth}
$Q(\cdot|\bbeta^*)$ is $\gamma$-strongly concave over $\Omega$, i.e.,
\begin{equation}
\label{strong concave}
 Q(\bbeta_2|\bbeta^*) - Q(\bbeta_1|\bbeta^*) - \big\langle \nabla Q(\bbeta_1|\bbeta^*),\bbeta_2 - \bbeta_1 \big\rangle \leq -\frac{\gamma}{2}\|\bbeta_2 - \bbeta_1\|^2, \;\forall\;\bbeta_1,\bbeta_2 \in \Omega.
\end{equation}
For any $\bbeta \in \cB(r;\bbeta^*)$, $Q(\cdot|\bbeta)$ is $\mu$-smooth over $\Omega$, i.e.,
\begin{equation}
\label{smooth}
 Q(\bbeta_2|\bbeta) - Q(\bbeta_1|\bbeta) - \big\langle \nabla Q(\bbeta_1|\bbeta),\bbeta_2 - \bbeta_1 \big\rangle \geq -\frac{\mu}{2}\|\bbeta_2 - \bbeta_1\|^2, \;\forall\;\bbeta_1,\bbeta_2 \in \Omega.
\end{equation}
\end{condition}

Condition \ref{condition: strong concavity-smooth} states that $Q(\cdot|\bbeta^*)$ is upper bounded by a quadratic function as shown in \eqref{strong concave}. Meanwhile, \eqref{smooth} implies that the function is lower bounded by another quadratic function. It's worth to note we require such lower bound holds for any function $Q(\cdot|\bbeta)$ with $\bbeta \in \cB(r;\bbeta^*)$ while the upper bound condition is imposed on single function $Q(\cdot|\bbeta^*)$. Similar strong concavity and smoothness conditions are widely used in convex optimization and play important roles in showing geometric convergence of gradient descent. Here, such condition will help us achieve geometric decay of optimization error in EM algorithm. 

The next condition is key in guaranteeing the curvature of $Q(\cdot|\bbeta)$ is similar to that of $Q(\cdot|\bbeta^*)$ when $\bbeta$ is close to $\bbeta^*$.
\begin{condition} \label{condition: gradient stability} (Gradient Stability ($\tau, r$)) For any $\bbeta \in \cB(r;\bbeta^*)$, we have
\[
\big\|\nabla Q(\cM(\bbeta)|\bbeta) - \nabla Q(\cM(\bbeta)|\bbeta^*)\big\| \leq \tau\|\bbeta - \bbeta^*\|.
\]
\end{condition}

The above condition only requires the gradient is stable at one point $\cM(\bbeta)$. This is sufficient for our analysis. In fact, for many concrete examples, one can verify a stronger version of condition \ref{condition: gradient stability}, i.e., for any $\bbeta' \in \cB(r;\bbeta^*)$ we have $\big\|\nabla Q(\bbeta'|\bbeta) - \nabla Q(\bbeta'|\bbeta^*)\big\| \leq \tau\|\bbeta - \bbeta^*\|$.

\subsubsection{Conditions on $Q_n(\cdot|\cdot)$}

Recall that $Q_n(\cdot|\cdot)$ is computed from finite number of samples according to \eqref{Q_sample}. We now turn to the two conditions about $Q_n(\cdot|\cdot)$. Our first condition, parallel to Condition \ref{condition: strong concavity-smooth} about function $Q(\cdot|\cdot)$, imposes curvature constraint on $Q_n(\cdot|\cdot)$ under finite number of samples. In order to guarantee the estimation error $\|\bbeta^{(t)} - \bbeta^*\|$ in step $t$ of EM algorithm is well controlled, we expect that $Q_n(\cdot|\bbeta^{(t-1)})$ is strongly concave at $\bbeta^*$. However, in the setting where $n \ll p$, there might exist directions along which $Q_n(\cdot|\bbeta^{(t-1)})$ is flat, as we observed in mixed linear regression and missing covariate regression. In contrast with Condition \ref{condition: strong concavity-smooth}, we suppose $Q_n(\cdot|\cdot)$ is strongly concave over a particular set $\cC(\cS,\overline{\cS};\cR)$ that is defined in terms of subspace pair $(\cS,\overline{\cS})$ and regularizer $\cR$. In detail, it takes form
\begin{equation}
\label{cone}
\cC(\cS,\overline{\cS};\cR) := \bigg\{\ub \in \RR^p: \big\| \Pi_{\overline{\cS}^{\bot}}(\ub)\big\|_{\cR} \leq 2\cdot\big\|\Pi_{\overline{\cS}}(\ub)\big\|_{\cR} + 2\cdot\Psi(\overline{\cS})\cdot\big\|\ub\big\|\bigg\},
\end{equation}
where the projection operator $\Pi_{\cS}: \RR^p \rightarrow \RR^p$ is defined as 
\begin{equation} \label{eq:proj}
\Pi_{\cS}(\ub) := \arg\min_{\vb \in \cS} \|\vb - \ub\|.
\end{equation}
With the geometric definition in hand, we provide the restricted strong concavity (RSC) condition as follows.
\begin{condition} \label{condition:restricted strong concavity} (RSC ($\gamma_n, \cS, \overline{\cS}, r, \delta$)) For any fixed $\bbeta \in \cB(r;\bbeta^*)$, with probability at least $1 - \delta$, we have that for all $\bbeta'-\bbeta^* \in \Omega\bigcap\cC(\cS,\overline{\cS};\cR)$,
\[
 Q_n(\bbeta'|\bbeta) - Q_n(\bbeta^*|\bbeta) - \big\langle \nabla Q_n(\bbeta^*|\bbeta),\bbeta' - \bbeta^* \big\rangle \leq -\frac{\gamma_n}{2}\|\bbeta' - \bbeta^*\|^2.
\]
\end{condition}
The above condition states that $Q_n(\cdot|\bbeta)$ is strongly concave in direction $\bbeta' - \bbeta^*$ that belongs to $\cC(\cS,\overline{\cS};\cR)$. It's instructive to compare Condition \ref{condition:restricted strong concavity} with a related condition proposed by  \citet{negahban2009unified} for analyzing high dimensional M-estimator. In detail, they assume the loss function is strongly convex over cone $\{\ub \in \RR^p: \|\Pi_{\overline{\cS}^{\bot}}(\ub)\|_{\cR} \lesssim \|\Pi_{\overline{\cS}}(\ub)\|_{\cR}\}$. Therefore our restrictive set \eqref{cone} is similar to the cone but has additional term $2\Psi(\overline{\cS})\|\ub\|$. The main purpose of term $2\Psi(\overline{\cS})\|\ub\|$ is to allow regularization parameter $\lambda_n$ jointly control optimization and statistical error. Note that while Condition \ref{condition:restricted strong concavity} is stronger than RSC condition in M-estimator since we expand the set, we usually only require the property $\|\Pi_{\overline{\cS}^{\bot}}(\ub)\|_{\cR} \lesssim \Psi(\overline{\cS})\|\ub\|$ for showing strong convexity/concavity. Both the set \eqref{cone} and the cone in M-estimator imply such property naturally.

Next, we establish the second condition that characterizes the achievable statistical error.
\begin{condition} \label{condition:statistical error}(Statistical Error ($\Delta_n, r, \delta$)) For any fixed $\bbeta \in \cB(r;\bbeta^*)$, with probability at least $1 - \delta$, we have 
\begin{equation}
\label{statistical error}
\big\|\nabla Q_n(\bbeta^*|\bbeta) - \nabla Q(\bbeta^*|\bbeta)\big\|_{\cR^*} \leq \Delta_n.
\end{equation}
\end{condition}

To provide some intuitions why the quantity $\big\|\nabla Q_n(\bbeta^*|\bbeta) - \nabla Q(\bbeta^*|\bbeta)\big\|_{\cR^*}$ is useful in representing the statistical error, we first note that $\lim_{n\rightarrow \infty} \Delta_n = 0$, which suggests that we obtain zero statistical error with infinite number of samples. In the case of finite samples, it's reasonable to believe that  $\Delta_n$ decreases while we increase $n$. The decreasing rate is indeed the statistical convergence rate we aims to figure out. We note that, in \citet{balakrishnan2014statistical} and \citet{wang2014high}, the statistical error is charactrized in terms of $\|\cM_n(\bbeta) - \cM(\bbeta)\|_2$ and $\|\cM_n(\bbeta) - \cM(\bbeta)\|_{\infty}$ respectively. As mentioned earlier, in high dimensional setting, $\cM_n(\bbeta)$ is not well defined in some models such as mixed linear regression. For mixed linear regression, \citet{wang2014high} resolves this issue by invoking a high dimensional inverse covariance matrix estimation algorithm proposed by \citet{cai2011constrained}. Our formulation \eqref{statistical error} avoids resolving such ad hoc problems arising from specific models.

\subsection{Main Results}
\label{sec:main results}

In this section, we provide the theoretical guarantees for regularized EM algorithm. Instead of analyzing Algorithm \ref{alg} directly, we introduce a resampling version of Algorithm \ref{alg} that is well suited to Conditions \ref{condition:restricted strong concavity}-\ref{condition:statistical error}. The key idea is to split the whole dataset into $T$ pieces and use a fresh piece of data in each iteration of regularized EM. We present the details in Algorithm \ref{alg:resampling}.

\begin{algorithm}[!htb]
	\caption{ High Dimensional Regularized EM Algorithm with Resampling}
	\label{alg:resampling}
	\begin{algorithmic}[1]
		\INPUT  Samples $\{\yb_i\}_{i=1}^{n}$, number of iterations $T$, $m = n/T$, initial regularization parameter $\lambda_{m}^{(0)}$, regularizer $\cR$, initial parameter $\bbeta^{(0)}$, estimated statistical error $\Delta$, contractive factor $\kappa < 1$.
		\STATE Evenly split $\{\yb_i\}_{i=1}^{n}$ into $T$ disjoint subsets $\cD_1,\cD_2,\ldots,\cD_{T}$.
		\FOR{$t = 1, 2, \ldots, T$}
		\STATE {\bf Regularization parameter update:}
		\begin{equation}
		\label{lambda_update_resampling}
		\lambda_{m}^{(t)}  \leftarrow \kappa\lambda_{m}^{(t-1)} + \Delta.
		\end{equation}
		\STATE {\bf E-step}: Compute function $Q_{m}^{(t)}(\cdot|\bbeta^{(t-1)})$ from sample set $\cD_t$ according to \eqref{Q_sample}.
		\STATE {\bf Regularized M-step:} 
		\[
		\bbeta^{(t)} \leftarrow \arg \max_{\bbeta \in \Omega} Q_{m}^{(t)}(\bbeta|\bbeta^{(t-1)}) - \lambda_{m}^{(t)}\cdot \cR(\bbeta).
		\]
		\ENDFOR
		\OUTPUT $\bbeta^{(T)}$.
	\end{algorithmic}
\end{algorithm}

For norm $\|\cdot\|$ under our consideration, we let $\alpha := \sup_{\ub \in \RR^p\setminus\{0\}}\|\ub\|_{*}/\|\ub\|$, where $\|\cdot\|_{*}$ is the dual norm of $\|\cdot\|$. For Algorithm \ref{alg}, we prove the following result.
\begin{theorem}
\label{thm:main}
We assume the model parameter $\bbeta^* \in \cS$ and regularizer $\cR$ is decomposable with respect to $(\cS,\overline{\cS})$ where $\cS \subseteq \overline{\cS} \subseteq \RR^p$. For some $r > 0$, suppose $\cB(r;\bbeta^*) \subseteq \Omega$. Suppose function $Q(\cdot|\cdot)$, defined in \eqref{Q_population}, is self consistent and satisfies Conditions \ref{condition: strong concavity-smooth}-\ref{condition: gradient stability} with parameters $(\gamma,\mu, r)$ and $(\tau, r)$. Given $n$ samples and $T$ iterations, let $m := n/T$. Suppose $Q_{m}(\cdot|\cdot)$, computed from any $m$ i.i.d. samples according to \eqref{Q_sample}, satisfies conditions \ref{condition:restricted strong concavity}-\ref{condition:statistical error} with parameters $(\gamma_{m},\cS,\overline{\cS}, r, 0.5\delta/T)$ and $(\Delta_{m}, r,0.5\delta/T)$.  Let
\[
\kappa^* := 5\frac{\alpha\mu\tau}{\gamma\gamma_{m}}.
\]
We assume $0 < \tau < \gamma$ and $0 < \kappa^* \leq 3/4$. Moreover, we define $\overline{\Delta} := r\gamma_{m}/[60\Psi(\overline{\cS})]$ and assume $\Delta_{m}$ is sufficiently small such that 
\begin{equation} \label{deltam}
\Delta_{m} \leq \overline{\Delta}.
\end{equation}
Consider the procedures in Algorithm \ref{alg:resampling} with initialization $\bbeta^{(0)} \in \cB(r;\bbeta^*)$ and let the regularization parameters be
\begin{equation}
\label{lambda_update2}
\lambda_{m}^{(t)} = \kappa^t\frac{\gamma_{m}}{5\Psi(\overline{\cS})}\|\bbeta^{(0)} - \bbeta^*\| + \frac{1 - \kappa^t}{1 - \kappa}\Delta,\; t = 1,2,\ldots,T
\end{equation}
for any $\Delta \in [3\Delta_{m}, 3\overline{\Delta}]$, $\kappa \in [\kappa^*,3/4]$. Then with probability at least $1 - \delta$, we have that for any $t \in [T]$,
\begin{equation}
\label{main_convergence}
\|\bbeta^{(t)} - \bbeta^*\| \leq  \kappa^t\|\bbeta^{(0)} - \bbeta^*\| + \frac{5}{\gamma_m} \frac{1 - \kappa^t}{1-\kappa}\Psi(\overline{\cS})\Delta.
\end{equation}
\end{theorem}
\begin{proof}
See Section \ref{proof:thm:main} for detailed proof.
\end{proof}
The above result suggests that with suitable regularization parameters, the estimation error is bounded by two terms. The first term, decaying geometrically with number of iterations $t$,  results from iterative optimization of function $Q_m$ thus is referred to as {\sl optimization error}. The second term called {\sl statistical error} characterizes the ultimate estimation error of Algorithm \ref{alg:resampling}.  With sufficiently large $T$ such that the second term dominates the first term and suitable choice of $\Delta$ such that $\Delta = O(\Delta_{n/T})$, we have the ultimate estimation error as 
\begin{equation}
\label{eq:tmp7}
\|\bbeta^{(T)} - \bbeta^*\| \lesssim \frac{1}{(1 - \kappa)\gamma_{n/T}}\Psi(\overline{\cS})\Delta_{n/T}.
\end{equation}
Since the optimization error decays exponentially with $T$ and $\Delta_n$ usually decays polynomially with $1/n$, it's sufficient to set $T = O(\log n)$ which thus bring us estimation error $O\left(\Psi(\overline{\cS})\Delta_{n/\log n}\right)$. We have a $\log n$ factor loss by using the resampling technique. Removing the logarithmic factor requires direct analysis of Algorithm \ref{alg} where the main ingredient is to assume Conditions \ref{condition:restricted strong concavity}-\ref{condition:statistical error} hold uniformly for all $\bbeta \in \cB(r;\bbeta^*)$ with high probability. Although extending Theorem \ref{thm:main} to cover Algorithm \ref{alg} with the new ingredient is straightforward, it's challenging to validate the new conditions in example models. 

We place the constraint $\Delta_{m} \lesssim r\gamma_{m}/\Psi(\overline{\cS})$ in \eqref{deltam} so that  $\bbeta^{(t)}$ is guaranteed to be contained in $\cB(r;\bbeta^*)$ for all $t \in [T]$. Note that this constraint is quite mild in the sense that if $\Delta_{m} = \Omega( r\gamma_{m}/\Psi(\overline{\cS}))$, $\bbeta^{(0)}$ is a decent estimator with estimation error $O(\Psi(\overline{\cS})\Delta_m/\gamma_m)$ that already matches our expectation. 

Equality \eqref{lambda_update2} corresponds to update rule of regularization parameters \eqref{lambda_update_resampling} in Algorithm \ref{alg:resampling}. Recall that with ($\lambda_m^{(0)}, \Delta, \kappa$), we choose the update rule to be
\[
\lambda_m^{t} \leftarrow \kappa\cdot\lambda_m^{t-1} + \Delta.
\]
Following Theorem \ref{thm:main},  ($\Delta, \kappa, \lambda_m^{(0)}$) should satisfy the following conditions
\begin{equation}
\label{lambda_update_rule}
\kappa^* \leq \kappa \leq 3/4, ~ 3\Delta_m \leq \Delta \leq 3\overline{\Delta},~  \lambda_m^{(0)} = \frac{\gamma_m}{5\Psi(\overline{\cS})}\|\bbeta^{(0)} - \bbeta^*\|.
\end{equation}
As we observe, quantity $\kappa^*$ is the minimum contractive parameter that is allowed to set. Parameter $\Delta$ characterizes the obtainable statistical error and should be set proportional to $\Delta_m$. Initial regularization parameter $\lambda_m^{(0)}$ characterizes the initial estimation error $\|\bbeta^{(0)} - \bbeta^*\|$. Note that accurate estimation of $\|\bbeta^{(0)} - \bbeta^*\|$ is not required. In fact, one can set $\lambda_m^{(0)} = \frac{\gamma_m}{5\Psi(\overline{\cS})} \varepsilon$ with any $\varepsilon \in \big[\|\bbeta^{(0)} - \bbeta^*\|_2,~r\big]$. Then with proof similar to that of Theorem \ref{thm:main}, we can show that
\[
\|\bbeta^{(t)} - \bbeta^*\| \leq  \kappa^t\varepsilon + \frac{5}{\gamma_m} \frac{1 - \kappa^t}{1-\kappa}\Psi(\overline{\cS})\Delta, \;\text{for all}\; t \in [T].
\] 
Consequently, overestimating initial error will potentially increase the total number of iterations but has no essential impact on the ultimate estimation error.

\section{Applications to Example Models} \label{sec:app}
In this section, we apply our high dimensional regularized algorithm and the analytical framework introduced in Section \ref{sec:theory} to the aforementioned three example models: Gaussian mixture model, mixed linear regression and missing covariate regression. For each model, based on the high dimensional regularized EM iterations introduced in Section \ref{sec:example models}, we provide the corresponding initialization condition, regularization update, computational convergence guarantee and statistical rate. 
\subsection{Gaussian Mixture Model}

We now use our analytical framework to analyze the Gaussian Mixture Model (GMM) with sparse model parameter $\bbeta^*$. Recall that we consider the isotropic, balanced Gaussian Mixture Model with two components where sample $\yb_i$ is generated from either $\cN(\bbeta^*,\sigma^2\Ib_p)$ or $\cN(\bbeta^*,\sigma^2\Ib_p)$. The following quantity called SNR is critical in characterizing the difficulty of estimating $\bbeta^*$.
\begin{equation}
\label{gaussian_mixture_snr}
{\rm SNR}:=\|\bbeta^*\|_2/\sigma.
\end{equation}
We focus on the high SNR regime where we assume ${\rm SNR} \geq \rho$ for some constant $\rho$. Note that the work in \citet{ma2005asymptotic} provides empirical and theoretical evidences that in low ${\rm SNR}$ regime, where the overlap density of two Gaussian cluster is small, standard EM algorithm suffers from sublinear convergence asymptotically. Therefore the high ${\rm SNR}$ condition is necessary for showing exponential/linear convergence of EM algorithm and our high dimensional variant.  Note that we are interested in quantizing estimation error using $\ell_2$ norm. We thus set the norm $\|\cdot\|$ in our framework to be $\|\cdot\|_2$ in this section. Recall that we set regularizer $\cR$ to be $\ell_1$ norm. For any subset $\cS \subseteq \{1,\ldots,p\}$, $\ell_1$ norm is decomposable with respect to $(\cS,\cS)$. For any $\bbeta^* \in \cB_0(s;p)$, by letting $\cS = \supp(\bbeta^*), \overline{\cS} = \supp(\bbeta^*)$,  we have $\Psi(\cS) = \sqrt{s}$ and $\cC(\cS,\overline{\cS};\cR)$ corresponds to $\{\|\ub_{\cS^{\bot}}\|_1 \leq 2\|\ub_{\cS}\|_1 + 2\sqrt{s}\|\ub\|_2\}$.  

According to the $Q^{GMM}_n(\cdot|\cdot)$ introduced in \eqref{gaussian_mixutre_Qn}, by taking expectation of it, we have
\begin{equation}
\label{gaussian_mixtrue_Q}
Q^{GMM}(\bbeta'|\bbeta) =  -\frac{1}{2}\mathbb{E}\left[w(Y;\bbeta)\|Y - \bbeta'\|_2^2 + (1 - w(Y;\bbeta))\|Y + \bbeta'\|_2^2\right].
\end{equation}
We now check Conditions \ref{condition: self_consistency}-\ref{condition: gradient stability} hold for $Q^{GMM}(\cdot|\cdot)$. We begin with proving the following result.
\begin{lemma}(Self consistency of GMM)\label{lem:gaussian_mixture_self_consistency}
Consider Gaussian mixture model with $Q^{GMM}(\cdot|\cdot)$ given in \eqref{gaussian_mixtrue_Q}. For model parameter $\bbeta^*$ we have
\[
\bbeta^* = \arg\max_{\bbeta \in \RR^p} Q^{GMM}(\bbeta|\bbeta^*).
\]
\end{lemma}
\begin{proof}
See Appendix \ref{proof:lem:gaussian_mixture_self_consistency} for detailed proof.
\end{proof}
The above result suggests that $Q^{GMM}(\cdot|\cdot)$ satisfies Condition \ref{condition: self_consistency}. 
It is easy to see $\nabla^2Q^{GMM}(\bbeta'|\bbeta) = -\Ib_p$, which implies that $Q^{GMM}(\cdot|\cdot)$ satisfy Condition \ref{condition: strong concavity-smooth} with parameters $(\gamma,\mu, r) = (1,1, r)$ for any $r > 0$. Next we present a result showing that $Q^{GMM}(\cdot|\cdot)$ satisfies Condition \ref{condition: gradient stability} with arbitrarily small  stability factor $\tau$ when ${\rm SNR}$ is sufficiently large.
\begin{lemma} (Gradient stability of GMM)\label{lem:gaussian_mixture_gstable}
Consider the Gaussian Mixture Model with $Q^{GMM}(\cdot|\cdot)$ given in \eqref{gaussian_mixtrue_Q}. Suppose ${\rm SNR}$ defined in \eqref{gaussian_mixture_snr} is lower bounded by $\rho$, i.e., ${\rm SNR} \geq \rho$. Function $Q^{GMM}(\cdot|\cdot)$ satisfies Condition \ref{condition: gradient stability} with parameters $(\tau, \|\bbeta^*\|_2/4)$, where $\tau \leq \exp(-C\rho^2)$ for some absolute constant $C$.
\end{lemma}
\begin{proof}
	See the proof of Lemma 3 in \citet{balakrishnan2014statistical}.
\end{proof}

Now we turn to the conditions about $Q^{GMM}_n(\cdot|\cdot)$.
\begin{lemma} (RSC of GMM)
\label{lem:gaussian_mixture_rsc}
Consider Gaussian mixture model with any $\bbeta^* \in \cB_0(s;p)$ and $Q_n^{GMM}(\cdot|\cdot)$ given in \eqref{gaussian_mixutre_Qn}. For any $r > 0$, we have $Q_n^{GMM}(\cdot|\cdot)$ satisfies Condition \ref{condition:restricted strong concavity} with parameters $(\gamma_n,\cS,\overline{\cS},r,\delta)$, where 
\[
\gamma_n = 1,\; \delta = 0,\; (\cS, \overline{\cS}) = (\supp(\bbeta^*),\supp(\bbeta^*)).
\]
\end{lemma}
\begin{proof}
See Appendix \ref{proof:lem:gaussian_mixture_rsc} for detailed proof.
\end{proof}	
This above result indicates that the restricted strong concavity condition holds deterministically in this example. The next lemma validates the statistical error condition and provides the corresponding parameters.
\begin{lemma} (Statistical error of GMM)\label{lem:gaussian_mixture_staterr}
Consider Gaussian mixture model with $Q^{GMM}_n(\cdot|\cdot)$ and $Q^{GMM}(\cdot|\cdot)$ given in \eqref{gaussian_mixutre_Qn} and \eqref{gaussian_mixtrue_Q} respectively. For any $r > 0$, $\delta \in (0,1)$ and some absolute constant $C$, Condition \ref{condition:statistical error} holds with parameters $(\Delta_n, r, \delta)$ where
\[
\Delta_n = C(\|\bbeta^*\|_{\infty} + \sigma)\sqrt{\frac{\log p + \log(2e/\delta)}{n}}.
\]
\end{lemma} 
\begin{proof}
	See Appendix \ref{proof:lem:gaussian_mixture_staterr} for detailed proof.
\end{proof}
Now we give the guarantees of Algorithm \ref{alg:resampling} for Gaussian mixture model.
\begin{corollary} (Sparse Recovery in GMM) \label{thm:GMM}
Consider the Gaussian mixture model with any fixed $\bbeta^* \in \cB_0(s;p)$ and implementations of Algorithm \ref{alg:resampling} with regularizer $\ell_1$ norm. Suppose $\bbeta^{(0)} \in \cB(\|\bbeta^*\|_2/4;\bbeta^*)$ and ${\rm SNR} \geq \rho$ with sufficiently large $\rho$. Let initial regularization parameter $\lambda_{n/T}^{(0)}$ be
\[
\lambda_{n/T}^{(0)} = \frac{1}{5\sqrt{s}}\|\bbeta^{(0)}-\bbeta^*\|_2
\]
and quantity $\Delta$ be
\[\Delta = C(\|\bbeta^*\|_{\infty} + \sigma)\sqrt{\frac{\log p}{n}T}
\] 
for sufficiently large constant $C$. Moreover, let the number of samples $n$ be sufficiently large such that 
\begin{equation}
\label{eq:tmp6}
n/T \geq \left[80C(\|\bbeta^*\|_{\infty} + \sigma)/\|\bbeta^*\|_2\right]^2s\log p. 
\end{equation}
Then by setting $\lambda_{n/T}^{(t)} = \kappa^t\lambda_{n/T}^{(0)} + \frac{1-\kappa^t}{1-\kappa}\Delta$ for any $\kappa \in [1/2, 3/4]$, with probability at least $1 - T/p$, we have that
\begin{equation} \label{GMM_esterr}
\|\bbeta^{(t)} - \bbeta^*\|_2 \leq \kappa^t\|\bbeta^{(0)}-\bbeta^*\|_2 + \frac{5C(\|\bbeta^*\|_{\infty} + \sigma)}{1-\kappa}\sqrt{\frac{s\log p}{n}T}, \;\text{for all}\; t \in [T].
\end{equation}
\end{corollary} 
\begin{proof}
	This result follows from Theorem \ref{thm:main}. First, recall that the minimum contractive factor $\kappa^*$ is $\kappa^* = 5\frac{\alpha\mu\tau}{\gamma\gamma_{n/T}}$. For $\ell_2$ norm, we have $\alpha = 1$. Following the fact that $(\gamma,\mu) = (1,1)$ and Lemma \ref{lem:gaussian_mixture_gstable}-\ref{lem:gaussian_mixture_rsc}, we have $\kappa^* \leq 20\exp(-C\rho^2)$ for some constant $C$. We further have $\kappa^* \leq \frac{1}{2}$ when $\rho$ is sufficiently large. Second, based on Lemma \ref{lem:gaussian_mixture_staterr}, we set $\delta = 1/p$ and let $\Delta$ be $\Delta = C(\|\bbeta^*\|_{\infty} + \sigma)\sqrt{T\log p/n}$ with sufficiently large $C$ such that $\Delta \geq 3\Delta_{n/T}$. Thirdly, by the assumption in \eqref{eq:tmp6}, we have that $\Delta \leq 3\overline{\Delta}$ where $\overline{\Delta} = \|\bbeta^*\|_2/(240\sqrt{s})$ in this example. Finally, we choose $\lambda_{n/T}^{(0)} = \|\bbeta^{(0)}-\bbeta^*\|/(5\sqrt{s})$ by following \eqref{lambda_update_rule}. Packing up these ingredients and following Theorem \ref{thm:main}, we have that by choosing any $\kappa \in [1/2,3/4]$, $\|\bbeta^{(t)} - \bbeta^*\|_2 \leq \kappa^{t}\|\bbeta^{(0)}-\bbeta^*\|_2 + 5\sqrt{s}\Delta/(1-\kappa)$, which thus completes the proof.
\end{proof}

Note that $\|\bbeta^{(0)} - \bbeta^*\| \lesssim \|\bbeta^*\|_2 \leq \sqrt{s}\|\bbeta^*\|_{\infty}$. Let us set $T = C\log(\frac{n}{\log p})$ for sufficiently large $C$ such that the first term in \eqref{GMM_esterr} is dominated by the second term. Then Corollary \ref{thm:GMM} suggests the final estimation is 
\[
\|\bbeta^{(T)} - \bbeta^*\|_2 \lesssim C(\|\bbeta^*\|_{\infty}+ \delta)\sqrt{\frac{s\log p}{n}\log\left(\frac{n}{\log p}\right)}.
\]
Note that the minimax rate for estimating $s$-sparse vector in a single Gaussian cluster is $\sqrt{s\log p/n}$, thereby the established rate is optimal on $(n,p,s)$ up to a logarithmic factor.

\subsection{Mixed Linear Regression}

We now turn to Mixed Linear Regression (MLR) model. In particular, we will consider two sets of model parameters: $\bbeta^* \in \cB_0(s;p)$ and $\Gammab^* \in \RR^{p_1\times p_2}$ with $\rank(\Gammab^*) = r$. For the two settings, the population level analysis is identical under i.i.d. Gaussian covariate design. Without loss of generality, we begin with treating the model parameter as a vector $\bbeta^* \in \RR^p$ and validate Conditions \ref{condition: self_consistency}-\ref{condition: gradient stability} for $Q^{MLR}(\cdot|\cdot)$ in this example. Given function $Q_n^{MLR}(\cdot|\cdot)$ in \eqref{mlr_Qn}, by taking expectation of it, we have 
\begin{equation}
\label{mlr_Q}
Q^{MLR}(\bbeta'|\bbeta) = -\frac{1}{2}\mathbb{E}\left[w(Y,X;\bbeta)(Y-\langle X,\bbeta'\rangle)^2 + (1 - w(Y,X;\bbeta))(Y + \langle X,\bbeta'\rangle)^2\right]
\end{equation}
For now, we set the norm $\|\cdot\|$ in our framework to $\|\cdot\|_2$. We begin by checking the self consistency condition.
\begin{lemma} (Self consistency of MLR)\label{lem:mlr:self_consistency}
Consider mixed linear regression with model parameter $\bbeta^* \in \RR^p$ and $Q^{MLR}(\cdot|\cdot)$ given in \eqref{mlr_Q}. We have
\[
\bbeta^* = \arg\max_{\bbeta \in \RR^p}Q^{MLR}(\bbeta|\bbeta^*).
\]	
\end{lemma} 
\begin{proof}
	See Appendix \ref{proof:lem:mlr:self_consistency} for detailed proof.
\end{proof}
It is easy to check $\nabla^2Q^{MLR}(\bbeta'|\bbeta) = -\Ib_p$. Therefore, $Q^{MLR}(\cdot|\cdot)$ satisfies Condition \ref{condition: strong concavity-smooth} with parameters $(\gamma,\mu, r) = (1,1, r)$ for any $r > 0$. Similar to Gaussian mixture model, we introduce the following ${\rm SNR}$ quantity to characterize the hardness of the problem.
\[
{\rm SNR} := \|\bbeta^*\|/\sigma.
\]
The work in \cite{chen2014convex} shows that there exists an unavoidable phase transition of statistical rate from high ${\rm SNR}$ to low ${\rm SNR}$. In detail, in low-dimensional setting, the obtainable statistical error is $\Omega(\sqrt{p/n})$ that matches the standard linear regression when ${\rm SNR} \geq \rho$ for some constant $\rho$. Meanwhile, the unavoidable rate becomes $\Omega((p/n)^{1/4})$ when ${\rm SNR} \ll \rho$. We conjecture such transition phenomenon still exists in high dimensional setting. For now we focus on the high ${\rm SNR}$ regime and show our algorithm achieves statistical rate that matches the standard sparse linear regression and low rank matrix recovery (up to logarithmic factor) in the end. 

The following result validates Condition \ref{condition: gradient stability} holds with arbitrarily small stability factor $\tau$ when ${\rm SNR}$ is sufficiently large and the radius $r$ of ball $\cB(r;\bbeta^*)$ is sufficiently small.
\begin{lemma} (Gradient Stability of MLR)\label{lem:mlr:gradient_stability}
Consider mixed linear regression model with function $Q^{MLR}(\cdot|\cdot)$ given in \eqref{mlr_Q}. For any $\omega \in [0, 1/4]$, let $r = \omega\|\bbeta^*\|_2$. Suppose ${\rm SNR} \geq \rho$ for some constant $\rho$. Then for any $\bbeta \in \cB(r;\bbeta^*)$, we have 
\[
\|\nabla Q^{MLR}(\cM(\bbeta)|\bbeta) - \nabla Q^{MLR}(\cM(\bbeta)|\bbeta^*)\|_2 \leq \tau \|\bbeta - \bbeta^*\|_2
\]	
with
\[
\tau = \frac{17}{\rho} + 7.3 \omega.
\]
\end{lemma}
\begin{proof}
	See Appendix \ref{proof:mlr:gradient_stability} for detailed proof.
\end{proof}
In \citet{balakrishnan2014statistical}, it is proved that when $r = \frac{1}{32}\|\bbeta^*\|_2$, there exists $\tau \in [0,1/2]$ such that $Q^{MLR}(\cdot|\cdot)$ satisfies Condition \ref{condition: gradient stability} with parameter $\tau$ when $\rho$ is sufficiently large. Note that Lemma \ref{lem:mlr:gradient_stability} recovers this result. Moreover, Lemma \ref{lem:mlr:gradient_stability} provides an explicit function to characterize the relationship between $\tau$ and $\rho,\omega$.

Next we turn to validate the two technical conditions of $Q_n^{MLR}(\cdot|\cdot)$ and establish the computational and statistical guarantees of estimating mixed linear parameters in high dimensional regime. We consider two different structures of linear parameters: (1) model parameter $\bbeta^*$ is a sparse vector; (2) model parameter $\bGamma^*$ is a low rank matrix. Note that we assume $X$ is a fully random Gaussian vector/matrix, thereby the population level conditions about $Q^{MLR}(\cdot|\cdot)$ hold in both settings.

\vspace{0.1in}
\noindent{\bf Sparse Recovery.}
We assume model parameter $\bbeta^*$ is $s$-sparse, i.e., $\bbeta^* \in \cB_0(s;p)$. Recall that, in order to serve sparse structure, we choose $\cR$ to be $\ell_1$ norm. Setting $\cS = \overline{\cS} = \supp(\bbeta^*)$, set $\cC(\cS,\overline{\cS};\cR)$ corresponds to $\{\ub:\|\ub_{\cS^{\bot}}\|_1 \leq 2\|\ub_{\cS}\|_1 + 2\sqrt{s}\|\ub\|_2\}$. Restricted concavity of $Q^{MLR}(\cdot|\cdot)$ is validated in the following result.
\begin{lemma} (RSC of MLR with sparsity) \label{lem:mlr_rsc}
Consider mixed linear regression with any model parameter $\bbeta^* \in \cB_0(s;p)$ and function $Q^{MLR}_n(\cdot|\cdot)$ defined in \eqref{mlr_Qn}. There exit absolute constants $\{C_i\}_{i=0}^3$ such that, if $n \geq C_0s\log p$, then for any $r > 0$, $Q_n^{MLR}(\cdot|\cdot)$ satisfies Condition \ref{condition:restricted strong concavity} with parameters $(\gamma_n, \cS, \overline{\cS}, r, \delta)$, where 
\[
\gamma_n = \frac{1}{3},\;(\cS,\overline{\cS}) = (\supp(\bbeta^*), \supp(\bbeta^*)),\; \delta = C_1\exp(-C_2n).
\]	
\end{lemma}
\begin{proof}
	See Appendix \ref{proof:lem:mlr_rsc} for detailed proof.
\end{proof}

Lemma \ref{lem:mlr_rsc} states that using $n = O(s\log p)$ samples makes $Q_n^{MLR}(\cdot|\cdot)$ be strongly concave over $\cC$ with high probability.

\begin{lemma} \label{lem:mlr:stat_error} (Statistical error of MLR with sparsity) 
Consider mixed linear regression model with any $\bbeta^* \in \cB_0(s;p)$ and functions $Q^{MLR}_n(\cdot|\cdot), Q^{MLR}(\cdot|\cdot)$ defined in \eqref{mlr_Qn} and \eqref{mlr_Q} respectively. There exist constants $C$ and $C_1$ such that, for any $r > 0$ and $\delta \in (0,1)$, if $n \geq C_1 (\log p + \log(6/\delta))$, then
\[
\|\nabla Q^{MLR}_n(\bbeta^*|\bbeta) - \nabla Q^{MLR}(\bbeta^*|\bbeta)\|_{\infty} \leq  C(\|\bbeta^*\|_2 + \delta)\sqrt{\frac{\log p + \log(6/\delta)}{n}} \;\text{for all}\; \bbeta \in \cB(r;\bbeta^*)
\]
with probability at least $1 - \delta$.
\end{lemma}
\begin{proof}
	See Appendix \ref{proof:lem:mlr:stat_error} for detailed proof.
\end{proof}
Lemma \ref{lem:mlr:stat_error} implies Condition \ref{condition:statistical error} hold with parameters $\Delta_n = O\left(\left(\|\bbeta^*\|_2 + \delta\right)\sqrt{\frac{\log p}{n}}\right)$, any $r > 0$ and $\delta = 1/p$. Putting all the ingredients together leads to the following guarantee about sparse recovery in mixed linear regression using regularized EM algorithm.
\begin{corollary}\label{thm:mlr:sparse} (Sparse recovery in MLR)
Consider the mixed linear regression model with any fixed model parameter $\bbeta^* \in \cB_0(s;p)$ and the implementation of Algorithm \ref{alg:resampling} using $\ell_1$ regularization. Suppose ${\rm SNR} \geq \rho$ for sufficiently large $\rho$. Let quantity $\Delta$ be
\[
\Delta = C(\|\bbeta^*\|_2 + \delta)\sqrt{\frac{\log p}{n}T}
\] 
and number of samples $n$ satisfy
\[
n/T \geq C'\left[(\|\bbeta^*\|_2 + \delta)/\|\bbeta^*\|_2\right]^2s\log p
\]
for some sufficiently large constants $C$ and $C'$. Given any fixed $\bbeta^{(0)} \in \cB(\|\bbeta^*\|_2/240, \bbeta^*)$, let initial regularization parameter $\lambda_{n/T}^{(0)}$ be
\[
\lambda_{n/T}^{(0)} = \frac{1}{15\sqrt{s}}\|\bbeta^{(0)} - \bbeta^*\|_2.
\]
Then by setting $\lambda_{n/T}^{(t)} = \kappa^t\lambda_{n/T}^{(0)} + \frac{1-\kappa^t}{1-\kappa}\Delta$ for any $\kappa \in [1/2,3/4]$, we have that, with probability at least $1 - T/p$,
\[
\|\bbeta^{(t)} - \bbeta^*\|_2 \leq \kappa^t\|\bbeta^{(0)} - \bbeta^*\|_2 + \frac{15C(\|\bbeta^*\|_2 + \delta)}{1-\kappa}\sqrt{\frac{s\log p}{n}T},\;\text{for all}\; t \in [T].
\]
\end{corollary}
\begin{proof}
The result follows from Theorem \ref{thm:main}. First, we note that the minimum contractive factor $\kappa^* = 5\frac{\alpha\mu\tau}{\gamma\gamma_{n/T}} = 15\tau$ in this example since $\alpha = 1, \mu = \gamma = 1$ and $\gamma_{n/T} = 1/3$ w.h.p when $n \gtrsim s\log p$ (see Lemma \ref{lem:mlr_rsc}). Following Lemma \ref{lem:mlr:gradient_stability}, $\kappa^* \leq 1/2$ when $w \leq 1/240$ and $\rho$ is sufficiently large. Second, by choosing $n/T \gtrsim s\log p$, we have $\Delta_{n/T} \lesssim (\|\bbeta^*\|_2 + \delta)\sqrt{\frac{T\log p}{n}}$ w.h.p., as proved in Lemma \ref{lem:mlr:stat_error}. Lastly, we have $\Delta \leq 3\overline{\Delta}$ by assuming $n/T \gtrsim \left[(\|\bbeta^*\|_2 + \delta)/\|\bbeta^*\|_2\right]^2s\log p$. Putting these ingredients together and plugging the established parameters into \eqref{main_convergence} complete the proof.
\end{proof}
Corollary \ref{thm:mlr:sparse} provides that the final estimation error is 
\[
\|\bbeta^{(T)} - \bbeta^*\|_2 \lesssim \kappa^{T}\|\bbeta^{(0)} - \bbeta^*\|_2 + (\|\bbeta^*\|_2 + \delta)\sqrt{T\frac{s\log p}{n}}.
\]
Note that the second term dominates when $T$ is chosen to satisfy $T \gtrsim \log\left(n/(Ts\log p)\right)$. Performing $T = C\log(n/(s\log p))$ iterations gives us 
\[
\|\bbeta^{(T)} - \bbeta^*\|_2 \lesssim (\|\bbeta^*\|_2 + \delta)\sqrt{\frac{s\log p}{n}\log\left(\frac{n}{s\log p}\right)}.
\]
The dependence on $(s,p,n)$ is thus $O\left((s\log p/n)^{1/2-c}\right)$ for any $c > 0$. Note that the standard sparse regression has optimal statistical error $\sqrt{s\log p/n}$, thereby the obtained rate for mixed linear regression is optimal up to logarithmic factor. A caveat here is that the estimation error is proportional to signal strength $\|\bbeta^*\|_2$, i.e., $s\log p/n$ determines the relative error instead of absolute error as usually observed in high dimensional estimation. This phenomena, also appearing in low dimensional analysis \citet{balakrishnan2014statistical}, arises from the fundamental limits of EM algorithm. It's worth to note that \cite{chen2014convex} establish near-optimal low dimensional estimation error that does not depend on $\|\bbeta^*\|_2$ based on a convex optimization approach. It's interesting to explore how to remove $\|\bbeta^*\|_2$ in high dimensional setting. 

\vspace{0.1in}
\noindent{\bf Low Rank Recovery.}
In the sequel, we assume model parameter $\bGamma^* \in \RR^{p_1\times p_2}$ is a low rank matrix that has ${\rm rank}(\bGamma^*) = \theta \ll \min\{p_1,p_2\}$. We focus on measuring the estimation error in Frobenius norm thus set $\|\cdot\|$ in our framework to be $\|\cdot\|_F$. Note that by treating $\bGamma^*$ as a vector, Frobenius norm is equivalent to $\ell_2$ norm, thereby we still have Lemma \ref{lem:mlr:self_consistency}-\ref{lem:mlr:gradient_stability} in this setting. Moreover, ${\rm SNR}$ is similarly defined as  
\[
{\rm SNR} := \|\bGamma^*\|_{F}/\sigma.
\] 
In order to serve the low rank structure, we choose $\cR$ to be nuclear norm $\|\cdot\|_{*}$. For any matrix $\Mb$, we let ${\rm row}(\Mb)$ denote the subspace spanned by the rows of $\Mb$ and ${\rm col}(\Mb)$ denote the subspace spanned by the columns of $\Mb$. Moreover, for subspace represented by the columns of matrix $\Ub$, we denote the subspace orthogonal to $\Ub$ as $\Ub^{\bot}$. For $\bGamma^*$ with singular value decomposition $\Ub^*\bSigma\Vb^{*\top}$, we thus let
\begin{equation} \label{mlr:S_def}
\cS = \left\{\Mb \in \RR^{p_1 \times p_2}: {\rm col}(\Mb) \subseteq \Ub^*, {\rm row}(\Mb) \subseteq \Vb^*\right\}
\end{equation}
and
\begin{equation} \label{mlr:Sbar_def}
\overline{\cS}^{\bot} = \left\{\Mb \in \RR^{p_1 \times p_2}: {\rm col}(\Mb) \subseteq \Ub^{*\bot}, {\rm row}(\Mb) \subseteq \Vb^{*\bot}\right\}.
\end{equation}
So $\cS$ contains all matrices with rows (and columns) living in the row (and column) space of $\bGamma^*$. Subspace $\overline{\cS}^{\bot}$ contains all matrices with rows (and columns) orthogonal to the row (and column) space of $\bGamma^*$. Nuclear norm is decomposable with respect to $(\cS,\overline{\cS})$. We have $\Psi(\overline{\cS}) = \sup_{\Mb \in \overline{\cS}\setminus \{{\bm 0}\} } \|\Mb\|_{*}/\|\Mb\|_{F} \leq \sqrt{2\theta}$ since matrix in $\overline{\cS}$ has rank at most $2\theta$. Similar to Lemma \ref{lem:mlr_rsc} and \ref{lem:mlr:stat_error} for sparse structure, we have the following two results for low rank structure.
\begin{lemma} (RSC of MLR with low rank structure) \label{lem:mlr_rsc_lowrank}
Consider mixed linear regression with model parameter $\bGamma^* \in \RR^{p_1 \times p_2}$ that has $\rank(\bGamma^*) = \theta$. There exists constants $\{C_i\}_{i=0}^2$ such that, if $n \geq C_0\theta\max\{p_1,p_2\}$, then for any $\theta \in (0, \min\{p_1,p_2\})$, $Q_n^{MLR}(\cdot|\cdot)$ satisfies Condition \ref{condition:restricted strong concavity} with parameters $(\gamma_n,\cS,\overline{\cS},r,\delta)$, where $(\cS, \overline{\cS})$ are given in \eqref{mlr:S_def} and \eqref{mlr:Sbar_def}, 
\[
\gamma_n = \frac{1}{20}, \;\delta = C_1\exp(-C_2n).
\]
\end{lemma}
\begin{proof}
	See Appendix \ref{proof:lem:mlr_rsc_lowrank} for detailed proof.
\end{proof}

\begin{lemma} \label{lem:mlr_staterr_lowrank}(Statistical error of MLR with low rank structure) Consider the mixed linear regression with any $\bGamma^* \in \RR^{p_1\times p_2}$. There exists constants $C$ and $C_1$ such that, for any fixed $\bGamma \in \RR^{p_1\times p_2}$ and $\delta \in (0,1)$, if  $n \geq C_1(p_1 + p_2 + \log(6/\delta))$, then
\[
\|\nabla Q^{MLR}(\bGamma^*|\bGamma) - \nabla Q_n^{MLR}(\bGamma^*|\bGamma)\|_{2} \leq C(\|\bSigma^*\|_F + \sigma)\sqrt{\frac{p_1+p_2+\log(6/\delta)}{n}}	
\]
with probability at least $1 - \delta$.
\end{lemma}
\begin{proof}
	See the Appendix \ref{proof:lem:mlr_staterr_lowrank} for detailed proof.
\end{proof}
Setting $\delta = 6\exp(-(p_1+p_2))$ in Lemma \ref{lem:mlr_staterr_lowrank} suggests that Condition \ref{condition:statistical error} holds with parameters $(\Delta_n, r, \delta)$ where $\Delta_n \lesssim (\|\bGamma^*\|_F + \delta)\sqrt{(p_1+p_2)/n}$, $\delta = \exp(-(p_1+p_2))$ and $r$ can be any positive number. Putting these pieces together leads to the following guarantee about low rank recovery.

\begin{corollary} \label{thm:mlr_lowrank} (Low rank recovery in MLR) Consider mixed linear regression with any model parameter $\bGamma^* \in \RR^{p_1\times p_2}$ that has rank at most $\theta$ and the implementation of Algorithm \ref{alg:resampling} with nuclear norm regularization. Suppose ${\rm SNR} \geq \rho$ for sufficiently large $\rho$. Let quantity $\Delta$ be
\[
\Delta = C(\|\bGamma^*\|_F + \sigma)\sqrt{\frac{p_1+p_2}{n}T}	
\]
and the number of samples $n$ satisfy
\[
n/T \geq C'\left[(\|\bGamma^*\|_F + \sigma)/\|\bGamma^*\|_F\right]^2\theta(p_1+p_2)
\]
for some sufficiently large constants $C$ and $C'$. Given any fixed $\bGamma^{(0)} \in \cB(\|\bGamma^*\|_F/1600,\bGamma^*)$, let initial regularization parameter $\lambda_{n/T}^{(0)}$ be 
\[
\lambda_{n/T}^{(0)} = \frac{1}{100\sqrt{2\theta}}\|\bGamma^{(0)}-\bGamma^*\|_F.
\]
Then by setting $\lambda_{n/T}^{(t)} = \kappa^t\lambda_{n/T}^{(0)} + \frac{1-\kappa^t}{1-\kappa}\Delta$ for any $\kappa \in [1/2,3/4]$, we have that, with probability at least $1 - T\exp(-p_1-p_2)$,
\[
\|\bGamma^{(t)} - \bGamma^*\|_F \leq \kappa^t\|\bGamma^{(0)} - \bGamma^*\|_F + \frac{100C'(\|\bGamma^*\|_F + \sigma)}{1-\kappa}\sqrt{\frac{2\theta(p_1+p_2)}{n}T}, \;\text{for all}\; t \in [T].
\]
\end{corollary}
\begin{proof}
This result is parallel to Corollary \ref{thm:mlr:sparse} for sparse recovery thus can be proved similarly. We omit the details.
\end{proof}
Corollary \ref{thm:mlr_lowrank} indicates that the final estimation error can be characterized by
\[
\|\bGamma^{(T)} - \bGamma^*\|_F \lesssim \kappa^T\|\bGamma^{(0)} - \bGamma^*\|_F + (\|\bGamma^*\|_F + \sigma)\sqrt{\frac{\theta(p_1+p_2)}{n}T}.
\]
Note that the initialization error is proportional to $\|\bGamma^*\|_F$. Choosing $T = O(\log(n/[\theta(p_1+p_2)]))$, the first term representing optimization error is then dominated by the second term. We thus have
\[
\|\bGamma^{(T)} - \bGamma^*\|_F \lesssim (\|\bGamma^*\|_F + \sigma)\sqrt{\frac{\theta(p_1+p_2)}{n}\log\left(\frac{n}{\theta(p_1+p_2)}\right)}.
\]
The established statistical rate matches (up to the logarithmic factor) the (single) low rank matrix estimation rate proved in \citet{candes2011tight, negahban2011estimation}, which is known to be minimax optimal. It's worth to note that our rate is proportional to the signal strength $\|\bGamma^*\|_F$. Therefore, the normalized sample complexity $n/[\theta(p_1+p_2)]$ controls the relative error instead of absolute error in standard low rank matrix estimation.
\subsection{Missing Covariate Regression} \label{sec:mcr}
We now consider the sparse recovery guarantee of Algorithm \ref{alg:resampling} for missing covariate regression. We begin by validating conditions about function $Q^{MCR}(\cdot|\cdot)$, which has form
\begin{equation} \label{mcr_Q}
Q^{MCR}(\bbeta'|\bbeta) = \left\langle \mathbb{E}\left[Y\bmu_{\bbeta}(Y,Z,X)\right],\bbeta'\right\rangle - \frac{1}{2}\left\langle \mathbb{E}\left[\Sigmab_{\bbeta}(Y,Z,X)\right], \bbeta\bbeta^{\top}\right\rangle.
\end{equation}
First, $\cM(\cdot)$ is self consistent as stated below. 
\begin{lemma} \label{lem:mcr_self_consistency} (Self-consistency of MCR) Consider missing covariate regression with parameter $\bbeta^* \in \RR^p$  and $Q^{MCR}(\cdot|\cdot)$ given in \eqref{mcr_Q}. We have
\[
\bbeta^* = \arg\max_{\bbeta \in \RR^p} Q^{MCR}(\bbeta|\bbeta^*).
\]
\end{lemma}
\begin{proof}
	See Appendix \ref{proof:lem:mcr_self_consistency} for detailed proof.
\end{proof}

For our analysis, we define $\rho := \|\bbeta^*\|_2/\sigma$ to be the {\sl signal to noise ratio} and $\omega := r/\|\bbeta^*\|_2$ to be the {\sl relative contractivity radius}. Let 
\[
\zeta : = (1+\omega)\rho.
\]
Recall that $\epsilon$ is the missing probability of every entry. The next result characterizes the smoothness and concavity of $Q^{MCR}(\cdot|\cdot)$.

\begin{lemma} \label{lem:mcr_smrc}(Smoothness and concavity of MCR) Consider missing covariate regression with parameter $\bbeta^* \in \RR^p$ and $Q^{MCR}(\cdot|\cdot)$ given in \eqref{mcr_Q}. For any $\omega > 0$, we have that $Q^{MCR}(\cdot|\cdot)$ satisfies Condition \ref{condition: strong concavity-smooth} with parameters $(\gamma,\mu, \omega\|\bbeta^*\|_2)$, where
\[
\gamma = 1,~~ \mu = 1 + 2\zeta^2\sqrt{\epsilon} + (1+\zeta^2)\zeta^2\epsilon.
\]
\end{lemma}
\begin{proof}
	See Appendix \ref{proof:lem:mcr_smrc} for detailed proof.
\end{proof}
We revisit the following result about the gradient stability from \citet{balakrishnan2014statistical}. 

\begin{lemma} \label{lem:mcr_grad}(Gradient stability of MCR) Consider the missing covariate regression with $\bbeta^* \in \RR^p$ and $Q^{MCR}(\cdot|\cdot)$ given in \eqref{mcr_Q}. For any $\omega > 0, \rho > 0$, $Q^{MCR}(\cdot|\cdot)$ satisfies Condition \ref{condition: gradient stability} with parameter $(\tau, \omega\|\bbeta^*\|_2)$ where
\[
\tau = \frac{\zeta^2 + 2\epsilon(1+\zeta^2)^2}{1+\zeta^2}. 
\]	
\end{lemma}
\begin{proof}
	See the proof of Corollary 6 in \cite{balakrishnan2014statistical}.
\end{proof}
Unlike the previous two models, we require an upper bound on the signal noise ratio. This unusual constraint is in fact unavoidable, as pointed out in \cite{loh2012corrupted}. 

We now turn to validate the conditions about finite sample function $Q^{MCR}_n(\cdot|\cdot)$. In particular, we have the following two guarantees.

\begin{lemma} \label{lem:mcr_rsc}(RSC of MCR) Consider missing covariate regression with any fixed parameter $\bbeta^* \in \cB_0(s;p)$ and $Q_n^{MCR}(\cdot|\cdot)$ given in \eqref{mcr_Qn}. There exist constants $\{C_i\}_{i=0}^3$ such that if $\epsilon \leq C_0\min\{1,\zeta^{-4}\}$ and $n \geq C_1(1+\zeta)^8s\log p$, then we have $Q_n^{MCR}(\cdot|\cdot)$ satisfies Condition \ref{condition:restricted strong concavity} with parameters $(\gamma_n,\cS,\overline{\cS},\omega\|\bbeta^*\|_2,\delta)$, where
\[
\gamma_n = \frac{1}{9}, ~~ (\cS,\overline{\cS}) = (\supp(\bbeta^*), \supp(\bbeta^*)), ~~ \delta = C_2\exp(-C_3n(1+\zeta)^{-8}).
\]
\end{lemma}
\begin{proof}
	See Appendix \ref{proof:lem:mcr_rsc} for detailed proof.
\end{proof}
\begin{lemma} \label{lem:mcr_staterror}(Statistical error of MCR) Consider missing covariate regression with any fixed parameter $\bbeta^* \in \cB_0(s;p)$ and $Q_n^{MCR}(\cdot|\cdot)$ given in \eqref{mcr_Qn}. There exist constants $C_0,C_1$ such that if $n \geq C_0[\log p + \log(24/\delta)]$, then for any $\delta \in (0,1)$ and any fixed $\bbeta \in \cB(\omega\|\bbeta^*\|_2,\bbeta^*)$, we have that for 
\[
\|\nabla Q_n^{MCR}(\bbeta^*|\bbeta) - Q^{MCR}(\bbeta^*|\bbeta)\|_{\infty} \leq C_1(1+\zeta)^5\sigma\sqrt{\frac{\log p + \log (24/\delta)}{n}}
\]
with probability at least $1 - \delta$.
\end{lemma}
\begin{proof}
	See Appendix \ref{proof:lem:mcr_staterror} for detailed proof.
\end{proof}
By setting $\delta = 1/p$ in Lemma \ref{lem:mcr_staterror} immediately implies that $Q_n^{MCR}$ satisfies Condition \ref{condition:statistical error} with parameters $\Delta_n = O\left((1+\zeta)^5\sigma\sqrt{\log p/n}\right)$, $r = \omega\|\bbeta^*\|_2$ and $\delta = 1/p$.

Ensembling all pieces leads to the following guarantee about resampling version of regularized EM on missing covariate regression.
\begin{corollary} \label{thm:mcr}(Sparse Recovery in MCR) Consider the missing covariate regression with any fixed model parameter $\bbeta^* \in \cB_0(s;p)$ and the implementation of Algorithm \ref{alg:resampling} with $\ell_1$ regularization. Let quantity $\Delta$ be 
\[
\Delta = C\sigma\sqrt{\frac{\log p}{n}T}
\]
and number of samples $n$ satisfies 
\[
n/T \geq C'\max\{\sigma^2(\omega\rho)^{-1},1\}s\log p
\]
for sufficiently large constants $C,C'$. Suppose $(1+\omega)\rho \leq C_0 < 1$ and $\epsilon \leq C_1$ for sufficiently small constants $C_0, C_1$. Given any fixed $\bbeta^{(0)} \in \cB(\omega\|\bbeta^*\|_2,\bbeta^*)$, let initial regularization parameter $\lambda_{n/T}^{(0)}$ be
\[
\lambda_{n/T}^{(0)} = \frac{1}{45\sqrt{s}}\|\bbeta^{(0)}-\bbeta^*\|_2.
\]	
By choosing $\lambda^{(t)}_{n/T} = \kappa^t\lambda^{(0)}_{n/T} + \frac{1-\kappa^t}{1-\kappa}\Delta$ for any $\kappa \in [1/2,3/4]$ in Algorithm \ref{alg:resampling} leads to 
\[
\|\bbeta^{(t)} - \bbeta^*\|_2 \leq \kappa^t\|\bbeta^{(0)}-\bbeta^*\|_2 + \frac{45C\sigma}{1-\kappa}\sqrt{\frac{s\log p}{n}T}, \;\text{for all}\; t \in [T],
\]
with probability at least $1 - T/p$.
\end{corollary}
\begin{proof}
Following Theorem \ref{thm:main}, we have $\kappa^* = 5\frac{\alpha\mu\tau}{\gamma\gamma_{n/T}}$. For $\ell_2$ norm, $\alpha = 1$. Based on Lemma \ref{lem:mcr_rsc}, we have $\gamma_n = 1/9$. Following Lemma \ref{lem:mcr_smrc} and \ref{lem:mcr_grad}, we have $\gamma = 1$ and can always find sufficiently small constants $C_0,C_1$ such that $\mu \leq 10/9$ and $\tau \leq 1/100$. We thus obtain $\kappa^* \leq 1/2$. From Lemma \ref{lem:mcr_staterror}, one can check $\Delta > 3\Delta_{n/T}$ under suitable $C$. We choose $n/T \gtrsim \sigma^2(\omega\rho)^{-1}s\log p$ to make sure $\Delta \leq 3\overline{\Delta}$. With these conditions in hand, direct applying Theorem \ref{thm:main} completes the proof.
\end{proof}
By choosing $T = O(\log(n/[s\log p]))$ (for simplicity, we let $\omega = O(1)$) in Corollary \ref{thm:mcr}, the final estimation can be controlled by 
\[
\|\bbeta^{(T)} - \bbeta^*\|_2 \lesssim \sigma\sqrt{\frac{s\log p}{n}\log\left(\frac{n}{s\log p}\right)},
\] 
which is optimal up to logarithmic factor. As stated, Corollary \ref{thm:mcr} is applicable whenever $(1+\omega)\rho \leq C_0$ and $\epsilon \leq C_1$ for some constants $C_0$. In particular, we have $C_0 < 1$ that implies $\sigma > \|\bbeta^*\|_2$. Note that while low ${\rm SNR}$ is favorable in analysis, for fixed signal strength, lower ${\rm SNR}$ still leads to higher estimation error as standard (sparse) linear regression. For models with $\|\bbeta^*\|_2 \geq \sigma$, we can always add stochastic noise manually to the response $y_i$ such that $(1+\omega)\rho \leq C_0$ holds. This {\sl preprocessing trick} combined with regularized EM algorithm thus leads to sparse recovery with error $\tilde{O}(\max\{\sigma, \|\bbeta^*\|_2\}\sqrt{s\log p/n})$ for the whole range of ${\rm SNR}$.

\section{Simulations}
\label{sec:sim}
In this section, we provide the simulation results to back up our theory. Note that even our theory built on resampling technique, it's statistically efficient to use partial dataset in practice. Consequently, we test the performance of regularized EM algorithm without sample splitting (Algorithm \ref{alg}). We apply Algorithm \ref{alg} to the four latent variable models introduced in Section \ref{sec:example models}: Gaussian mixture model (GMM), mixed linear regression with sparse vector (MLR-Sparse), mixed linear regression with low rank matrix (MLR-LowRank) and missing covariate regression (MCR). We conduct two sets of experiments.

\subsection{Convergence Rate} \label{sec:sim_convergence}
We first evaluate the convergence of Algorithm \ref{alg} with good initialization $\bbeta^{(0)}$ ( particularly, we use $\bGamma^{(0)}$ to denote a matrix initial parameter for model MLR-LowRank), that is, $\|\bbeta^{(0)} - \bbeta^*\|_2 = \omega\|\bbeta^*\|_2$ for some constant $\omega$. For models with $s$-sparse parameters (GMM, MLR-Sparse and MCR), we choose the problem size to be $n = 500$, $p = 800$, $s = 5$. For MLR-LowRank, we choose $n = 600$, $p_1 = p_2 = p = 30$, rank $\theta = 3$. In addition, we set $\text{SNR} = 5$, $\omega = 0.5$ for GMM, MLR-Sparse and MLR-LowRank; we set $\text{SNR} = 0.5$, $\omega = 0.5$ and missing probability $\epsilon = 20\%$ for MCR. The initialization error we set, represented by $\omega$, for some models is larger than that provided by our theory. It's worth to note that we didn't put much effort to optimize the constant about initialization error in theory. The empirical results indicate that the practical convergence region can be much bigger than the theoretical region we proved in many settings. For a given error $\omega\|\bbeta^*\|_2$, the initial parameter $\bbeta^{(0)}$ is picked from sphere $\{\ub: \|\ub - \bbeta^*\|_2 = \omega\|\bbeta^*\|_2\}$ uniformly at random. We ran Algorithm \ref{alg} on each model for $T = 7$ iterations. We set contractive factor $\kappa = 0.7$. The choice of $\lambda_n^{(0)}$ follows Theorem \ref{thm:main}. Parameter $\Delta$ for each model is given in Table \ref{tab:delta}. For every single independent trial, we report the estimation error $\|\bbeta^{(t)} - \bbeta^*\|_2$ in each iteration and the optimization error $\|\bbeta^{(t)} - \bbeta^{(T)}\|_2$, which is the difference between $\bbeta^{(t)}$ and the final output $\bbeta^{(T)}$. We plot the log of errors over iteration $t$ in Figure \ref{fig:errortrack}. We observe that for each of the plotted 10 independent trials, estimation error converges to certain value that is much smaller than the initialization error. Moreover, the optimization error has an approximately linear convergence as predicted by our theory.

\begin{table} 
\centering
 \begin{tabular} {c|c|c|c|c} 
 \hline
 \;\; & GMM & MLR-Sparse & MLR-LowRank &  MCR \\ \hline
  $\Delta$ & $0.1(\|\bbeta^*\|_{\infty} + \sigma)\sqrt{\frac{\log p}{n}}$ &  $0.1(\|\bbeta^*\|_2 + \sigma)\sqrt{\frac{\log p}{n}}$ & $0.01(\|\bGamma^*\|_F + \sigma)\sqrt{\frac{p_1+p_2}{n}}$ & $0.2\sigma\sqrt{\frac{\log p}{n}}$ \\ \hline
 \end{tabular}
 \caption{Choice of parameter $\Delta$ in Algorithm \ref{alg}.}
 \label{tab:delta}
\end{table}

\begin{figure}[ht] 
	\centering
	\subfigure[Gaussian mixture model]{
		\centering
		\label{fig:gmm}
		\includegraphics[width=0.42\columnwidth]{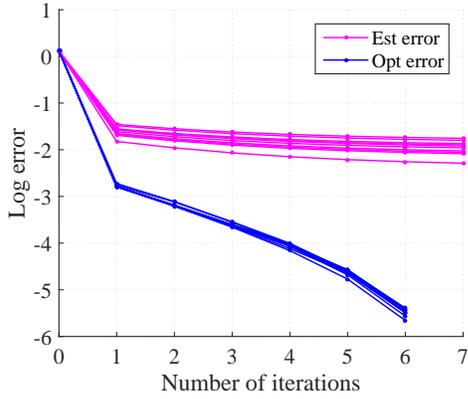}
	}
	\hfill
	\subfigure[Mixed linear regression with sparse parameter]{
		\centering
		\label{fig:mlr-sparse}
		\includegraphics[width=0.42\columnwidth]{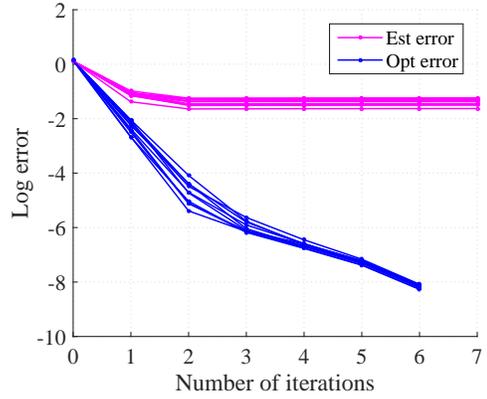}
	}
	\hfill
	\subfigure[Mixed linear regression with low rank parameter] {
		\centering
		\label{fig:mlr-matrix}
		\includegraphics[width=0.42\columnwidth]{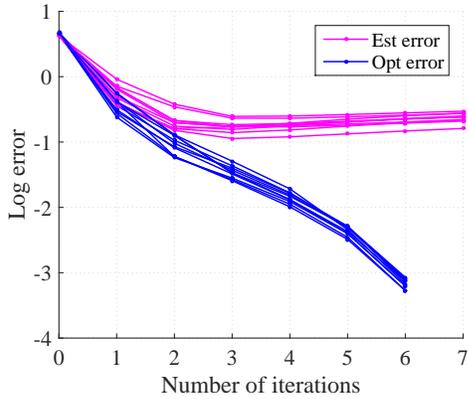}
	}
	\hfill
	\subfigure[Missing covariate regression] {
		\centering
		\label{fig:mcr}
		\includegraphics[width=0.42\columnwidth]{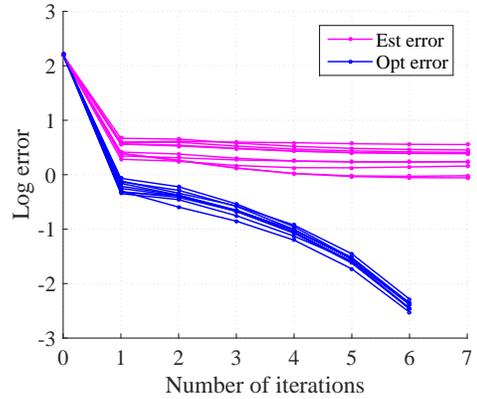}
	}
	\caption{Convergence of regularized EM algorithm. In each panel, one curve is plotted from single independent trial. Magenta lines represent overall estimation error; blue lines represent optimization error.}
	\label{fig:errortrack}
\end{figure}

\subsection{Statistical Rate} 
In the second set of experiments, we evaluate the relationship between final estimation error $\|\bbeta^{(T)} - \bbeta^*\|_2$ and problem dimensions $(n,p,s)$ or $(n,p,\theta)$ for the aforementioned latent variable models. The choices of algorithmic parameters, i.e., $\kappa$, $\Delta$ and $\lambda_n^{(0)}$, and the initial parameter follow the first set of experiments in Section \ref{sec:sim_convergence}. Moreover, we set $T = 7$ and let output $\widehat{\bbeta} = \bbeta^{(T)}$. In Figure \ref{fig:staterr}, we plot $\|\widehat{\bbeta} - \bbeta^*\|_2$ over normalized sample complexity, i.e., $n/(s\log p)$ for $s$-sparse parameter and $n/(\theta p)$ for rank $\theta$ $p$-by-$p$ parameter. In particular, we fix $s = 5$ and $\theta = 3$ for related models. We observe that the same normalized sample complexity leads to almost identical estimation error in practice, which thus supports the corresponding statistical rate established in Section \ref{sec:app}.    

\begin{figure}[ht] 
	\centering
	\subfigure[Gaussian mixture model]{
		\centering
		\includegraphics[width=0.42\columnwidth]{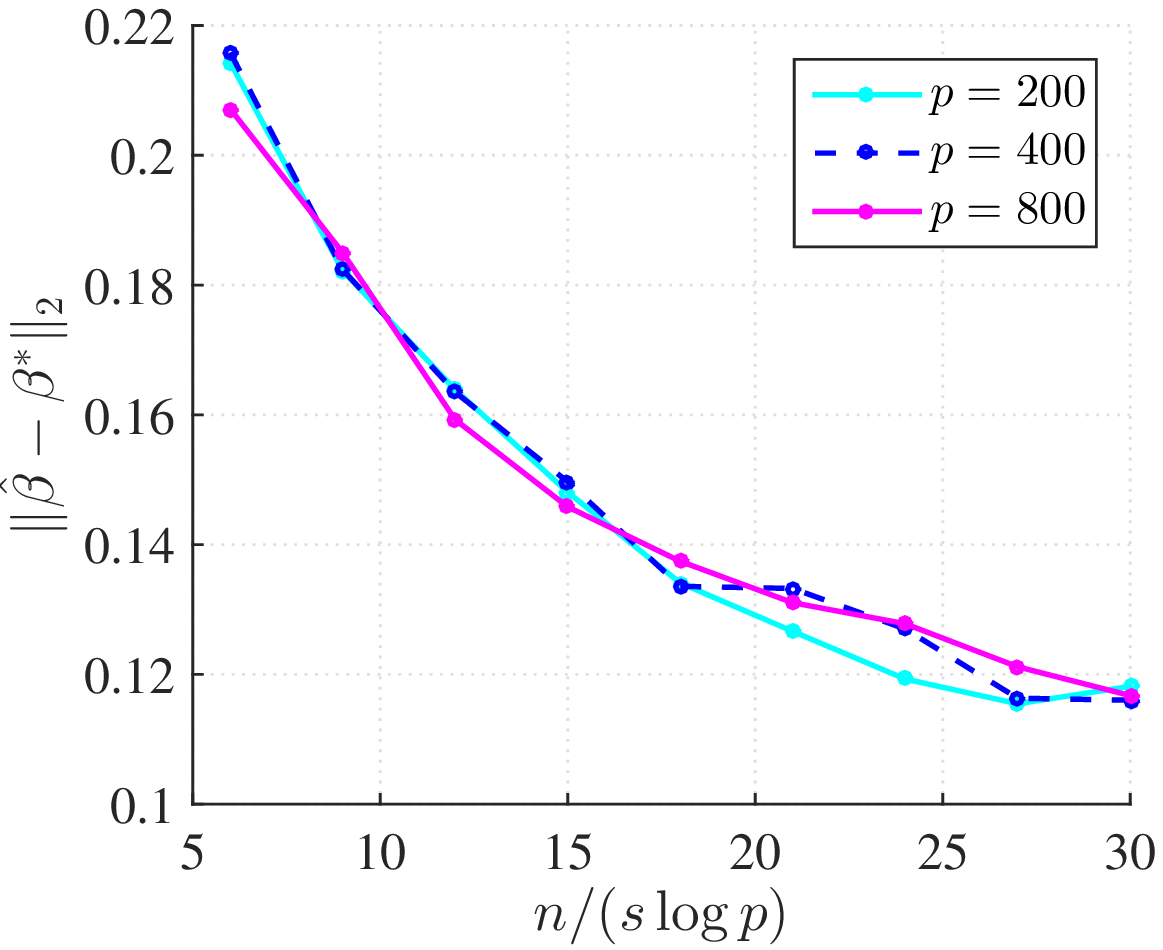}
	}
	\hfill
	\subfigure[Mixed linear regression with sparse parameter]{
		\centering
		\includegraphics[width=0.42\columnwidth]{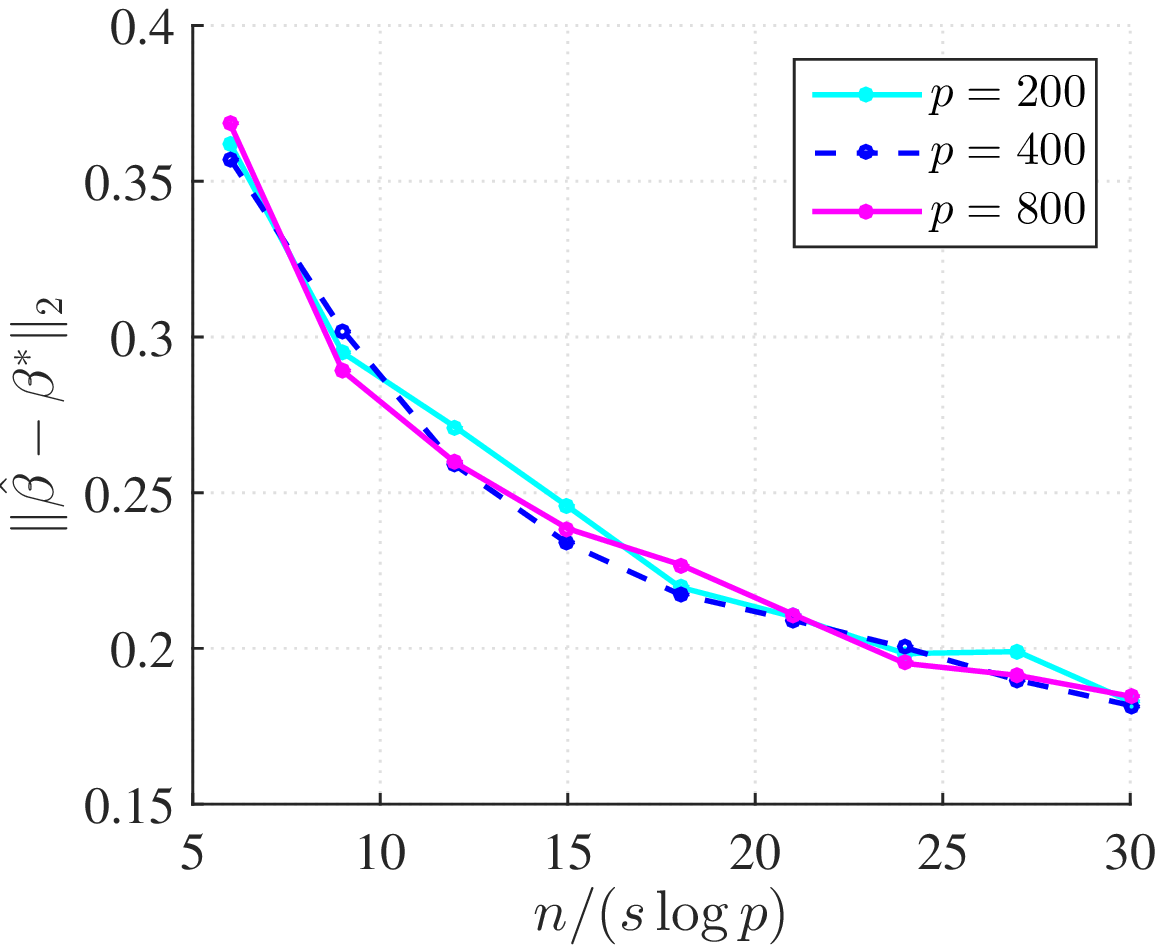}
	}
	\hfill
	\subfigure[Mixed linear regression with low rank parameter] {
		\centering
		\includegraphics[width=0.42\columnwidth]{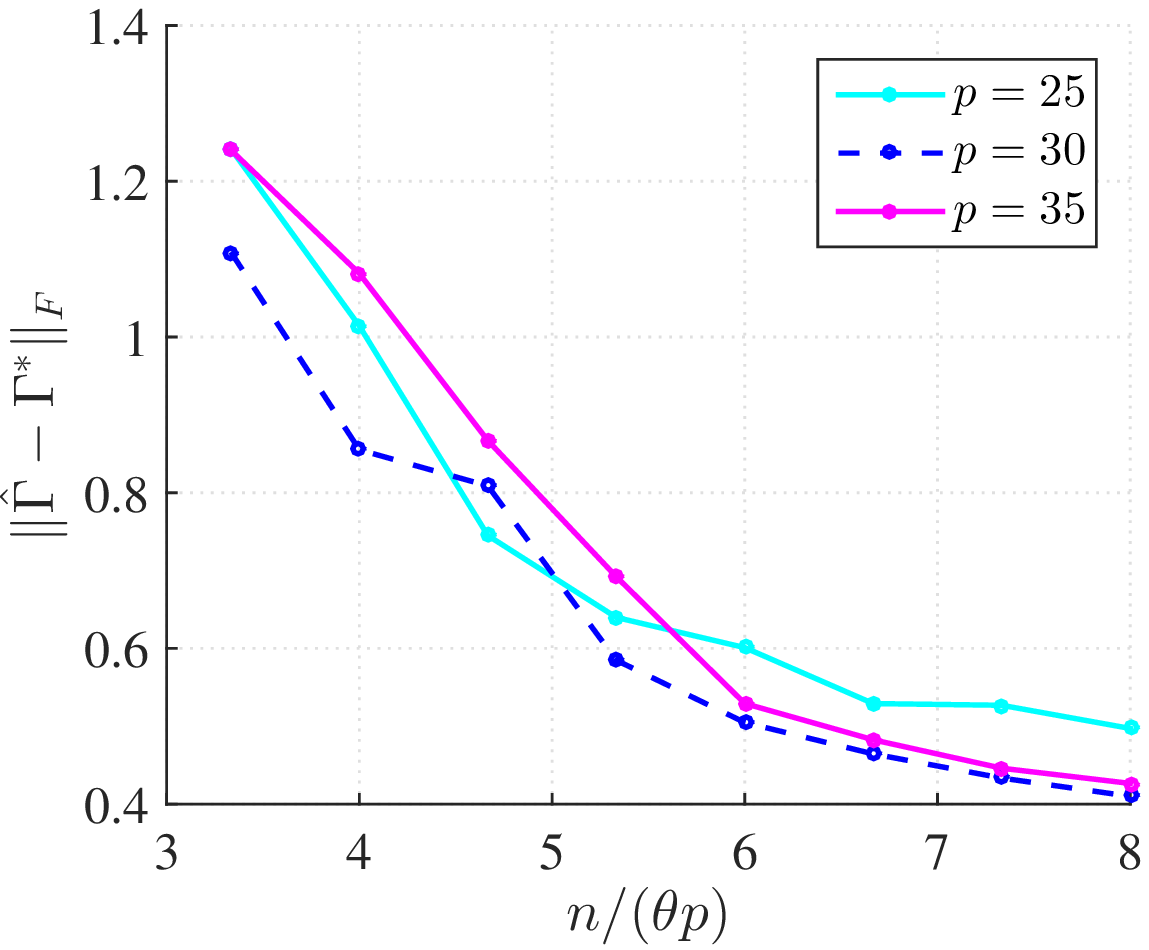}
	}
	\hfill
	\subfigure[Missing covariate regression] {
		\centering
		\includegraphics[width=0.42\columnwidth]{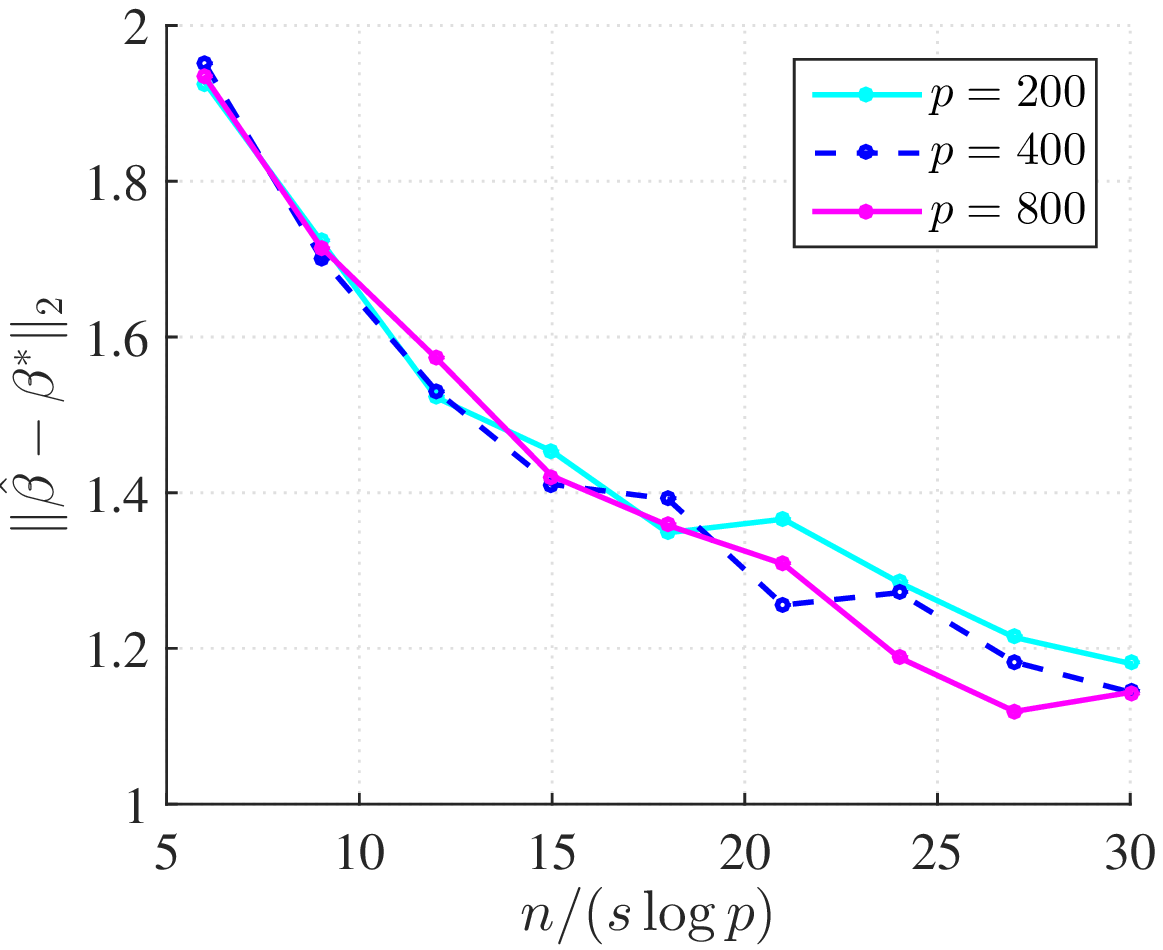}
	}
	\caption{Statistical rates of Algorithm \ref{alg} on example models. Each point is an average of $20$ independent trials.}
	\label{fig:staterr}
\end{figure}

\section{Proof of Main Result} \label{proof:thm:main}
In this section, we provide the proof of Theorem \ref{thm:main} that characterizes the computational and statistical performance of regularized EM algorithm with resampling. We first present a result which shows population EM operator $\cM:\Omega \rightarrow \Omega$ is contractive when $\tau < \gamma$.
\begin{lemma}
\label{lem:population convergence}
Suppose $Q(\cdot|\cdot)$ satisfies all the corresponding conditions stated in Theorem \ref{thm:main}. Mapping $\cM$ is contractive over  $\cB(r;\bbeta^*)$, namely
\[
\|\cM(\bbeta) - \bbeta^*\| \leq \frac{\tau}{\gamma}\|\bbeta - \bbeta^*\|,\;\forall\; \bbeta \in \cB(r;\bbeta^*).
\]
\end{lemma}
\begin{proof}
A similar result is proved in \citet{balakrishnan2014statistical}. The slight difference is that \citet{balakrishnan2014statistical} shows Lemma \ref{lem:population convergence} with $\ell_2$ norm. Extending $\ell_2$ norm to arbitrary norm is trivial, so we omit the details.
\end{proof}

Now we are ready to prove Theorem \ref{thm:main}.
\begin{proof}[Proof of Theorem \ref{thm:main}]
We first consider one iteration of Algorithm \ref{alg} and shows the relationship between $\|\bbeta^{(t)} - \bbeta^*\|$ and $\|\bbeta^{(t-1)} - \bbeta^*\|$. Recall that  
\[
\bbeta^{(t)} = \arg \max_{\bbeta' \in \Omega} Q_m(\bbeta'|\bbeta^{(t-1)}) - \lambda_m^{(t)}\cdot\cR(\bbeta').
\]
where $m = n/T$ is the number of samples in each step. We assume $\bbeta^{(t-1)} \in \cB(r;\bbeta^*)$. To simplify the notation, we drop the superscripts of $\bbeta^{(t-1)}$, $\lambda_m^{(t)}$ and denote  $\bbeta^{(t)}$ as $\bbeta^+$. From the optimality of $\bbeta^+$, we have
\begin{equation}
\label{tmp2}
 Q_m(\bbeta^+|\bbeta) - \lambda_m\cdot \cR(\bbeta^+) \geq Q_m(\bbeta^*|\bbeta) - \lambda_m \cdot \cR(\bbeta^*). 
\end{equation}
Equivalently, 
\begin{equation}
\label{eq:tmp1}
\lambda_m\cdot \cR(\bbeta^+) - \lambda_m \cdot \cR(\bbeta^*) \leq  Q_m(\bbeta^+|\bbeta) - Q_m(\bbeta^*|\bbeta).
\end{equation}
Using the fact that $Q_m(\cdot|\bbeta)$ is concave function, the right hand side of the above inequality can be bounded as
\begin{equation}
\label{tmp}
Q_m(\bbeta^+|\bbeta) - Q_m(\bbeta^*|\bbeta) \leq \big\langle \nabla Q_m(\bbeta^*|\bbeta), \bbeta^+ - \bbeta\big\rangle \leq \underbrace{\big|\big\langle \nabla Q_m(\bbeta^*|\bbeta), \bbeta^+ - \bbeta\big\rangle\big|}_{A}.
\end{equation}
A key ingredient of our proof is to bound the term $A$. Let $\Theta := \bbeta^+ - \bbeta^*$, we have
\begin{align}
\label{bound on gradient}
 \big|\big\langle \nabla Q_m(\bbeta^*|\bbeta), \bbeta^+ - \bbeta\big\rangle\big| &
=  \big|\big\langle \nabla Q_m(\bbeta^*|\bbeta) - \nabla Q(\bbeta^*|\bbeta) + \nabla Q(\bbeta^*|\bbeta), \Theta\big\rangle \big| \notag\\
& \leq  \big|\big\langle \nabla Q_m(\bbeta^*|\bbeta) - \nabla Q(\bbeta^*|\bbeta), \Theta\big\rangle\big| + \big|\big\langle \nabla Q(\bbeta^*|\bbeta) , \Theta\big\rangle \big| \notag\\
& \overset{(a)}{\leq}  \big\|Q_m(\bbeta^*|\bbeta) - \nabla Q(\bbeta^*|\bbeta)\|_{\cR^*}\cdot{\cR}(\Theta) + \big\|\nabla Q(\bbeta^*|\bbeta)\big\|_{*}\times\|\Theta\| \notag\\
& \overset{(b)}{\leq}  \Delta_m \cR(\Theta) + \alpha\big\|\nabla Q(\bbeta^*|\bbeta)\big\|\times\|\Theta\| \notag\\
& \overset{(c)}{\leq} \Delta_m \cR(\Theta) + \alpha\big\|\nabla Q(\bbeta^*|\bbeta) - \nabla Q(\cM(\bbeta)|\bbeta)\big\|\times\|\Theta\| \notag\\
& \overset{(d)}{\leq}  \Delta_m \cR(\Theta) + \alpha \mu \big\|\cM(\bbeta) - \bbeta^*\big\|\times\|\Theta\|\notag\\
& \overset{(e)}{\leq} \Delta_m \cR(\Theta) +  \frac{\alpha \mu \tau}{\gamma} \big\|\bbeta - \bbeta^*\big\|\times\|\Theta\|
\end{align}
where $(a)$ follows from Cauchy-Schwarz inequality, $(b)$ follows from the statistical error condition \ref{condition:statistical error} and the definition of $\alpha$, $(c)$ follows from the fact that $\cM(\bbeta)$ maximizes $Q(\cdot|\bbeta)$, $(d)$ follows from the smoothness condition \ref{condition: strong concavity-smooth}, $(e)$ follows from Lemma \ref{lem:population convergence}. For inequality $(c)$, note that we assume that $\cB(r;\bbeta^*) \subseteq \Omega$. From Lemma \ref{lem:population convergence}, we know that if $\bbeta \in  \cB(r;\bbeta^*)$, under condition $\tau < \gamma$, we must have $\cM(\bbeta) \in \cB(r\tau/\gamma;\bbeta^*) \subseteq \cB(r;\bbeta^*)$. Therefore $\cM(\bbeta)$ lies in the interior of $\Omega$ thus the optimality condition corresponds to $\nabla Q(\cM(\bbeta)|\bbeta) = \bm{0}$. 

Plugging \eqref{bound on gradient} back into \eqref{tmp}, we obtain
\[
Q_m(\bbeta^+|\bbeta) - Q_m(\bbeta^*|\bbeta) \leq \Delta_m \cR(\Theta) + \frac{ \alpha \mu\tau}{\gamma} \big\|\bbeta - \bbeta^*\big\|\times\|\Theta\|.
\]
Using the above result and \eqref{eq:tmp1}, we have
\begin{equation}
\label{tmp3}
\lambda_m  \cR(\bbeta^*+\Theta) - \lambda_m\cR(\bbeta^*) \leq \Delta_m \cR(\Theta) +  \frac{ \alpha \mu \tau}{\gamma} \big\|\bbeta - \bbeta^*\big\|\times\|\Theta\|.
\end{equation}

To ease notation, we use $\ub_{\cS}$ to denote the projection operator $\Pi_{\cS}(\ub)$ defined in \eqref{eq:proj}. From the decomposability of $\cR$, we have 
\begin{align*}
\cR(\bbeta^* + \Theta) - \cR(\bbeta^*) & \geq \cR(\bbeta^* + \Theta_{\overline{\cS}^{\bot}}) - \cR(\Theta_{\overline{\cS}}) - \cR(\bbeta^*) \\
& = \cR(\Theta_{\overline{\cS}^{\bot}}) - \cR(\Theta_{\overline{\cS}^{\bot}}),
\end{align*}
where the inequality is from triangle inequality and the equality is from decomposability of $\cR$. Plugging the above result back into \eqref{tmp3} yields that
\[
\lambda_m \cdot \big(\cR(\Theta_{\overline{\cS}^{\bot}}) - \cR(\Theta_{\overline{\cS}}) \big) \leq \Delta_m \cR(\Theta) +  \frac{\alpha \mu \tau}{\gamma} \big\|\bbeta - \bbeta^*\big\|\times\|\Theta\|.
\]
By assuming that $\lambda_m$ satisfies the following condition
\begin{equation}
\label{tmp6}
\lambda_m \geq 3\Delta_m + \frac{\alpha\mu\tau}{\gamma\Psi(\overline{\cS})}\|\bbeta - \bbeta^*\|,
\end{equation}
we have that
\[
\cR(\Theta_{\overline{\cS}^{\bot}}) - \cR(\Theta_{\overline{\cS}}) \leq \frac{\Delta_m}{\lambda_m} \cR(\Theta) + \frac{\alpha \mu \tau \big\|\bbeta - \bbeta^*\big\|}{\gamma\lambda_m}\|\Theta\| \leq \frac{1}{3}\cR(\Theta) + \Psi(\overline{\cS})\|\Theta\|.
\]
Plugging $\cR(\Theta) \leq \cR(\Theta_{\overline{\cS}}) + \cR(\Theta_{\overline{\cS}^{\bot}})$ into the above inequality, we  obtain 
\begin{equation}
\label{tmp4}
2\cR(\Theta_{\overline{\cS}^{\bot}}) \leq 4 \cR(\Theta_{\overline{\cS}}) + 3\Psi(\overline{\cS})\cdot\|\Theta\|.
\end{equation}
Therefore, we have shown that $\Theta$ lies in the quasi cone $\cC(\cS,\overline{\cS};\cR)$ defined in \eqref{cone}.  Recall that Condition \ref{condition:restricted strong concavity} states that for any fixed $\bbeta \in \cB(r;\bbeta^*)$, $Q_m(\cdot|\bbeta)$ is strongly concave over set $\Omega\bigcap\big(\{\bbeta^*\} + \cC(\cS,\overline{\cS};\cR)\big)$. Using this condition yields that
\begin{align}
\label{eq:tmp2}
Q_m(\bbeta^*+\Theta|\bbeta) - Q_m(\bbeta^*|\bbeta) & \leq \big\langle  \nabla Q_m(\bbeta^*|\bbeta), \Theta\big\rangle - \frac{\gamma_m}{2}\|\Theta\|^2 \notag\\
& \leq \Delta_m \cR(\Theta) + \frac{\alpha \mu \tau}{\gamma} \big\|\bbeta - \bbeta^*\big\|\times\|\Theta\| - \frac{\gamma_m}{2}\|\Theta\|^2,
\end{align}
where the second inequality follows from \eqref{bound on gradient}.

Now we turn back to optimality condition \eqref{eq:tmp1}, following which we have
\begin{align}
\label{eq:tmp5}
 Q_m(\bbeta^*+\Theta|\bbeta) - Q_m(\bbeta^*|\bbeta)  & \geq  \lambda_m \cdot \cR(\bbeta^* + \Theta) - \lambda_m\cdot \cR(\bbeta^*) \geq  -\lambda_m \cR(\Theta_{\overline{\cS}}).
\end{align}
Putting \eqref{eq:tmp2} and \eqref{eq:tmp5} together gives us
\begin{align*}
\frac{\gamma_m}{2}\|\Theta\|^2  \leq  \lambda_m\cR(\Theta_{\overline{\cS}}) + \Delta_m\cR(\Theta) + \frac{\alpha\mu\tau}{\gamma}\|\bbeta - \bbeta^*\| \times \|\Theta\|.
\end{align*}
Using  $\cR(\Theta) \leq \cR(\Theta_{\overline{\cS}^{\bot}}) + \cR(\Theta_{\overline{\cS}}) \leq (9/2)\Psi(\overline{\cS})\|\Theta\|$, we further have
\[
\frac{\gamma_m}{2}\|\Theta\|^2 \leq \lambda_m\Psi(\overline{\cS})\|\Theta\| + \frac{9}{2}\Delta_m\Psi(\overline{\cS})\|\Theta\| + \frac{\alpha\mu\tau}{\gamma}\|\bbeta - \bbeta^*\| \times \|\Theta\|.
\]
Canceling term $\|\Theta\|$ on both sides of the above inequality yields that
\begin{equation}
\label{tmp7}
\|\Theta\|  \leq 2\Psi(\overline{\cS})\frac{\lambda_m}{\gamma_m} + \frac{\Psi(\overline{\cS})}{\gamma_m}\left(9\Delta_m + 2\frac{\alpha\mu\tau}{\gamma\Psi(\overline{\cS})}\|\bbeta - \bbeta^*\|\right) \leq 5\Psi(\overline{\cS})\frac{\lambda_m}{\gamma_m}.
\end{equation}  
The last inequality follows from our assumption \eqref{tmp6}. Putting \eqref{tmp6} and \eqref{tmp7} together, we reach the conclusion that if $\bbeta^{(t-1)} \in \cB(r;\bbeta^*)$ and
\begin{equation}
\label{tmp8}
\lambda_m^{(t)} \geq 3\Delta_m + \frac{\alpha\mu\tau}{\gamma\Psi(\overline{\cS})}\|\bbeta^{(t-1)}-\bbeta^*\|,
\end{equation}
then we have
\begin{equation}
\label{tmp10}
\|\bbeta^{(t)} - \bbeta^*\| \leq 5\Psi(\overline{\cS})\frac{\lambda_m^{(t)}}{\gamma_m}. 
\end{equation}
We let $\kappa^* := 5\frac{\alpha\mu\tau}{\gamma\gamma_m}$ and assume $\kappa^* < 3/4$. Then for any $\kappa \in [\kappa^*, 3/4]$, $\Delta \geq 3\Delta_m$ suppose we set
\begin{equation}
\label{tmp9}
\lambda_m^{(t)} = \frac{1 - \kappa^t}{1 - \kappa}\Delta + \kappa^t \frac{\gamma_m}{5\Psi(\overline{\cS})}\|\bbeta^{(0)} - \bbeta^*\|
\end{equation}
for all $t \in [T]$. When $t=1$, we have $\bbeta^{(0)} \in \cB(r;\bbeta^*)$ and one can check inequality \eqref{tmp8} holds by setting $t=1$ in \eqref{tmp9}, thereby applying \eqref{tmp10} yields that
\[
\|\bbeta^{(1)} - \bbeta^*\| \leq 5\Psi(\overline{\cS})\frac{\lambda_m^{(1)}}{\gamma_m} = \frac{5\Psi(\overline{\cS})}{\gamma_m}\frac{1 - \kappa}{1-\kappa}\Delta + \kappa\|\bbeta^{(0)} - \bbeta^*\|.
\]
Now we prove Theorem \ref{thm:main} by induction. Assume that for some $t \geq 1$,
\begin{equation}
\label{tmp11}
\|\bbeta^{(t)} - \bbeta^*\| \leq \frac{5\Psi(\overline{\cS})}{\gamma_m}\frac{1 - \kappa^t}{1-\kappa}\Delta + \kappa^t\|\bbeta^{(0)} - \bbeta^*\|.
\end{equation}
Under condition $\Delta \leq 3\overline{\Delta}, \kappa \leq 3/4$, we have 
\begin{align*}
& \|\bbeta^{(t)} - \bbeta^*\| \leq \frac{15\Psi(\overline{\cS})}{\gamma_m}\frac{1 - (3/4)^t}{1-3/4}\overline{\Delta} + (3/4)^t\|\bbeta^{(0)} - \bbeta^*\| \leq  \frac{15\Psi(\overline{\cS})}{\gamma_m}\frac{1 - (3/4)^t}{1-3/4}\overline{\Delta} + (3/4)^t\cdot r \\& = (1-(3/4)^t)\cdot r + (3/4)^t\cdot r = r,
\end{align*}
where the first equality is from our definition of $\overline{\Delta}$. Consequently, we have $\bbeta^{(t)} \in \cB(r;\bbeta^*)$. Now we check that by our choice of $\lambda_m^{(t+1)}$, inequality \eqref{tmp8} holds. Note that
\begin{align*}
& 3\Delta_m + \frac{\alpha\mu\tau}{\gamma\Psi(\overline{\cS})}\|\bbeta^{(t)} - \bbeta^*\| \leq \Delta + \frac{5\alpha\mu\tau}{\gamma\gamma_m}\frac{1 - \kappa^t}{1-\kappa}\Delta + \frac{\alpha\mu\tau}{\gamma\Psi(\overline{\cS})}\kappa^t\|\bbeta^{(0)} - \bbeta^*\| \\
& \leq \Delta + \kappa \frac{1-\kappa^t}{1-\kappa}\Delta + \kappa^{t+1}\frac{\gamma_m}{5\Psi(\overline{\cS})}\|\bbeta^{(0)} - \bbeta^*\| = \frac{1-\kappa^{t+1}}{1 - \kappa}\Delta + \kappa^{t+1}\frac{\gamma_m}{5\Psi(\overline{\cS})}\|\bbeta^{(0)} - \bbeta^*\| = \lambda_m^{(t+1)},
\end{align*}
where the first inequality is from \eqref{tmp11} and the second inequality is from the fact $\kappa \geq \kappa^* = 5\frac{\alpha\mu\tau}{\gamma\gamma_m}$. Therefore \eqref{tmp8} holds for $t+1$. Then applying  \eqref{tmp10} with $t+1$ implies that 
\[
\|\bbeta^{(t+1)} - \bbeta^*\| \leq \frac{5\Psi(\overline{\cS})}{\gamma_m}\frac{1 - \kappa^{t+1}}{1-\kappa}\Delta + \kappa^{t+1}\|\bbeta^{(0)} - \bbeta^*\|.
\]  
Putting pieces together we prove that \eqref{tmp11} holds for all $t \in [T]$ when Conditions \ref{condition:restricted strong concavity} and \ref{condition:statistical error} hold in every step. Applying probabilistic union bound, we reach the conclusion. 
\end{proof}

\newpage
\bibliographystyle{plainnat}
\bibliography{high-dim-EM}

\begin{thebibliography}{26}
\providecommand{\natexlab}[1]{#1}
\providecommand{\url}[1]{\texttt{#1}}
\expandafter\ifx\csname urlstyle\endcsname\relax
  \providecommand{\doi}[1]{doi: #1}\else
  \providecommand{\doi}{doi: \begingroup \urlstyle{rm}\Url}\fi

\bibitem[Balakrishnan et~al.(2014)Balakrishnan, Wainwright, and
  Yu]{balakrishnan2014statistical}
Sivaraman Balakrishnan, Martin~J Wainwright, and Bin Yu.
\newblock Statistical guarantees for the {EM} algorithm: From population to
  sample-based analysis.
\newblock \emph{arXiv preprint arXiv:1408.2156}, 2014.

\bibitem[Cai and Zhang(2015)]{cai2015rop}
T~Tony Cai and Anru Zhang.
\newblock Rop: Matrix recovery via rank-one projections.
\newblock \emph{The Annals of Statistics}, 43\penalty0 (1):\penalty0 102--138,
  2015.

\bibitem[Cai et~al.(2011)Cai, Liu, and Luo]{cai2011constrained}
Tony Cai, Weidong Liu, and Xi~Luo.
\newblock A constrained ℓ 1 minimization approach to sparse precision matrix
  estimation.
\newblock \emph{Journal of the American Statistical Association}, 106\penalty0
  (494):\penalty0 594--607, 2011.

\bibitem[Candes and Tao(2007)]{candes2007dantzig}
Emmanuel Candes and Terence Tao.
\newblock The {D}antzig selector: statistical estimation when p is much larger
  than n.
\newblock \emph{The Annals of Statistics}, pages 2313--2351, 2007.

\bibitem[Cand{\`e}s and Plan(2011)]{candes2011tight}
Emmanuel~J Cand{\`e}s and Yaniv Plan.
\newblock Tight oracle inequalities for low-rank matrix recovery from a minimal
  number of noisy random measurements.
\newblock \emph{Information Theory, IEEE Transactions on}, 57\penalty0
  (4):\penalty0 2342--2359, 2011.

\bibitem[Cand{\`e}s and Recht(2009)]{candes2009exact}
Emmanuel~J Cand{\`e}s and Benjamin Recht.
\newblock Exact matrix completion via convex optimization.
\newblock \emph{Foundations of Computational mathematics}, 9\penalty0
  (6):\penalty0 717--772, 2009.

\bibitem[Chaganty and Liang(2013)]{chaganty2013spectral}
Arun~Tejasvi Chaganty and Percy Liang.
\newblock Spectral experts for estimating mixtures of linear regressions.
\newblock \emph{arXiv preprint arXiv:1306.3729}, 2013.

\bibitem[Chen et~al.(2014{\natexlab{a}})Chen, Sanghavi, and Xu]{Chen2014Graph}
Yudong Chen, Sujay Sanghavi, and Huan Xu.
\newblock Improved graph clustering.
\newblock \emph{Information Theory, IEEE Transactions on}, 60\penalty0
  (10):\penalty0 6440--6455, Oct 2014{\natexlab{a}}.
\newblock ISSN 0018-9448.
\newblock \doi{10.1109/TIT.2014.2346205}.

\bibitem[Chen et~al.(2014{\natexlab{b}})Chen, Yi, and
  Caramanis]{chen2014convex}
Yudong Chen, Xinyang Yi, and Constantine Caramanis.
\newblock A convex formulation for mixed regression with two components:
  Minimax optimal rates.
\newblock In \emph{Conf. on Learning Theory}, 2014{\natexlab{b}}.

\bibitem[Chen et~al.(2013)Chen, Chi, and Goldsmith]{chen2013exact}
Yuxin Chen, Yuejie Chi, and Andrea Goldsmith.
\newblock Exact and stable covariance estimation from quadratic sampling via
  convex programming.
\newblock \emph{arXiv preprint arXiv:1310.0807}, 2013.

\bibitem[Dempster et~al.(1977)Dempster, Laird, and Rubin]{dempster1977maximum}
Arthur~P Dempster, Nan~M Laird, and Donald~B Rubin.
\newblock Maximum likelihood from incomplete data via the em algorithm.
\newblock \emph{Journal of the royal statistical society. Series B
  (methodological)}, pages 1--38, 1977.

\bibitem[Jain et~al.(2013)Jain, Netrapalli, and Sanghavi]{jain2013low}
Prateek Jain, Praneeth Netrapalli, and Sujay Sanghavi.
\newblock Low-rank matrix completion using alternating minimization.
\newblock In \emph{Proceedings of the forty-fifth annual ACM symposium on
  Theory of computing}, pages 665--674. ACM, 2013.

\bibitem[Loh and Wainwright(2011)]{loh2011high}
Po-Ling Loh and Martin~J Wainwright.
\newblock High-dimensional regression with noisy and missing data: Provable
  guarantees with non-convexity.
\newblock In \emph{Advances in Neural Information Processing Systems}, pages
  2726--2734, 2011.

\bibitem[Loh and Wainwright(2012)]{loh2012corrupted}
Po-Ling Loh and Martin~J Wainwright.
\newblock Corrupted and missing predictors: Minimax bounds for high-dimensional
  linear regression.
\newblock In \emph{Information Theory Proceedings (ISIT), 2012 IEEE
  International Symposium on}, pages 2601--2605. IEEE, 2012.

\bibitem[Ma and Xu(2005)]{ma2005asymptotic}
Jinwen Ma and Lei Xu.
\newblock Asymptotic convergence properties of the em algorithm with respect to
  the overlap in the mixture.
\newblock \emph{Neurocomputing}, 68:\penalty0 105--129, 2005.

\bibitem[McLachlan and Krishnan(2007)]{mclachlan2007algorithm}
Geoffrey McLachlan and Thriyambakam Krishnan.
\newblock \emph{The EM algorithm and extensions}, volume 382.
\newblock John Wiley \& Sons, 2007.

\bibitem[Negahban et~al.(2009)Negahban, Yu, Wainwright, and
  Ravikumar]{negahban2009unified}
Sahand Negahban, Bin Yu, Martin~J Wainwright, and Pradeep~K Ravikumar.
\newblock A unified framework for high-dimensional analysis of $m$-estimators
  with decomposable regularizers.
\newblock In \emph{Advances in Neural Information Processing Systems}, pages
  1348--1356, 2009.

\bibitem[Negahban et~al.(2011)Negahban, Wainwright,
  et~al.]{negahban2011estimation}
Sahand Negahban, Martin~J Wainwright, et~al.
\newblock Estimation of (near) low-rank matrices with noise and
  high-dimensional scaling.
\newblock \emph{The Annals of Statistics}, 39\penalty0 (2):\penalty0
  1069--1097, 2011.

\bibitem[Recht et~al.(2010)Recht, Fazel, and Parrilo]{recht2010guaranteed}
Benjamin Recht, Maryam Fazel, and Pablo~A Parrilo.
\newblock Guaranteed minimum-rank solutions of linear matrix equations via
  nuclear norm minimization.
\newblock \emph{SIAM review}, 52\penalty0 (3):\penalty0 471--501, 2010.

\bibitem[St{\"a}dler et~al.(2010)St{\"a}dler, B{\"u}hlmann, and Van
  De~Geer]{stadler2010ℓ}
Nicolas St{\"a}dler, Peter B{\"u}hlmann, and Sara Van De~Geer.
\newblock L1-penalization for mixture regression models.
\newblock \emph{Test}, 19\penalty0 (2):\penalty0 209--256, 2010.

\bibitem[Tseng(2004)]{tseng2004analysis}
Paul Tseng.
\newblock An analysis of the em algorithm and entropy-like proximal point
  methods.
\newblock \emph{Mathematics of Operations Research}, 29\penalty0 (1):\penalty0
  27--44, 2004.

\bibitem[Vershynin(2010)]{vershynin2010introduction}
Roman Vershynin.
\newblock Introduction to the non-asymptotic analysis of random matrices.
\newblock \emph{arXiv preprint arXiv:1011.3027}, 2010.

\bibitem[Wainwright(2014)]{wainwright2014structured}
Martin~J Wainwright.
\newblock Structured regularizers for high-dimensional problems: Statistical
  and computational issues.
\newblock \emph{Annual Review of Statistics and Its Application}, 1:\penalty0
  233--253, 2014.

\bibitem[Wang et~al.(2014)Wang, Gu, Ning, and Liu]{wang2014high}
Zhaoran Wang, Quanquan Gu, Yang Ning, and Han Liu.
\newblock High dimensional expectation-maximization algorithm: Statistical
  optimization and asymptotic normality.
\newblock \emph{arXiv preprint arXiv:1412.8729}, 2014.

\bibitem[Wu(1983)]{wu1983convergence}
C.F.Jeff Wu.
\newblock On the convergence properties of the em algorithm.
\newblock \emph{The Annals of statistics}, pages 95--103, 1983.

\bibitem[Yi et~al.(2013)Yi, Caramanis, and Sanghavi]{yi2013alternating}
Xinyang Yi, Constantine Caramanis, and Sujay Sanghavi.
\newblock Alternating minimization for mixed linear regression.
\newblock \emph{arXiv preprint arXiv:1310.3745}, 2013.

\end{thebibliography}
\clearpage
\appendix

\section{Proofs about Gaussian Mixture Model}

\subsection{Proof of Lemma \ref{lem:gaussian_mixture_self_consistency}} \label{proof:lem:gaussian_mixture_self_consistency}

In this example, we have 
\[
\cM(\bbeta^*) = 2\mathbb{E}\left[w(Y;\bbeta^*)Y\right] = 2\mathbb{E}\left[\frac{1}{1+\exp(-\frac{2}{\sigma^2}\langle Z \cdot \bbeta^* + W,\bbeta^* \rangle)}(Z\cdot \bbeta^* + W)\right],
\]
where $W \sim \cN(0,\sigma^2)$ and $Z$ has Rademacher distribution over $\{-1,1\}$. Due to the rotation invariance of Gaussianity, without loss of generality, we assume $\bbeta^* = A\e_1$. It's easy to check $\supp(\cM(\bbeta^*)) = \{1\}$. Moreover, the first coordinate of $\cM(\bbeta^*)$ takes form
\[
(\cM(\bbeta^*))_1 = 2\mathbb{E}\left[\frac{1}{1+\exp(-\frac{2}{\sigma^2}(AZ+W_1))}(AZ+W_1)\right] = A,
\]
where the last equality follows by the substitution $X = W_1, Z = Z, \gamma = 0, a = A$ in Lemma \ref{lem:key_expectation}. Therefore, $\cM(\bbeta^*) = \bbeta^*$.
\subsection{Proof of Lemma \ref{lem:gaussian_mixture_rsc}} \label{proof:lem:gaussian_mixture_rsc}

Although Condition \ref{condition:restricted strong concavity} is a stochastic condition, for Gaussian mixture model, particularly it is satisfied deterministically. Note that
\[
Q^{GMM}_n(\bbeta'|\bbeta) = -\frac{1}{2n}\sum_{i=1}^n\left[w(\yb_i;\bbeta)\|\yb_i - \bbeta'\|_2^2 + (1 - w(\yb_i;\bbeta))\|\yb_i + \bbeta'\|_2^2\right].	
\]
We have that for any $\bbeta',\bbeta \in \RR^p$, $\nabla^2Q_n^{GMM}(\bbeta'|\bbeta) = -\Ib_p$, which implies that $Q_n^{GMM}(\bbeta'|\bbeta)$ is strongly concave with parameter $1$. Consequently, Condition \ref{condition:restricted strong concavity} holds with $\gamma_n = 1$.
\subsection{Proof of Lemma \ref{lem:gaussian_mixture_staterr}}
\label{proof:lem:gaussian_mixture_staterr}

Note that $\cR^*$ is $\|\cdot\|_{\infty}$ in this example. Following the specific formulations of $Q^{GMM}_n(\cdot|\cdot)$ and $Q^{GMM}(\cdot|\cdot)$ in \eqref{gaussian_mixutre_Qn} and \eqref{gaussian_mixtrue_Q}, we have
\[
\nabla Q^{GMM}_n(\bbeta^*|\bbeta) - \nabla Q^{GMM}(\bbeta^*|\bbeta) = -\frac{1}{n}\sum_{i=1}^n\yb_i + \frac{2}{n}\sum_{i=1}^{n}w(\yb_i;\bbeta)\yb_i - 2\mathbb{E}\left[w(Y;\bbeta)Y\right].
\]
Therefore,
\[
\left\|\nabla Q^{GMM}_n(\bbeta^*|\bbeta) - \nabla Q^{GMM}(\bbeta^*|\bbeta)\right\|_{\infty} \leq  \underbrace{\left\|\frac{1}{n}\sum_{i=1}^n\yb_i\right\|_{\infty}}_{(a)} +  \underbrace{\left\|\frac{2}{n}\sum_{i=1}^{n}w(\yb_i;\bbeta)\yb_i - 2\mathbb{E}\left[w(Y;\bbeta)Y\right]\right\|_{\infty}}_{(b)}
\]
Next we bound the two terms $(a)$ and $(b)$ respectively. 

\noindent{\bf Term $(a)$.}  Let $\bzeta := \frac{1}{n}\sum_{i=1}^n\yb_i$. Let $\yb_i = (y_{i,1},\ldots,y_{i,p})^{\top}$ for all $i \in [n]$. Consider the $j$-th coordinate $\zeta_j$ of $\bzeta$, we have 
\[
\zeta_j = \frac{1}{n}\sum_{i=1}^{n}y_{i,j}.
\]
Note that $\{y_{i,j}\}_{i=1}^{n}$ are independent copies of random variable $Y_j$ that is
\begin{equation}
\label{eq:Yj}
Y_j = Z\cdot \beta^*_{j} + V,
\end{equation}
where $Z$ is Rademacher random variable taking values in $\{-1,1\}$ and $V$ has distribution $\cN(0,\sigma^2)$. Since $Z\cdot\beta^*_j$ and $V$ are both sub-Gaussian random variables with norm $\|Z\cdot\beta^*_j\|_{\psi_2} \leq |\beta^*_j|$ and $\|V\|_{\psi_2} \lesssim \delta$. Following the rotation invariance sub-Gaussian random variables (e.g., Lemma 5.9 in \cite{vershynin2010introduction}), we have that
\[
\|Y_j\|_{\psi_2} \lesssim \sqrt{\|Z\cdot\beta^*_{j}\|^2_{\psi_2} + \|V\|^2_{\psi_2}} \lesssim \sqrt{\|\bbeta^*\|_{\infty}^2 + \sigma^2}.
\]
Following the standard sub-Gaussian concentration argument in Lemma \ref{lem:sub-Gaussian_sum}, there exists some constant $C$ such that for any $j \in [p]$ and all $t \geq 0$,
\[
\Pr\left(\big|\zeta_j\big| \geq t\right) \leq e\cdot\exp\left(-\frac{Cnt^2}{\|\bbeta^*\|_{\infty}^2 + \sigma^2}\right).
\]
Then by applying union bound, we have
\[
\Pr\left(\sup_{j \in [p]}\big|\zeta_j\big| \geq t\right) \leq pe\cdot\exp\left(-\frac{Cnt^2}{\|\bbeta^*\|_{\infty}^2 + \sigma^2}\right).
\]
Setting the right hand side to be $\delta$, we have that, with probability at least $1 - \delta/2$,
\begin{equation} \label{eq:tmp_gmm1}
\left\|\frac{1}{n}\sum_{i=1}^n\yb_i\right\|_{\infty} \lesssim (\|\bbeta^*\|_{\infty} + \delta)\sqrt{\frac{\log p + \log(2e/\delta)}{n}}.
\end{equation}

\noindent{\bf Term $(b)$.} Now let $\bzeta := \frac{2}{n}\sum_{i=1}^{n}w(\yb_i;\bbeta)\yb_i - 2\mathbb{E}\left[w(Y;\bbeta)Y\right]$. We also consider the $j$-th coordinate $\zeta_j$ of $\bzeta$, which takes form
\[
\zeta_j = \frac{2}{n}\sum_{i=1}^n \bigg\{w(\yb_i;\bbeta)y_{i,j} - \mathbb{E}(w(Y;\bbeta)Y_j)\bigg\}.
\] 
Note that $w(\yb_i;\bbeta)y_{i,j} - \mathbb{E}(w(Y;\bbeta)Y_j), i = 1,\ldots,n$ are independent copies of random variable $w(Y;\bbeta)Y_j - \mathbb{E}(w(Y;\bbeta)Y_j)$ where $Y_j$ is given in \eqref{eq:Yj}. We have shown that $Y_j$ is sub-Gaussian random variable. Note that $w(Y;\bbeta)$ is random variable taking values in $[0,1]$. We thus always have
\[
\Pr\left(|w(Y;\bbeta)Y_j| \geq t\right) \leq \Pr(|Y_j| > t) \leq \exp(1 - Ct^2/\|Y_j\|^2_{\psi_2}). 
\]
Using the equivalent properties of sub-Gaussian (see Lemma 5.5 in \citet{vershynin2010introduction}) , we conclude that $w(Y;\bbeta)Y_j$ is sub-Gaussian random variable with norm $\|w(Y;\bbeta)Y_j\|_{\psi_2} \leq \|Y_j\|_{\psi_2} \lesssim \sqrt{\|\bbeta^*\|_{\infty}^2 + \sigma^2}$. Following Lemma \ref{lem:sub-Gaussian_centering}, we have $\|w(Y;\bbeta)Y_j - \mathbb{E}\left[w(Y;\bbeta)Y_j\right]\|_{\psi_2} \leq 2\|w(Y;\bbeta)Y_j\|_{\psi_2}$. Using concentration result in Lemma \ref{lem:sub-Gaussian_sum}  yields that for any $j \in [p]$ and some constant $C$,
\[
\Pr\left(|\zeta_j| \geq t\right) = \Pr\bigg\{\bigg| \frac{2}{n}\sum_{i=1}^n w(\yb_i;\bbeta)y_{i,j} - \mathbb{E}(w(Y;\bbeta)Y) \bigg| > t \bigg\} \leq e\cdot\exp\left(-\frac{Cnt^2}{\|\bbeta^*\|_{\infty}^2 + \sigma^2}\right).
\] 
Applying union bound over $p$ coordinates, we have
\[
\Pr\left(\sup_{j \in [p]} |\zeta_j| > t\right) \leq pe\cdot\exp\left(-\frac{Cnt^2}{\|\bbeta^*\|_{\infty}^2 + \sigma^2}\right),
\]
which implies that, with probability at least $1 - \delta/2$,
\begin{equation} \label{eq:tmp_gmm2}
\left\|\frac{2}{n}\sum_{i=1}^{n}w(\yb_i;\bbeta)\yb_i - 2\mathbb{E}\left[w(Y;\bbeta)Y\right]\right\|_{\infty} \lesssim (\|\bbeta^*\|_{\infty} + \sigma)\sqrt{\frac{\log p + \log(2e/\delta)}{n}}.
\end{equation}
Putting \eqref{eq:tmp_gmm1} and \eqref{eq:tmp_gmm2} together completes the proof.

\section{Proofs about Mixed Linear Regression}

\subsection{Proof of Lemma \ref{lem:mlr:self_consistency}} \label{proof:lem:mlr:self_consistency} 
In this example, we have
\[
\cM(\bbeta^*) = 2\mathbb{E}\left[w(Y,X;\bbeta^*)YX\right] = 2\mathbb{E}\left[\frac{1}{1+\exp(-\frac{2(\langle X,Z\cdot\bbeta^*\rangle + W) \langle X,\bbeta^* \rangle}{\sigma^2})}(Z\cdot\bbeta^* + W)X\right],
\]
where $X \sim \cN(\bm{0},\Ib_p), W \sim \cN(0,\sigma^2)$, $Z$ has Rademacher distribution. Due to the rotation invariance of Gaussianity, without loss of generality, we can assume $\bbeta^* = A\e_1$. It's easy to check $\supp(\cM(\bbeta^*)) = \{1\}$. Moreover,
\[
(\cM(\bbeta^*))_1 = 2\mathbb{E}\left[\frac{1}{1+\exp(-\frac{2}{\sigma^2}(AZX_1+W)AX_1)}(AZX_1^2+X_1W)\right] = \mathbb{E}(AX_1^2) = A,
\]
where the second inequality follows by the substitution $X = W, Z = Z, \gamma = 0, a = AX_1$ in Lemma \ref{lem:key_expectation}. We thus have $\cM(\bbeta^*) = \bbeta^*$.
\subsection{Proof of Lemma \ref{lem:mlr:gradient_stability}}
\label{proof:mlr:gradient_stability}

Recall that we hope to find $\tau$ such that for any $\bbeta \in \cB(r;\bbeta^*)$
\[
\|\nabla Q^{MLR}(\cM(\bbeta)|\bbeta) - \nabla Q^{MLR}(\cM(\bbeta)|\bbeta^*)\|_2 \leq \tau \|\bbeta - \bbeta^*\|_2.
\]
In this example, we have
\[
\cM(\bbeta) = 2\mathbb{E}\left[w(Y,X;\bbeta)YX\right],
\]	
and
\[
\nabla Q^{MLR}(\bbeta'|\bbeta) = 2\mathbb{E}\left[w(Y,X;\bbeta)YX\right] - \bbeta'.
\]
Therefore,
\begin{align*}
& \nabla Q^{MLR}(\cM(\bbeta)|\bbeta) - \nabla Q^{MLR}(\cM(\bbeta)|\bbeta^*) \\ 
& = 2\mathbb{E}\left[w(Y,X;\bbeta)YX\right] - 2\mathbb{E}\left[w(Y,X;\bbeta^*)YX\right] = 2\mathbb{E}\left[w(Y,X;\bbeta)YX\right] - \bbeta^*,
\end{align*}
where the last equality is from the self consistent property of $Q^{MLR}(\cdot|\cdot)$.
Due to the rotation invariance of Gaussianity, without loss of generality, we assume $\bbeta^* = A\e_1, \bbeta = (1+\epsilon_1)A\e_1 + \epsilon_2 A\e_2$, where $A = \|\bbeta^*\|_2$, $\|\bbeta - \bbeta^*\|_2 = A\sqrt{\epsilon_1^2 + \epsilon_2^2}$. Let random vector $T$ be
\[
T := w(Y,X;\bbeta)YX - \frac{1}{2}\bbeta^*.
\]
Note that for any $\bbeta \in \RR^p$,
\[
w(Y,X;\bbeta) = \frac{\exp(-\frac{(Y-\langle X,\bbeta\rangle)^2}{2\sigma^2})}{\exp(-\frac{(Y-\langle X,\bbeta\rangle)^2}{2\sigma^2})+ \exp(-\frac{(Y+\langle X,\bbeta\rangle)^2}{2\sigma^2})} = \frac{1}{1+ \exp(-\frac{2Y\langle X,\bbeta\rangle}{\sigma^2})},
\]
thereby
\begin{align*}
& T = \frac{1}{1+ \exp(-\frac{2Y\langle X,\bbeta\rangle}{\sigma^2})}YX - \frac{1}{2}\bbeta^* \\
& =  \frac{1}{1+ \exp(-\frac{2(ZAX_1 + W)(A(1+\epsilon_1)X_1 + \epsilon_2X_2)}{\sigma^2})}(ZAX_1 + W)X - \frac{1}{2}A\e_1,
\end{align*}
where $Z$ is Rademacher random variable taking values in $\{-1,1\}$, $W$ is stochastic noise with distribution $ \cN(0,\sigma^2)$, $X_1$ and $X_2$ are the first two coordinates of $X$. It's easy to note that $\mathbb{E}\left[T_i\right] = 0$ for $i = 3,\ldots,p$. We focus on characterizing the first two coordinates $T_1,T_2$ of $T$. 

\vspace{0.1in}
\noindent {\bf Coordinate $T_1$.} \\
First, we compute the expectation of $T_1$. Particularly we let $\gamma = \epsilon_1 + \epsilon_2X_2/X_1$. Then we have
\begin{align} \label{eq:tmp9}
& \big|\mathbb{E}\left[T_1\right]\big| = \left|\mathbb{E}\left[\frac{X_1( W + ZAX_1)}{1+ \exp(-\frac{2AX_1(1 + \gamma)}{\sigma^2}(W + ZAX_1))} - \frac{1}{2}AX_1^2\right]\right| \notag\\
& \leq \mathbb{E}\left[|X_1|\cdot\left|\frac{( W + ZAX_1)}{1+ \exp(-\frac{2AX_1(1 + \gamma)}{\sigma^2}(W + ZAX_1))} - \frac{1}{2}AX_1\right|\right] \notag\\
& = \mathbb{E}_{X_1,X_2}\left\{|X_1|\cdot \mathbb{E}_{W,Z}\left[\left|\frac{( W + ZAX_1)}{1+ \exp(-\frac{2AX_1(1 + \gamma)}{\sigma^2}(W + ZAX_1))} - \frac{1}{2}AX_1\right|\right]\right\} \notag\\
& \leq \mathbb{E}_{X_1,X_2} \left[ |X_1| \cdot \min\bigg\{ \frac{1}{2}A\cdot|X_1\gamma|\cdot\exp(\frac{\gamma^2(AX_1)^2-(AX_1)^2}{2\sigma^2}),~\frac{\sigma}{\sqrt{2\pi}} + A|X_1|\bigg\} \right],
\end{align}
where the last inequality follows from Lemma \ref{lem:key_expectation} by replacing the parameters $(X,Z,a,\gamma)$ in the statement with $(W,Z,AX_1,\gamma)$. Let event $\cE$ be $\cE : = \{\gamma^2 \leq 0.9\}$. Computing the expectation in \eqref{eq:tmp9} conditioning on $\cE$ and $\cE^c$ yields that 
\begin{align}\label{eq:tmp15}
\big|\mathbb{E}\left[T_1\right]\big| \leq & \mathbb{E}\left[ \frac{1}{2} |\gamma|AX_1^2\exp(\frac{\gamma^2(AX_1)^2 - (AX_1)^2}{2\sigma^2})  ~\bigg|~ \cE\right]\cdot \Pr(\cE) \notag\\
& + \mathbb{E}\left[ \frac{\sigma|X_1|}{\sqrt{2\pi}} + AX_1^2 ~\bigg|~ \cE^c\right]\cdot\Pr(\cE^c).
\end{align}
We bound the two terms on the right hand side of the above inequality respectively. For the first term we have
\begin{align} \label{eq:tmp11}
 & \mathbb{E}\left[ \frac{1}{2} |\gamma|AX_1^2\exp(\frac{\gamma^2(AX_1)^2 - (AX_1)^2}{2\sigma^2})  ~\bigg|~ \cE\right]\cdot \Pr(\cE) \leq  \mathbb{E}\left[ \frac{1}{2} |\gamma|AX_1^2\exp(\frac{ -  (AX_1)^2}{20\sigma^2})  ~\bigg|~ \cE\right]\cdot\Pr(\cE) \notag\\
 & \leq \mathbb{E}\left[ \frac{1}{2}|\gamma| AX_1^2\exp(\frac{ -  (AX_1)^2}{20\sigma^2})\right] \leq \mathbb{E}\left[\frac{1}{2}A\left(|\epsilon_1|\cdot X_1^2+ |\epsilon_2X_1X_2|\right)\exp(-\frac{1}{20}\rho^2X_1^2)\right] \notag \\ 
  &  = \frac{1}{2}A\frac{|\epsilon_1|}{(1+0.1\rho^2)^{3/2}} + \frac{1}{\pi}A\frac{|\epsilon_2|}{1+0.1\rho^2} \leq \frac{1}{2}A\frac{1}{1+0.1\rho^2}(|\epsilon_1| + |\epsilon_2|),
\end{align}
where the third inequality is from $\|\bbeta^*\|_2/\sigma \geq \rho$. For the second term in \eqref{eq:tmp15}, first note that 
\[
\sqrt{\epsilon_1^2 + \epsilon_2^2} \leq \frac{\|\bbeta - \bbeta^*\|_2}{\|\bbeta^*\|_2} \leq \omega \leq 1/4,
\]
thereby
\[
|\gamma| \leq |\epsilon_1| + |\epsilon_2|\cdot|X_2/X_1| \leq 1/4 + |\epsilon_2|\cdot|X_2/X_1|.
\]
We define event $\cE' := \{X_2^2/X_1^2 \geq (2.1\epsilon_2^2)^{-1}\}$. Note that $\cE^c = \{\gamma^2 \geq 0.9\}$, we thus have $\cE^c \subseteq \cE'$, i.e., the occurrence of $\cE^c$ must lead to the occurrence of $\cE'$. For the second term in \eqref{eq:tmp15}, we have
\begin{align} \label{eq:tmp10}
& \mathbb{E}\left[ \frac{\sigma|X_1|}{\sqrt{2\pi}} + AX_1^2 ~\bigg|~ \cE^c\right]\cdot\Pr(\cE^c)  \leq \mathbb{E}\left[ \frac{\sigma|X_1|}{\sqrt{2\pi}} + AX_1^2 ~\bigg|~ \cE'\right]\cdot\Pr(\cE') \notag \\
& \leq \mathbb{E}\left[ \frac{\sigma|X_1|}{\sqrt{2\pi}} + \sqrt{2.1\epsilon_2^2}A|X_1X_2| ~\bigg|~ \cE'\right]\cdot\Pr(\cE') \\
& =  \frac{\sigma}{\pi}\left[1 - \sqrt{\frac{1}{1+2.1\epsilon_2^2}}\right] + \sqrt{2.1\epsilon_2^2}A\frac{2}{\pi}\frac{2.1\epsilon_2^2}{1+2.1\epsilon_2^2}  \leq \frac{\sqrt{2.1}\sigma}{\pi}|\epsilon_2| + \frac{2\sqrt{2.1}^3}{\pi}A|\epsilon_2|^3,
\end{align}
where the equality is from Lemma \ref{lem:angleExpectation} by setting $C$ in the statement to be $\sqrt{2.1\epsilon_2^2}$.


Putting \eqref{eq:tmp11} and \eqref{eq:tmp10} together, we have
\begin{equation} \label{eq:T1}
|\mathbb{E}\left[T_1\right]| \leq \frac{1}{2}A\frac{1}{1+0.1\rho^2}(|\epsilon_1| + |\epsilon_2|) + \frac{\sqrt{2.1}\sigma}{\pi}|\epsilon_2| + \frac{2\sqrt{2.1}^3}{\pi}A|\epsilon_2|^3.
\end{equation}

\vspace{0.1in}
\noindent {\bf Coordinate $T_2$.} \\
Now we turn to the second coordinate $T_2$. Using $\mathbb{E}\left[X_1X_2\right] = 0$, we have
\begin{align*}
& \big|\mathbb{E}\left[T_2\right]\big| = \left|\mathbb{E}\left[\frac{X_2( W + ZAX_1)}{1+ \exp(-\frac{2AX_1(1 + \gamma)}{\sigma^2}(W + ZAX_1))} - \frac{1}{2}AX_1X_2\right]\right| \notag\\
& \leq \mathbb{E}\left[|X_2|\cdot\left|\frac{( W + ZAX_1)}{1+ \exp(-\frac{2AX_1(1 + \gamma)}{\sigma^2}(W + ZAX_1))} - \frac{1}{2}AX_1\right|\right]. 
\end{align*}
Similar to \eqref{eq:tmp9}, using Lemma \ref{lem:key_expectation} leads to
\begin{align*}
& \big|\mathbb{E}\left[T_2\right]\big| \leq \mathbb{E} \left[ |X_2| \cdot \min\bigg\{ \frac{1}{2}A\cdot|X_1\gamma|\cdot\exp(\frac{\gamma^2(AX_1)^2-(AX_1)^2}{2\sigma^2}),~\frac{\sigma}{\sqrt{2\pi}} + A|X_1|\bigg\} \right] \\
& \leq  \mathbb{E}\left[ \frac{1}{2} A|\gamma|\cdot|X_1X_2|\exp(\frac{\gamma^2(AX_1)^2 - (AX_1)^2}{2\sigma^2})  ~\bigg|~ \cE\right]\cdot \Pr(\cE) \\
& \hspace{2in} + ~\mathbb{E}\left[ \frac{\sigma|X_2|}{\sqrt{2\pi}} + A|X_1X_2| ~\bigg|~ \cE^c\right]\cdot\Pr(\cE^c).
\end{align*}
We bound the two terms in the right hand side of the above inequality respectively. For the first term, we have
\begin{align} \label{eq:tmp13}
& \mathbb{E}\left[ \frac{1}{2} A|\gamma|\cdot|X_1X_2|\exp(\frac{\gamma^2(AX_1)^2 - (AX_1)^2}{2\sigma^2})  ~\bigg|~ \cE\right]\cdot \Pr(\cE) \notag \\
& \leq \mathbb{E}\left[ \frac{1}{2} A|\gamma|\cdot|X_1X_2|\exp(\frac{ - 0.1(AX_1)^2}{2\sigma^2}) ~\big|~\cE\right]\cdot\Pr(\cE) \leq \mathbb{E}\left[ \frac{1}{2} A|\gamma|\cdot|X_1X_2|\exp(\frac{ - 0.1(AX_1)^2}{2\sigma^2}) \right]\notag \notag \\
& \leq \mathbb{E}\left[ \frac{1}{2} A \left(|\epsilon_1X_1X_2| + |\epsilon_2|X_2^2\right)\exp(-\frac{1}{20}\rho^2X_1^2)\right] = \frac{1}{\pi}A\frac{|\epsilon_1|}{1+0.1\rho^2} + \frac{1}{2}A\frac{|\epsilon_2|}{\sqrt{1+0.1\rho^2}}
\end{align}
For the second term, recall that event $\cE'$ is defined as $\{X_2^2/X_1^2 \geq (2.1\epsilon_2^2)^{-1}\}$, we have
\begin{align} \label{eq:tmp14}
& \mathbb{E}\left[ \frac{\sigma|X_2|}{\sqrt{2\pi}} + A|X_1X_2| ~\bigg|~ \cE^c\right]\cdot\Pr(\cE^c)  \leq \mathbb{E}\left[ \frac{\sigma|X_2|}{\sqrt{2\pi}} + A|X_1X_2| ~\bigg|~ \cE'\right]\cdot\Pr(\cE') \notag \\
& = \frac{\sigma}{\pi}\frac{\sqrt{2.1}\epsilon_2}{\sqrt{1+2.1\epsilon_2^2}} + \frac{2A}{\pi}\frac{2.1\epsilon_2^2}{1+2.1\epsilon_2^2} \leq \frac{\sqrt{2.1}\sigma}{\pi}|\epsilon_2| + \frac{4.2A}{\pi}\epsilon_2^2.
\end{align}
where the equality follows from Lemma \ref{lem:angleExpectation} by setting $C$ in the statement to be $\sqrt{2.1\epsilon_2^2}$. Putting \eqref{eq:tmp13} and \eqref{eq:tmp14} together, we have
\begin{equation}
\label{eq:T2}
|\mathbb{E}\left[T_2\right]| \leq \frac{1}{\pi}A\frac{|\epsilon_1|}{1+0.1\rho^2} + \frac{1}{2}A\frac{|\epsilon_2|}{\sqrt{1+0.1\rho^2}} + \frac{\sqrt{2.1}\sigma}{\pi}|\epsilon_2| + \frac{4.2A}{\pi}\epsilon_2^2.
\end{equation}

Now based on \eqref{eq:T1} and \eqref{eq:T2}, we conclude that
\begin{align*}
& \mathbb{E}\left[\|T\|_2\right] = \mathbb{E}\left[\sqrt{T_1^2 + T_2^2}\right] \leq \mathbb{E}\left[|T_1| +  |T_2|\right] \\
&  \leq A\frac{1}{\sqrt{1+0.1\rho^2}}(|\epsilon_1| + |\epsilon_2|) + \frac{\sqrt{2.1}\sigma}{\pi}|\epsilon_2| + \frac{2\sqrt{2.1}^3}{\pi}A|\epsilon_2|^3 + \frac{\sqrt{2.1}\sigma}{\pi}|\epsilon_2| + \frac{4.2A}{\pi}\epsilon_2^2 \\
& \leq A\left(\frac{1}{\sqrt{1+0.1\rho^2}}(|\epsilon_1| + |\epsilon_2|) + |\epsilon_2|/\rho + 1.83\omega |\epsilon_2|\right) \\
& \leq A(|\epsilon_1|+|\epsilon_2|)\cdot\left( \frac{4.2}{\rho} + 1.83\omega\right) \leq 2A\sqrt{\epsilon_1^2+\epsilon_2^2}\cdot\left( \frac{4.2}{\rho} + 1.83\omega\right) \\
& = 2\left(  \frac{4.2}{\rho} + 1.83\omega\right)\|\bbeta - \bbeta^*\|_2.
\end{align*}
Note that  $\nabla Q^{MLR}(\cM(\bbeta)|\bbeta) - \nabla Q^{MLR}(\cM(\bbeta)|\bbeta^*) = 2T$, thereby we conclude that for any  $\omega \leq 1/4$, $Q^{MLR}(\cdot|\cdot)$ satisfies gradient stability condition over $\cB(\omega\|\bbeta^*\|_2;\bbeta^*)$ with parameter 
\[
\tau = \frac{17}{\rho} + 7.3\omega.
\]

\subsection{Proof of Lemma \ref{lem:mlr_rsc}} \label{proof:lem:mlr_rsc}

Recall that 
\[
Q_n^{MLR}(\bbeta'|\bbeta) = -\frac{1}{2n}\sum_{i=1}^n\left[w(y_i,\xb_i;\bbeta)(y_i - \langle\xb_i,\bbeta'\rangle)^2  + (1 - w(y_i,\xb_i;\bbeta))(y_i + \langle\xb_i,\bbeta'\rangle)^2\right].
\]
For any $\bbeta,\bbeta' \in \RR^p$, we have
\begin{equation} \label{eq:tmp17}
Q_n^{MLR}(\bbeta'|\bbeta) - Q_n^{MLR}(\bbeta^*|\bbeta) - \langle \nabla Q_n^{MLR}(\bbeta^*|\bbeta), \bbeta' - \bbeta^*\rangle = -\frac{1}{2} (\bbeta' - \bbeta^*)^{\top}\left(\frac{1}{n}\sum_{i=1}^n\xb_i\xb_i^{\top}\right)(\bbeta' - \bbeta^*).
\end{equation}
Note that we want to find $\gamma_n$ such that the right hand side of \eqref{eq:tmp17} is less than $-\frac{\gamma_n}{2}\|\bbeta'-\bbeta\|_2^2$ for any $\bbeta' - \bbeta^* \in \cC(\cS,\overline{\cS};\cR)$. In this example, we have $\cC(\cS,\overline{\cS};\cR) = \left\{\ub \in \RR^p: \|\ub_{\cS^{\bot}}\|_1 \leq 2\|\ub_{\cS}\|_1 + 2\sqrt{s}\|\ub\|_2\right\}$. It's sufficient to prove that the sample covariance matrix has restricted eigenvalues over set $\cC(\cS,\overline{\cS};\cR)$. The following statement is follows by the substitution $\Sigmab = \Ib_p$ and $X = X$ in Lemma \ref{lem:RE}: there exist constants $\{C_i\}_{i=0}^2$ such that
\begin{equation} \label{eq:tmp20}
\frac{1}{n}\sum_{i=1}^n \langle \xb_i, \ub \rangle^2 \geq \frac{1}{2}\|\ub\|_2^2 - C_0\frac{\log p}{n}\|\ub\|_1^2, \;\text{for all}\; \ub \in \RR^p,
\end{equation}
with probability at least $1 - C_1\exp(-C_2n)$.
For any $\ub \in \cC(\cS,\overline{\cS};\cR)$, we have
\[
\|\ub\|_1 = \|\ub_{\cS}\|_1 + \|\ub_{\cS^{\bot}}\|_1 \leq 3\|\ub_{\cS}\|_1 + 2\sqrt{s}\|\ub\|_2 \leq 5\sqrt{s}\|\ub\|_2.
\]
Applying \eqref{eq:tmp20} yields that
\[
\frac{1}{n}\sum_{i=1}^{n}\langle \x_i, \ub\rangle^2 \geq \frac{1}{2}\|\ub\|_2^2 - 25C_0\frac{s\log p}{n}\|\ub\|_2^2, \;\;\text{for all}\; \ub \in \cC(\cS,\overline{\cS};\cR).
\]
Consequently, when $n \geq C_3s\log p$ for sufficiently large $C_3$, $\frac{1}{n}\sum_{i=1}^{n}\langle \x_i, \ub\rangle^2 \geq 1/3\|\ub\|_2^2$, which implies $\gamma_n = 1/3$.

\subsection{Proof of Lemma \ref{lem:mlr:stat_error}} \label{proof:lem:mlr:stat_error}

According to the formulations of $Q^{MLR}_n(\cdot|\cdot)$ and $Q^{MLR}(\cdot|\cdot)$ in \eqref{mlr_Qn} and \eqref{mlr_Q}, we have
\begin{align} \label{eq:tmp19}
& \nabla Q_n^{MLR}(\bbeta^*|\bbeta) - \nabla Q^{MLR}(\bbeta^*|\bbeta) \notag\\
& = \bbeta^* - \left(\frac{1}{n}\sum_{i=1}^{n}\xb_i\xb_i^{\top}\right) \bbeta^* + \frac{2}{n}\sum_{i=1}^nw(y_i,\xb_i;\bbeta)y_i\xb_i- 2\mathbb{E}\left[w(Y,X;\bbeta)YX\right] - \frac{1}{n}\sum_{i=1}^ny_i\xb_i.
\end{align}
So 
\begin{align*}
& \|\nabla Q_n^{MLR}(\bbeta^*|\bbeta) - \nabla Q^{MLR}(\bbeta^*|\bbeta)\|_{\infty} \\
& \leq \underbrace{ \left\|\frac{1}{n}\sum_{i=1}^ny_i\xb_i\right\|_{\infty}}_{(a)} + \underbrace{\left\|\bbeta^* - \left(\frac{1}{n}\sum_{i=1}^{n}\xb_i\xb_i^{\top}\right) \bbeta^*\right\|_{\infty}}_{(b)} + \underbrace{ \left\|\frac{2}{n}\sum_{i=1}^nw(y_i,\xb_i;\bbeta)y_i\xb_i- 2\mathbb{E}\left[w(Y,X;\bbeta)YX\right] \right\|_{\infty}}_{(c)}.
\end{align*}
Next we bound the above three terms $(a),(b)$ and $(c)$ respectively.

\noindent{\bf Term $(a)$.} We let vector $\bzeta : = \frac{1}{n}\sum_{i=1}^n y_i\xb_i$. Consider $j$th coordinate of $\bzeta$. For any $j \in [p]$, we have
\[
\zeta_j = \frac{1}{n} \sum_{i=1}^n y_ix_{i,j},
\]
where $x_{i,j}$ is the $j$th coordinate of $\xb_i$. Note that $\{y_ix_{ij}\}_{i=1}^n$ are independent copies of random variables $(\langle X, Z\cdot\bbeta^*\rangle + W)X_j$ where $X \sim \cN(0,\Ib_p)$, $W \sim \cN(0,\sigma^2)$ and $Z$ has Rademacher distribution. $\langle X, Z\cdot\bbeta^*\rangle + W$ is sub-Gaussian random variable that has norm $\|\langle X, Z\cdot\bbeta^*\rangle + W\|_{\psi_2} \lesssim \sqrt{\|\bbeta^*\|_2^2 + \sigma^2}$. Also $X_j$ is sub-Gaussian random variable that has norm $\|X_j\|_{\psi_2} \lesssim 1$. Then based on Lemma \ref{lem:sub-Gaussian_product}, $(\langle X, Z\cdot\bbeta^*\rangle + W)X_j$ is sub-exponential with norm $\|(\langle X, Z\cdot\bbeta^*\rangle + W)X_j\|_{\psi_1} \lesssim \sqrt{\|\bbeta^*\|_2^2 + \sigma^2}$. Following standard concentration result of sub-exponential random variables (e.g., Lemma \ref{lem:sub-exponential_sum}), there exists some constant $C$ such that the following inequality
\[
\Pr\left(|\zeta_j| \geq t\right) \leq 2\exp\left(-C\frac{t^2n}{\|\bbeta^*\|_2^2 + \sigma^2}\right)
\]
holds for sufficiently small $t > 0$. Therefore,
\[
\Pr\left(\sup_{j \in [p]} |\zeta_j| > t\right) \leq 2p\exp\left(-C\frac{t^2n}{\|\bbeta^*\|_2^2 + \sigma^2}\right).
\]
Setting the right hand side to be $\delta/3$, we have that, when $n$ is sufficiently large (i.e., $n \geq C(\log p + \log(6/\delta))$ for some constant $C$), with probability at least $1 - \delta/3$.
\begin{equation} \label{eq:tmp:mlr1}
\left\|\frac{1}{n}\sum_{i=1}^ny_i\xb_i\right\|_{\infty} \lesssim (\|\bbeta^*\|_2 + \sigma)\sqrt{\frac{\log p + \log(6/\delta)}{n}}.
\end{equation}

\noindent{\bf Term $(b)$.} Now we let $\bzeta = \bbeta^* - \frac{1}{n}\xb_i\xb_i\bbeta^*$. For any $j \in [p]$, 
\[
\zeta_j = \frac{1}{n}\sum_{i=1}^{n} \beta^*_j - x_{i,j}\langle\xb_i, \bbeta^*\rangle.
\]
Note that $\{\beta^*_j - x_{i,j}\langle\xb_i,\bbeta^* \rangle\}_{i=1}^n$ are independent copies of random variable $\beta^*_j - X_j\langle X,\bbeta^* \rangle$. Using similar analysis in bounding term $(a)$, we claim that $\beta^*_j - X_j\langle X,\bbeta^* \rangle$ is centered sub-exponential random variable with norm $\|\beta^*_j - X_j\langle X,\bbeta^* \rangle\|_{\psi_1} \lesssim \|\bbeta^*\|_2$. Therefore, for sufficiently small $t$ and some constant $C$,
\[
\Pr\left(|\zeta_j| \geq t\right) \leq 2\exp\left(-C\frac{t^2n}{\|\bbeta^*\|_2^2}\right).
\]
Using union bound implies that
\[
\Pr\left(\sup_{j \in [p]}|\zeta_j| \geq t\right) \leq 2p\cdot\exp\left(-C\frac{t^2n}{\|\bbeta^*\|_2^2}\right).
\]
Setting the right hand side to be $\delta/3$, we have that, when $n$ is sufficiently large,
\begin{equation} \label{eq:tmp:mlr2}
\left\|\bbeta^* - \left(\frac{1}{n}\sum_{i=1}^{n}\xb_i\xb_i^{\top}\right) \bbeta^*\right\|_{\infty} \lesssim \|\bbeta^*\|_2 \sqrt{\frac{\log p + \log(6/\delta)}{n}}
\end{equation}
holds with probability at least $1 - \delta/3$.

\noindent{\bf Term $(c)$}. The analysis of this term is similar to the previous two terms. We let 
\[
\bzeta := \frac{1}{n}\sum_{i=1}^nw(\y_i,\xb_i;\bbeta)y_i\xb_i- \mathbb{E}\left[w(Y,X;\bbeta)YX\right].
\]
For any $j \in [p]$, 
\[
\zeta_j = \frac{1}{n}\sum_{i=1}^nw(\y_i,\xb_i;\bbeta)y_ix_{i,j} - \mathbb{E}\left[w(Y,X;\bbeta)YX\right].
\]
Note that $\{w(\y_i,\xb_i;\bbeta)y_ix_{i,j}\}_{i=1}^n$ are independent copies of random variable $w(Y,X;\bbeta)YX_j$. We know that $Y$ is sub-Gaussian with norm $\|Y\|_{\psi_2} \lesssim \sqrt{\|\bbeta^*\|_2^2 + \sigma^2}$. Since $w(Y,X;\bbeta)$ is bounded, $w(Y,X;\bbeta)Y$ is also sub-Gaussian. Consequently, $w(Y,X;\bbeta)YX_j$ is sub-exponential. By standard concentration result, for some constant $C$ and sufficiently small $t$,
\[
\Pr(|\zeta_j| \geq t) \leq 2\exp\left(-C\frac{nt^2}{\|\bbeta^*\|_2^2 + \sigma^2}\right).
\]
Therefore,
\[
\Pr(\sup_{j \in [p]}|\zeta_j| \geq t) \leq 2\exp\left(-C\frac{nt^2}{\|\bbeta^*\|_2^2 + \sigma^2}\right).
\]
Setting the right hand side to be $\delta/3$, we have that, when $n$ is sufficiently large,
\begin{equation} \label{eq:tmp:mlr3}
\left\|\frac{2}{n}\sum_{i=1}^nw(\y_i,\xb_i;\bbeta)y_i\xb_i- 2\mathbb{E}\left[w(Y,X;\bbeta)YX\right] \right\|_{\infty} \lesssim (\|\bbeta^*\|_2 + \delta)\sqrt{\frac{\log p + \log(6/\delta)}{n}}
\end{equation}
with probability at least $1 - \delta/3$.

Putting \eqref{eq:tmp:mlr1}, \eqref{eq:tmp:mlr2} and \eqref{eq:tmp:mlr3} together completes the proof.

\subsection{Proof of Lemma \ref{lem:mlr_rsc_lowrank}} \label{proof:lem:mlr_rsc_lowrank}

Similar to \eqref{eq:tmp17}, we have that for any $\bGamma',\bGamma \in \RR^{p_1 \times p_2}$,
\begin{equation} \label{eq:tmp21}
Q_n^{MLR}(\bGamma'|\bGamma) - Q_n^{MLR}(\bGamma^*|\bGamma) - \langle \nabla Q_n^{MLR}(\bGamma^*|\bGamma), \bGamma' - \bGamma^*\rangle = -\frac{1}{2n} \sum_{i=1}^{n} \langle \Xb_i, \bGamma' - \bGamma^*\rangle^2.
\end{equation}
Note that $\bGamma' - \bGamma^* \in \cC(\cS,\overline{\cS};\|\cdot\|_{*})$. Let $\Theta := \bGamma' - \bGamma^*$, we thus have 
\[
\|\Theta_{\overline{\cS}^{\bot}}\|_{*} \leq 2\cdot\|\Theta_{\overline{\cS}}\|_{*} + 2\cdot\sqrt{2\theta}\|\Theta\|_{F}.
\]
We make use of the following result.
\begin{lemma}
	Let $\{\Xb_i\}_{i=1}^n$ be $n$ independent samples of random matrix $X \in \RR^{p_1\times p_2}$ where the entries are i.i.d. Gaussian random variable with distribution $\cN(0,1)$. There exits constants $C_1, C_2$ such that 
	\begin{equation*} \label{eq:tmp18}
	\frac{1}{\sqrt{n}}\sqrt{\sum_{i=1}^n \langle \Xb_i,\Theta\rangle^2} \geq \frac{1}{4}\|\Theta\|_F - 12\left(\sqrt{\frac{p_1}{n}} + \sqrt{\frac{p_2}{n}}\right)\|\Theta\|_{*}, \;\text{for all}\; \Theta \in \RR^{p_1\times p_2},
	\end{equation*}
	with probability at least $1 - C_1\exp(-C_2n)$.
\end{lemma}
\begin{proof}
	See Proposition 1 in \citet{negahban2011estimation} for detailed proof.
\end{proof}
Then for our $\Theta$, using the above result yields that
\begin{align*}
\frac{1}{\sqrt{n}}\sqrt{\sum_{i=1}^n \langle \Xb_i,\Theta\rangle^2} & \geq \frac{1}{4}\|\Theta\|_F - 12\left(\sqrt{\frac{p_1}{n}} + \sqrt{\frac{p_2}{n}}\right)\left(\|\Theta_{\overline{\cS}}\|_{*} + \|\Theta_{\overline{\cS}^{\bot}}\|_{*}\right) \\
& \geq \frac{1}{4}\|\Theta\|_F - 12\left(\sqrt{\frac{p_1}{n}} + \sqrt{\frac{p_2}{n}}\right)\left(3\|\Theta_{\overline{\cS}}\|_{*} + 2\sqrt{2r}\|\Theta\|_F\right) \\
& \geq \left[\frac{1}{4} - 60\sqrt{2\theta}\left(\sqrt{\frac{p_1}{n}} + \sqrt{\frac{p_2}{n}}\right)\right]\|\Theta\|_F.
\end{align*}
So when $n \geq C\theta\max\{p_1,p_2\}$ for sufficient large $C$, we have $\frac{1}{\sqrt{n}}\sqrt{\sum_{i=1}^n \langle \Xb_i,\Theta\rangle^2} \geq \|\Theta\|_F/\sqrt{20}$. Plugging this result back into \eqref{eq:tmp21} gives us $\gamma_n = 1/20$ thus completes the proof.

\subsection{Proof of Lemma \ref{lem:mlr_staterr_lowrank}} \label{proof:lem:mlr_staterr_lowrank}

Parallel to \eqref{eq:tmp19}, we have
	\begin{align*}
	& \nabla Q_n^{MLR}(\bGamma^*|\bGamma) - \nabla Q^{MLR}(\bGamma^*|\bGamma) \\
	& = \bGamma^* - \frac{1}{n}\sum_{i=1}^n\langle \Xb_i,\bGamma^*\rangle\bGamma^* + \frac{2}{n}\sum_{i=1}^n w(y_i,\Xb_i;\bGamma)y_i\Xb_i - 2\mathbb{E}\left[w(Y,X;\bGamma)YX\right] - \frac{1}{n}\sum_{i=1}^n y_i\Xb_i.
	\end{align*}
	The dual norm of nuclear norm is spectral norm. So we are interested in bounding the following term for fixed $\bGamma$:
	\begin{align*}
	& \left\|\nabla Q_n^{MLR}(\bGamma^*|\bGamma) - \nabla Q^{MLR}(\bGamma^*|\bGamma)\right\|_{2} \\
	& \leq \underbrace{\left\| \frac{1}{n}\sum_{i=1}^n y_i\Xb_i\right\|_{2}}_{U_1} + \underbrace{\left\|\bGamma^* - \frac{1}{n}\sum_{i=1}^n\langle \Xb_i,\bGamma^*\rangle\Xb_i \right\|_{2}}_{U_2} + \underbrace{\left\|\frac{2}{n}\sum_{i=1}^n w(y_i,\Xb_i;\bGamma)y_i\Xb_i - 2\mathbb{E}\left[w(Y,X;\bGamma)YX\right]\right\|_{2}}_{U_3}.
	\end{align*}
	Next we bound the three terms $U_1, U_2$ and $U_3$ respectively.
	
	\noindent{\bf Term $U_1$.} We first note that
	\[
	U_1 = \sup_{\substack{\ub ~\in~ \SSS^{p_1-1}\\ \vb ~\in~ \SSS^{p_2-1}}}\frac{1}{n}\sum_{i=1}^n y_i \langle\ub\vb^{\top}, \Xb_i \rangle.
	\]
	In particular, we let 
	\[
	Z(a,b) =  \sup_{\substack{\ub ~\in ~a\SSS^{p_1-1}\\ \vb ~\in~ b\SSS^{p_2-1}}}\frac{1}{n}\sum_{i=1}^n y_i \langle\ub\vb^{\top}, \Xb_i \rangle.
	\]
	We thus have $Z(a,b) = abZ(1,1)$. We construct $1/4$-covering sets of $\SSS^{p_1-1}$ and $\SSS^{p_2-1}$, which we denote as $\cN_1$ and $\cN_2$ respectively. Therefore, for any $\u \in \SSS^{p-1}, \v \in \SSS^{p_2 - 1 }$, we can always find $\ub' \in \cN_1, \vb' \in \cN_2$ such that $\|\ub - \ub'\|_2 \leq 1/4$, $\|\vb - \vb'\|_2 \leq 1/4$. Moreover, we have the following decomposition $\ub\vb^{\top} = \ub'\vb'^{\top} + (\ub - \ub')\vb'^{\top} + \ub'(\vb - \vb')^{\top} + (\ub - \ub')(\vb - \vb')^{\top}$. Therefore, we have
	\[
	Z(1,1) \leq \max_{\ub \in \cN_1, \vb \in \cN_2} \frac{1}{n}\sum_{i=1}^n y_i \langle\ub\vb^{\top}, \Xb_i \rangle + Z(1/4,1) + Z(1/4,1) + Z(1/4,1/4),
	\]
	which implies that
	\[
	Z(1,1) \leq \frac{16}{7} \max_{\ub \in \cN_1, \vb \in \cN_2} \frac{1}{n}\sum_{i=1}^n y_i \langle\ub\vb^{\top}, \Xb_i \rangle.
	\]
	For any fixed $\ub$ and $\vb$, $\{y_i\langle \ub\vb^{\top}, \Xb_i\rangle\}_{i=1}^{n}$ are $n$ independent copies of random variable $Y\langle \ub\vb^{\top}, X\rangle$ where $Y$ is sub-Gaussian with norm $\|Y\|_{\psi_2} \lesssim \sqrt{\|\bGamma^*\|_F^2+\sigma^2}$, $\langle \ub\vb^{\top}, X\rangle$ is zero mean Gaussian with variance $1$. Following Lemma \ref{lem:sub-Gaussian_product}, $Y\langle \ub\vb^{\top}, X\rangle$ is sub-exponential with norm $\|Y\langle \ub\vb^{\top}, X\rangle\|_{\psi_1} \lesssim \sqrt{\|\bGamma^*\|_F^2 + \sigma^2}$. Using concentration result in Lemma \ref{lem:sub-exponential_sum}, we have
	\[
	\Pr\left(\left|\frac{1}{n}\sum_{i=1}^n y_i \langle \ub\vb^{\top}, \Xb_i\rangle\right| \geq t\right) \leq 2\exp\left(-\frac{Ct^2n}{\|\bGamma^*\|_F^2 + \sigma^2}\right)
	\]
	for sufficiently small $t > 0$. Note that $|\cN_1| \leq 9^{p_1}, |\cN_2| \leq 9^{p_2}$. By applying union bounds over $\cN_1$ and $\cN_2$, we have
	\[
	\Pr\left(\max_{\ub \in \cN_1, \vb \in \cN_2} \frac{1}{n}\sum_{i=1}^n y_i \langle\ub\vb^{\top}, \Xb_i \rangle \geq t\right) \leq 2\cdot9^{(p_1+p_2)}\exp\left(-\frac{Ct^2n}{\|\bGamma^*\|_F^2 + \sigma^2}\right).
	\]
By setting the right hand side to be $\delta/3$, we have that if $n \geq C(p_1+p_2+\log(6/\delta))$ for sufficiently large $C$, then
\begin{equation} \label{eq:U1}
U_1 \lesssim (\|\bGamma^*\|_F + \sigma)\sqrt{\frac{p_1+p_2+\log(6/\delta)}{n}}
\end{equation}
with probability at least $1 - \delta/3$.

\noindent{\bf Term $U_2$.} Parallel to the analysis of term $U_1$, we have
\[
U_2 = \sup_{\substack{\ub ~\in~ \SSS^{p_1-1}\\ \vb ~\in~ \SSS^{p_2-1}}} \langle \ub\vb^{\top}, \bGamma^* \rangle - \frac{1}{n}\sum_{i=1}^n\langle \Xb_i,\bGamma^*\rangle \cdot \langle \ub\vb^{\top}, \Xb_i \rangle.
\]
We construct $1/4$-nets $\cN_1,\cN_2$ of $\SSS^{p_1-1}$ and $\SSS^{p_2-1}$ respectively. Then
\[
U_2 \leq \frac{16}{7} \max_{\ub \in \cN_1, \vb \in \cN_2} \langle \ub\vb^{\top}, \bGamma^* \rangle - \frac{1}{n}\sum_{i=1}^n\langle \Xb_i,\bGamma^*\rangle \cdot \langle \ub\vb^{\top}, \Xb_i \rangle.
\]
For any fixed $\ub,\vb$, note that $\{\langle \Xb_i,\bGamma^*\rangle \cdot \langle \ub\vb^{\top}, \Xb_i\rangle\}_{i=1}^n$ are $n$ independent samples of random variable $\langle X,\bGamma^* \rangle \cdot \langle\ub\vb^{\top},X \rangle$ where $\langle X,\bGamma^* \rangle \sim \cN(0, \|\bGamma^*\|_F^2)$ and $\langle \ub\vb^{\top}, X\rangle \sim \cN(0, 1)$. So $\langle X,\bGamma^* \rangle \cdot \langle\ub\vb^{\top},X \rangle$ is sub-exponential with norm $O(\|\bGamma^*\|_F)$. Using the centering argument (Lemma \ref{lem:sub-Gaussian_centering}) and concentration result (Lemma \ref{lem:sub-exponential_sum}), we have
\[
\Pr\left(\left|\langle \ub\vb^{\top}, \bGamma^* \rangle - \frac{1}{n}\sum_{i=1}^n\langle \Xb_i,\bGamma^*\rangle \cdot \langle \ub\vb^{\top}, \Xb_i\rangle\right| \geq t \right) \leq 2\cdot\exp\left(-C\frac{t^2n}{\|\bGamma^*\|_F^2}\right)
\]
for sufficiently small $t$. Using the union bound over sets $\cN_1,\cN_2$, we conclude that when $n \geq C(p_1 + p_2 + \log(6/\delta))$ for sufficiently large $C$, we have
\begin{equation} \label{eq:U2}
U_2 \lesssim \|\bGamma^*\|_F\sqrt{\frac{p_1+p_2+\log(6/\delta)}{n}}
\end{equation}
with probability at least $1 - \delta/3$.

\noindent{\bf Term $U_3$.} We first have
\[
U_3 = \sup_{\substack{\ub \in \SSS^{p_1-1} \\ \vb \in \SSS^{p_2-1}}}  \frac{2}{n}\sum_{i=1}^n w \cdot y_i\langle \ub\vb^{\top},\Xb_i\rangle - 2\mathbb{E}\left[w\cdot Y\langle \ub\vb^{\top}, X\rangle\right].
\]
Similar to the analysis of the first two terms, by constructing $\cN_1,\cN_2$, we have
\[
U_3 \leq \frac{16}{7}\max_{\ub \in \cN_1, \vb \in \cN_2} \frac{2}{n}\sum_{i=1}^n w \cdot y_i\langle \ub\vb^{\top},\Xb_i\rangle - 2\mathbb{E}\left[w\cdot Y\langle \ub\vb^{\top}, X\rangle\right].
\]
Note that $\{w\y_i\langle \ub \vb^{\top}, \Xb_i\rangle\}_{i=1}^n$ are $n$ independent samples of random variable $wY\langle\ub \vb^{\top}, X \rangle$ where $\langle \ub\vb^{\top}, X\rangle \sim \cN(0,1)$ and $wY$ is sub-Gaussian with norm $\|wY\|_{\psi_2} \lesssim \sqrt{\|\bGamma^*\|_F^2 + \sigma^2}$ since $|w| \leq 1$. We thus have $wY\langle\ub \vb^{\top}, X \rangle$ is sub-exponential with norm $\|wY\langle\ub \vb^{\top}, X \rangle\|_{\psi_1} \lesssim \sqrt{\|\bGamma^*\|_F^2 + \sigma^2}$. Then following the similar steps in analyzing the first two terms, we reach the conclusion that
\begin{equation} \label{eq:U3}
U_3 \lesssim (\|\bGamma^*\|_F + \sigma)\sqrt{\frac{p_1+p_2+\log(6/\delta)}{n}}
\end{equation}
with probability at least $1 - \delta/3$ when $n \gtrsim p_1 + p_2 + \log(6/\delta)$.

Putting \eqref{eq:U1}, \eqref{eq:U2} and \eqref{eq:U3} together completes the proof.

\section{Proofs about Missing Covariate Regression}
In this section, we provide the proofs for missing covariate regression model. We begin with a result that states several properties of the conditional correlation matrix, which play important roles in proving curvature conditions. Recall that, given samples $(y_i,\zb_i,\xb_i)$ and $\bbeta$, $\Sigmab_{\bbeta}(y_i,\zb_i, \xb_i)$ is given in \eqref{mcr:covariance}. We let $Z \in \RR^p$ be random vector with i.i.d. binary entries such that $\Pr(Z_1 = 1) = \epsilon$.  Define the population level correlation covariance matrix as
\[
\overline{\Sigmab}_{\bbeta} := \mathbb{E}\left[\Sigmab_{\bbeta}(Y,Z,X)\right].
\]
\begin{lemma} \label{lem:mcr_spectral}For $\overline{\Sigmab}_{\bbeta}$, we have the following decomposition
\begin{align*}
\overline{\Sigmab}_{\bbeta} = \epsilon\Ib_p + \Sigmab_1 - \Sigmab_2,
\end{align*}
where
\begin{align*} 
& \Sigmab_1 = \mathbb{E}\left\{ \left[ (\bm{1} - Z)\odot X + \nu Z\odot \bbeta\right]\cdot\left[ (\bm{1} - Z)\odot X + \nu Z\odot \bbeta\right]^{\top}  \right\}, \\
& \Sigmab_2 = \mathbb{E}\left[ \frac{1}{\sigma^2 + \|Z\odot\bbeta\|_2^2}(Z\odot\bbeta)(Z\odot\bbeta)^{\top}\right], ~~~
\nu = \frac{Y - \langle\bbeta, (\bm{1} - Z)\odot X\rangle}{\sigma^2 + \|Z\odot\bbeta\|_2^2}.
\end{align*}
Let $\zeta := (1+\omega)\rho$, we have
\begin{align}
& \lambda_{min}(\Sigmab_1) \geq 1-\epsilon-2\zeta^2\sqrt{\epsilon},\label{eq:mcr:spectral1} \\  & \lambda_{max}(\Sigmab_2) \leq \zeta^2\epsilon,\label{eq:mcr:spectral2} \\
& \lambda_{max}(\overline{\Sigmab}_{\bbeta}) \leq 1 + 2\zeta^2\sqrt{\epsilon} + (1+\zeta^2)\zeta^2\epsilon. \label{eq:mcr:spectral4}
\end{align}
In particular, let $\bbeta = \bbeta^*$, we have $\overline{\Sigmab}_{\bbeta^*} = \Ib_p$.
\end{lemma}
\begin{proof}
The decomposition follows by taking expectation of \eqref{mcr:covariance}. For $\Sigmab_1$, expanding the bracket leads to
\[
\Sigmab_1 = (1-\epsilon)\Ib_p + \underbrace{\mathbb{E}\left\{\nu[(\bm{1} - Z)\odot X](Z\odot\bbeta)^{\top} + \nu(Z\odot\bbeta)[(\bm{1} - Z)\odot X]^{\top}\right\}}_{\Mb} + \underbrace{\mathbb{E}\left[\nu^2(Z\odot\bbeta)(Z\odot\bbeta)^{\top}\right]}_{\Nb}.
\]
For term $\Mb$, consider its spectral norm. Since it's symmetric, we have
\begin{align*}
\|\Mb\|_{2} & = \sup_{\ub \in \SSS^{p-1}} 2\left|\mathbb{E}\left[\nu\langle Z\odot \bbeta,\ub\rangle\cdot \langle (\bm{1}-Z)\odot X, \ub\rangle\right]\right| \\
& = 2\sup_{\ub \in \SSS^{p-1} } \left| \mathbb{E}\left[\frac{1}{\sigma^2+\|Z\odot\bbeta\|_2}\langle (\bm{1}-Z)\odot(\bbeta^* - \bbeta), \ub\rangle\cdot \langle Z\odot \bbeta, \ub\rangle \right] \right| \\
& \leq 2\frac{1}{\sigma^2}\mathbb{E}\left[\|(\bm{1}-Z)\odot(\bbeta^* - \bbeta)\|_2\|Z\odot\bbeta\|_2\right] \leq 2\frac{1}{\sigma^2}\sqrt{\mathbb{E}\left[\|(\bm{1}-Z)\odot(\bbeta^* - \bbeta)\|_2^2 \cdot \|Z\odot\bbeta\|_2^2\right]} \\
& \leq 2\frac{1}{\sigma^2}\sqrt{\epsilon(1-\epsilon)}\|\bbeta-\bbeta^*\|_2\|\bbeta\|_2
\leq 2\rho^2\omega(1+\omega)\sqrt{\epsilon(1-\epsilon)} \leq 2\zeta^2\sqrt{\epsilon}.
\end{align*}
where the second equality follows by taking expectation of $X$ and Gaussian noise $W$, the last inequality follows from the definitions of $\omega, \rho$ given in Section \ref{sec:mcr}. Note that $\Nb \succeq \bm{0}$. Then the lower bound of $\lambda_{min}(\Sigmab_1)$ follows by using $ \lambda_{min}(\Sigmab_1) \geq 1-\epsilon - \|\Mb\|_{2}$. For $\Sigmab_2$, we have
\[
\Sigmab_2 = \mathbb{E}\left[\frac{1}{\sigma^2+\|Z\odot \bbeta\|_2^2}(Z\odot\bbeta)(Z\odot\bbeta)^{\top}\right] \preceq  \frac{1}{\sigma^2}\left((\epsilon - \epsilon^2){\rm diag}(\bbeta\odot \bbeta) + \epsilon^2\bbeta\bbeta^{\top}\right).
\] 
Therefore, $\lambda_{max}(\Sigmab_2) \leq \zeta^2\epsilon$.
Note that
\begin{align*}
\Nb & \preceq \frac{1}{\sigma^4}\mathbb{E}\left[(Y - \langle\bbeta, (\bm{1} - Z)\odot X\rangle)^2(Z\odot\bbeta)(Z\odot \bbeta)^{\top}\right]\\
& = \frac{1}{\sigma^4}\mathbb{E}\left[(\sigma^2 + \|\bbeta^* - (\bm{1}-Z)\odot\bbeta\|_2^2)(Z\odot\bbeta)(Z\odot\bbeta)^{\top} \right] \\
& \preceq \frac{1}{\sigma^4}(\sigma^2 + \|\bbeta^*\|_2^2 + \|\bbeta - \bbeta^*\|_2^2) \left((\epsilon - \epsilon^2){\rm diag}(\bbeta\odot \bbeta) + \epsilon^2\bbeta\bbeta^{\top}\right).
\end{align*}
We thus have $\lambda_{max}(\Nb) \leq \frac{1}{\sigma^4}(\sigma^2 + \|\bbeta^*\|_2^2 + \|\bbeta - \bbeta^*\|_2^2)\epsilon\|\bbeta\|_2^2 \leq (1+\zeta^2)\zeta^2\epsilon$. The corresponding bound for $\lambda_{max}(\overline{\Sigmab}_{\bbeta})$ then follows from $\lambda_{max}(\overline{\Sigmab}_{\bbeta}) \leq 1 + \lambda_{max}(\Mb) + \lambda_{max}(\Nb)$.
 
When $\bbeta = \bbeta^*$, we have 
\[
\mathbb{E}_{X,W}(\nu^2) = \frac{\mathbb{E}_{X,W}\left[(\langle X, \bbeta^*\rangle + W - \langle X,(1-Z)\odot\bbeta^* \rangle)^2\right]}{(\sigma^2 + \|Z\odot\bbeta^*\|_2^2)^2} = \frac{1}{\sigma^2+\|Z\odot\bbeta^*\|_2^2}
\]
and 
\begin{align*}
& \mathbb{E}_{X,W}(\nu (\bm{1}-Z)\odot X) = \frac{\mathbb{E}\left[(\langle X, \bbeta^*\rangle + W - \langle X,(1-Z)\odot\bbeta^* \rangle)(\bm{1}-Z)\odot X\right]}{\sigma^2+\|Z\odot\bbeta^*\|_2^2} \\
& = \frac{(\bm{1}-Z)\odot Z\odot\bbeta^*}{\sigma^2+\|Z\odot\bbeta^*\|_2^2} = \bm{0}.
\end{align*}
Therefore, $\Mb = \bm{0}$ and $\Nb = \Sigmab_2$. We thus have $\overline{\Sigmab}_{\bbeta^*} = \epsilon\Ib_p + (1-\epsilon)\Ib_p = \Ib_p$. 
\end{proof}

\subsection{Proof of Lemma \ref{lem:mcr_self_consistency}} \label{proof:lem:mcr_self_consistency}
In this example
\[
\cM(\bbeta^*) = \left(\mathbb{E}\left[\Sigmab_{\bbeta^*}(Y,Z,X)\right]\right)^{-1}\mathbb{E}\left[Y\bmu_{\bbeta^*}(Y,Z,X)\right].
\]
Following Lemma \ref{lem:mcr_spectral}, we have $\Sigmab_{\bbeta^*}(Y,Z,X) = \Ib_p$. Meanwhile, we have
\begin{align*}
\mathbb{E}\left[Y\bmu_{\bbeta^*}(Y,Z,X)\right] & = \mathbb{E}\left[(\langle \bbeta^*,X\rangle + W)\left((\bm{1}-Z)\odot X + \frac{\langle Z\odot \bbeta^*, X\rangle + W}{\sigma^2 + \|Z\odot\bbeta^*\|_2^2}Z\odot \bbeta^*\right)\right] \\
& = \mathbb{E}\left[(\bm{1}-Z)\odot\bbeta^* + Z\odot\bbeta^*\right] = \bbeta^*.
\end{align*}
Thus $\cM(\bbeta^*) = \bbeta^*$.
\subsection{Proof of Lemma \ref{lem:mcr_smrc}} \label{proof:lem:mcr_smrc}
Following Lemma \ref{lem:mcr_spectral}, we have $\overline{\Sigmab}_{\bbeta^*} = \Ib_p$. Therefore, $Q^{MCR}(\cdot|\bbeta^*)$ is $1$-strongly concave. For any $\bbeta \in \cB(w\|\bbeta^*\|;\bbeta^*)$, following \eqref{eq:mcr:spectral4}, we have that $Q^{MCR}(\cdot|\bbeta)$ is $\mu$-smooth with 
$
\mu = 1 + 2\zeta^2\sqrt{\epsilon} + (1+\zeta^2)\zeta^2\epsilon.
$

\subsection{Proof of Lemma \ref{lem:mcr_rsc}} \label{proof:lem:mcr_rsc}
In order to show $Q^{MCR}_n(\cdot|\bbeta)$ is $\gamma_n$-strongly concave over $\cC(\cS,\overline{\cS};\cR)$, since $Q_n^{MCR}(\cdot|\bbeta)$ is quadratic, it's then equivalent to show 
\[
\frac{1}{n}\sum_{i=1}^n \ub^{\top} \Sigmab_{\bbeta}(y_i,\zb_i,\xb_i)\ub \geq \gamma_n \|\ub\|_2^2
\]
for all $\ub \in \cC(\cS,\overline{\cS}, \cR)$. Expanding $\Sigmab_{\bbeta}$ gives us
\[
\frac{1}{n}\sum_{i=1}^n \ub^{\top} \Sigmab_{\bbeta}(y_i,\zb_i,\xb_i)\ub \geq \underbrace{\frac{1}{n}\sum_{i=1}^n \langle \bmu_{\bbeta}(y_i,\zb_i,\xb_i), \ub\rangle^2}_{L_1} - \underbrace{\frac{1}{n}\sum_{i=1}^{n}\left(\frac{1}{\sigma^2 + \|\zb_i\odot\bbeta\|_2^2}\right)\langle \zb_i\odot \bbeta, \ub\rangle^2}_{L_2}.
\]
We choose to bound each term using restricted eigenvalue argument in Lemma \ref{lem:RE}. To ease notation, we let $\nu := \frac{y_i - \langle(\bm{1} - \zb_i)\odot\bbeta, \xb_i \rangle}{\sigma^2 + \|\zb_i\odot\bbeta\|_2^2}$.

\noindent {\bf Term $L_1$}. Note that $\bmu_{\bbeta}(y_i,\zb_i,\xb_i)$ are samples of $\bmu_{\bbeta}(Y,Z,X)$ which is zero mean sub-Gaussian random vector with covariance matrix $\Sigmab_1$ given in Lemma \ref{lem:mcr_spectral}. Moreover, we have $\lambda_{min}(\Sigmab_1) \geq 1-\epsilon-2\zeta^2\sqrt{\epsilon}$. By restricting $\epsilon \leq 1/4$ and assuming $\epsilon \leq C\zeta^{-4}$ for sufficiently small $C$, we have $\lambda_{min}(\Sigmab_1) \geq \frac{1}{2}$. Moreover
\[
\|\bmu_{\bbeta}(Y,Z,X)\|_{\psi_2} \lesssim \|(\bm{1}-Z)\odot X\|_{\psi_2} + \|\nu Z\odot \bbeta\|_{\psi_2} \lesssim 1 + \|\nu Z\odot \bbeta\|_{\psi_2}.
\]
Note that $\|\nu Z\odot \bbeta\|_{\psi_2} = \sup_{\ub \in \SSS^{p-1}} \|\nu \langle Z\odot \bbeta, \ub\rangle\|_{\psi_2} \leq \|\bbeta\|_2\cdot \big\||\nu|\big\|_{\psi_2} \leq \sigma^{-2}\|\bbeta\|_2\cdot\big\||W + \langle X, \bbeta^* - (\bm{1} - Z)\odot\bbeta\rangle|\big\|_{\psi_2} \lesssim (1+\omega)\rho + (1+\omega)^2\rho^2$. As $\zeta := (1+\omega)\rho$. We thus have
$\|\bmu_{\bbeta}(Y,Z,X)\|_{\psi_2} \lesssim (1+\zeta)^2$. Using Lemma \ref{lem:RE} with the substitution $\Sigmab = \Sigmab_1$ and $X = \bmu_{\bbeta}(Y,Z,X)$, we claim that there exist constants $C_i$ such that
\begin{equation} \label{eq:L1}
L_1 \geq \frac{1}{4}\|\ub\|_2^2 - C_0(1 + \zeta)^8\frac{\log p}{n}\|\ub\|_1^2 \;\text{for all}\; \ub \in \RR^p.
\end{equation}
with probability at least $1 - C_1\exp(-C_2n(1 + \zeta)^{-8})$. 

\noindent {\bf Term $L_2$.}
We now turn to term $L_2$. We introduce $n$ i.i.d. samples $\{p_i\}_{i=1}^n$ of Rademacher random variable $P$ with $\Pr(P=1) = \Pr(P=-1) = 1/2$. Equivalently, we have
\[
L_2 = \frac{1}{n}\sum_{i=1}^{n}\frac{1}{\sigma^2 + \|\zb_i\odot\bbeta\|_2^2}\langle p_i\zb_i\odot \bbeta, \ub\rangle^2.
\]
Note that $\sqrt{ (\sigma^2 + \|Z\odot\bbeta\|_2^2)^{-1}}PZ\odot\bbeta$ is zero mean sub-Gaussian random vector with covariance matrix $\Sigmab_2$ given in Lemma \ref{lem:mcr_spectral}. Moreover, we have $\lambda_{max}(\Sigmab_2) \leq \zeta^2\epsilon \leq 1/12$, where the last inequality follows by letting $\epsilon \leq C\zeta^{-2}$ for sufficiently small $C$. Also note that 
\[
\left\|\sqrt{(\sigma^2 +\|Z\odot\bbeta\|_2^2)^{-1}}PZ\odot\bbeta\right\|_{\psi_2} \lesssim  \sigma^{-1}\|Z\odot\bbeta\|_{\psi_2} \lesssim \zeta. 
\]
Using Lemma \ref{lem:RE} with substitution $\Sigmab = \Sigmab_2$ and $X = \sqrt{ (\sigma^2+\|Z\odot\bbeta\|_2^2)^{-1}}PZ\odot\bbeta$, we claim there exists constants $C_i'$ such that
\begin{equation} \label{eq:L2}
L_2 \leq \frac{1}{8}\|\ub\|_2^2 + C_0'\max\{\zeta^4,1\} \frac{\log p}{n}\|\ub\|_1^2,\;\text{for all}\; \ub \in \RR^p.
\end{equation}
with probability at least $1 - C_1'\exp(-C_2'n\min\{\zeta^{-4},1\})$.

Now we put \eqref{eq:L1} and \eqref{eq:L2} together. So we obtain
\[
\frac{1}{n}\sum_{i=1}^n \ub^{\top} \Sigmab_{\bbeta}(y_i,\zb_i,\xb_i)\ub \geq \frac{1}{8}\|\ub\|_2^2 - (C_0+C_0')(1+\zeta)^8\frac{\log p}{n}\|\ub\|_1^2.
\]
For any $\ub \in \cC(\cS,\overline{\cS};\cR)$, we have $\|\ub\|_1 \leq 5\sqrt{s}\|\ub\|_2$. Consequently, when $n \geq C(1+\zeta)^8s\log p$ for sufficiently large $C$, we have that, with high probability, $Q_n^{MCR}(\cdot|\bbeta)$ is $\gamma_n$-strongly concave over $\cC$ with $\gamma_n = 1/9$.

\subsection{Proof of Lemma \ref{lem:mcr_staterror}} \label{proof:lem:mcr_staterror}
In this example,
\begin{align*}
& \|\nabla Q_n^{MCR}(\bbeta^*|\bbeta) - \nabla Q^{MCR}(\bbeta^*|\bbeta)\|_{\cR^*} \\
& \leq \underbrace{\left\|\frac{1}{n}\sum_{i=1}^ny_i\bmu_{\bbeta}(y_i,\zb_i,\xb_i) - \mathbb{E}\left[Y\bmu_{\bbeta}(Y,Z,X)\right]\right\|_{\infty}}_{U_1} + \underbrace{  \left\|\frac{1}{n}\sum_{i=1}^n\Sigmab_{\bbeta}(y_i,\zb_i,\xb_i)\bbeta^* - \mathbb{E}\left[\Sigmab_{\bbeta}(Y,Z,X)\right]\bbeta^*\right\|_{\infty}}_{U_2}.
\end{align*}
To ease notation, we let $\nu := \frac{y_i - \langle(\bm{1} - \zb_i)\odot\bbeta, \xb_i \rangle}{\sigma^2 + \|\zb_i\odot\bbeta\|_2^2}$. Next we bound the term $U_1$ and $U_2$ respectively.

\noindent {\bf Term $U_1$}. Consider one coordinate of vector $V := Y\bmu_{\bbeta}(Y,Z,X)$. For any $j \in [p]$, we have
\[
V_j = Y[(1-Z_j)X_j + \nu Z_j\beta_j].
\]
So $V_j$ is sub-exponential random variable since $Y$ and $(1-Z_j)X_j + \nu Z_j\beta_j$ are both sub-Gaussians. Moreover, we have $\|Y\|_{\psi_2} \lesssim \sigma + \|\bbeta^*\|_2$ and $\|(1-Z_j)X_j + \nu Z_j\beta_j\|_{\psi_2} \lesssim \|(1-Z_j)X_j\|_{\psi_2} + \|\nu Z_j\bbeta_j\|_{\psi_2} \lesssim 1 + \sigma^{-2}(\sigma + \sqrt{1+\omega^2}\|\bbeta^*\|_2)\|\bbeta\|_2$. The last inequality follows from the fact that $\nu$ is sub-Gaussian with $\|\nu\|_{\psi_2} \lesssim \sigma^{-2}(\sigma + \sqrt{1+\omega^2}\|\bbeta^*\|_2)$. We have $\|V_i\|_{\psi_1} \lesssim \|Y\|_{\psi_2}\cdot\|(1-Z_j)X_j + \nu Z_j\beta_j\|_{\psi_2} \lesssim (1 + \zeta)^3\sigma$, where $\zeta := (1+\omega)\rho$. By concentration result of sub-exponentials (Lemma \ref{lem:sub-exponential_sum}) and applying union bound, we have that there exists constant $C$ such that for $t \lesssim (1+\zeta)^3\sigma$,
\[
\Pr(U_1 \geq t) \leq pe\cdot\exp(-\frac{Cnt^2}{(1+\zeta)^6\sigma^2}).
\]
Setting the right hand side to be $\delta/2$ implies that for $n \gtrsim \log p + \log(2e/\delta)$,
\begin{equation} \label{eq:mcr:u_1}
U_1 \lesssim (1+\zeta)^3\sigma\sqrt{\frac{\log p + \log(2e/\delta)}{n}}
\end{equation}
with probability at least $1 - \delta/2$.

\noindent {\bf Term $U_2$.} Term $U_2$ can be further decomposed into several terms as follows
\[
U_2 \leq \|\ab_1\|_{\infty} + \|\ab_2\|_{\infty} + \|\ab_3\|_{\infty} + \|\ab_4\|_{\infty} + \sigma^{-2}\|\ab_5\|_{\infty} + \|\ab_6\|_{\infty},
\]
where 
\begin{align*}
& \ab_1 = \frac{1}{n}\sum_{i=1}^n \big\langle (\bm{1} - \zb_i)\odot \xb_i, \bbeta^*\big\rangle(\bm{1} - \zb_i)\odot\xb_i - \mathbb{E}\left[\big\langle (\bm{1} - Z)\odot X,\bbeta^*\big\rangle (\bm{1} - Z)\odot X\right], \\ 
& \ab_2 = \frac{1}{n}\sum_{i=1}^n \big\langle \nu\zb_i\odot\bbeta, \bbeta^*\big\rangle(\bm{1} - \zb_i)\odot\xb_i - \mathbb{E}\left[\big\langle \nu Z\odot\bbeta,\bbeta^*\big\rangle (\bm{1} - Z)\odot X\right], \\
& \ab_3 = \frac{1}{n}\sum_{i=1}^n \big\langle (\bm{1} - \zb_i)\odot \xb_i, \bbeta^*\big\rangle\nu \zb_i\odot\bbeta - \mathbb{E}\left[\big\langle (\bm{1} - Z)\odot X,\bbeta^*\big\rangle \nu Z\odot \bbeta\right],  \\
& \ab_4 = \frac{1}{n}\sum_{i=1}^n \nu^2\langle \zb_i\odot\bbeta, \bbeta^*\rangle \zb_i\odot\bbeta- \mathbb{E}\left[\nu^2\langle Z\odot \bbeta, \bbeta^*\rangle Z\odot\bbeta\right], \\
& \ab_5 = \frac{1}{n}\sum_{i=1}^n\big\langle\zb_i\odot\bbeta,\bbeta^* \big\rangle \zb_i\odot\bbeta  - \mathbb{E}\left[\big\langle Z\odot\bbeta,\bbeta^* \big\rangle Z\odot\bbeta\right],~ \ab_6 = \frac{1}{n}\sum_{i=1}^n {\rm diag}(\zb_i)\bbeta^* - \epsilon\bbeta^*.
\end{align*}
The key idea to bound the infinite norm of each term $\ab_i$ is the same: showing that each coordinate is finite summation of independent sub-Gaussian (or sub-exponential) random variables and applying concentration result and probabilistic union bound. For each term $\ab_i,i=1,2,\ldots,6$, we have that for any $j \in [p]$,
\begin{align*}
& \|\big\langle (\bm{1} - Z)\odot X,\bbeta^*\big\rangle (1 - Z_j)\odot X_j\|_{\psi_1} \lesssim \|\bbeta^*\|_2,~ \|\big\langle \nu Z\odot\bbeta,\bbeta^*\big\rangle (1 - Z_j)\odot X_j \|_{\psi_1} \lesssim \sigma(1+\zeta)\zeta^2,\\
& \|\big\langle (\bm{1} - Z)\odot X,\bbeta^*\big\rangle \nu Z_j \beta_j\|_{\psi_1} \lesssim \sigma(1+\zeta)\zeta^2,~\|\nu^2\langle Z\odot\bbeta,\bbeta^*\rangle Z_j\beta_j\|_{\psi_1} \lesssim \sigma(1+\zeta^2)\zeta^3,  \\
& \sigma^{-2}\|\big\langle Z\odot\bbeta,\bbeta^* \big\rangle Z_j\odot\beta_j\|_{\psi_2} \lesssim \sigma\zeta^3,~ \|\epsilon\beta_j^*\|_{\psi_2} \lesssim \epsilon\|\bbeta^*\|_{\infty}
\end{align*}
respectively. For simplicity, we treat coordinates of every $\ab_i$ as finite sum of sub-exponentials with $\psi_1$ norm $O(\sigma(1+\zeta)^5)$. Consequently, by concentration result in Lemma \ref{lem:sub-exponential_sum}, there exists constant $C$ such that
\[
\Pr(U_2 \geq t) \leq 12p\cdot\exp\left(-\frac{Cnt^2}{\sigma^2(1+\zeta)^{10}}\right)
\]
for $t \lesssim \sigma(1+\zeta)^5$. By setting the right hand side to be $\delta/2$ in the above inequality, we have that when $n \gtrsim \log p + \log(24/\delta)$,
\begin{equation} \label{eq:mcr:u_2}
U_2 \lesssim \sigma(1+\zeta)^5\sqrt{\frac{\log p + \log(24/\delta)}{n}}.
\end{equation}
with probability at least $1 - \delta/2$.

Finally, putting \eqref{eq:mcr:u_1} and \eqref{eq:mcr:u_2} together completes the proof.

\section{Supporting Lemmas}
\begin{lemma} \label{lem:sub-Gaussian_sum}
Suppose $X_1,X_2,\dots,X_n$ are $n$ i.i.d. centered sub-Gaussian random variables with Orlicz norm $\|X_1\|_{\psi_2} \leq K$. Then for every $t \geq 0$, we have
\[
\Pr\left(\bigg|\frac{1}{n}\sum_{i=1}^{n}X_i\bigg| \geq t\right) \leq e\cdot\exp\left(-\frac{Cnt^2}{K^2}\right),
\]
where $C$ is an absolute constant.
\end{lemma}
\begin{proof}
	See the proof of Proposition 5.10 in \citet{vershynin2010introduction}.
\end{proof}

\begin{lemma} \label{lem:sub-exponential_sum}
Suppose $X_1,X_2,\ldots,X_n$ are $n$ i.i.d. centered sub-exponential random variables with Orlicz norm $\|X_1\|_{\psi_1} \leq K$. Then for every $t > 0$, we have
\[
\Pr\left(\bigg|\frac{1}{n}\sum_{i=1}^{n}X_i\bigg| \geq t\right) \leq 2\cdot\exp\left(-C\min\left\{\frac{t^2}{K^2}, \frac{t}{K}\right\}n\right),
\]
where $C$ is an absolute constant.
\end{lemma}
\begin{proof}
	See the proof of Corollary 5.7 in \citet{vershynin2010introduction}.
\end{proof}

\begin{lemma} \label{lem:sub-Gaussian_centering}
Let $X$ be sub-Gaussian random variable and $Y$ be sub-exponential random variable. Then $X - \mathbb{E}[X]$ is also sub-Gaussian; $Y - \mathbb{E}[Y]$ is also sub-exponential. Moreover, we have
\[
\|X - \mathbb{E}[X]\|_{\psi_2} \leq 2\|X\|_{\psi_2}, \;\; \|Y - \mathbb{E}[Y]\|_{\psi_1} \leq 2\|Y\|_{\psi_1}.
\]
\end{lemma}
\begin{proof}
	See Remark 5.18 in \citet{vershynin2010introduction}.
\end{proof}

\begin{lemma} \label{lem:sub-Gaussian_product}
Let $X,Y$ be two sub-Gaussian random variables. Then $Z = X\cdot Y$ is sub-exponential random variable. Moreover, there exits constant $C$ such that
\[
\|Z\|_{\psi_1} \leq C\|X\|_{\psi_2}\cdot\|Y\|_{\psi_2}.
\]
\end{lemma}
\begin{proof}
It follows from the basic properties. We omit the details.
\end{proof}

\begin{lemma} \label{lem:RE}
Let matrix $\Xb$ be an $n$-by-$p$ random matrix with i.i.d. rows drawn from $X$, which is zero mean sub-Gaussian random vector with $\|X\|_{\psi_2} \leq K$ and covariance matrix $\Sigmab$. We let $\lambda_1 := \lambda_{min}(\Sigmab), \lambda_p := \lambda_{max}(\Sigmab)$. 

\noindent (1) There exist constants $C_i$ such that
\[
\frac{1}{n}\|\Xb\ub\|_2^2 \geq \frac{\lambda_1}{2}\|\ub\|_2^2 - C_0\lambda_1\max\left\{\frac{K^4}{\lambda_1^2},1\right\}\frac{\log p}{n}\|\ub\|_1^2, \;\text{for all}\; \ub \in \RR^p, 
\]
with probability at least $1 - C_1\exp\left(-C_2n\min\left\{\frac{\lambda_1^2}{K^4},1\right\}\right)$.

\noindent (2) In Parallel, there exist constants $C_i'$ such that
\[
\frac{1}{n}\|\Xb\ub\|_2^2 \leq \frac{3\lambda_p}{2}\|\ub\|_2^2 + C_0'\lambda_p\max\left\{\frac{K^4}{\lambda_p^2},1\right\}\frac{\log p}{n}\|\ub\|_1^2, \;\text{for all}\; \ub \in \RR^p, 
\]
with probability at least $1 - C_1'\exp\left(-C_2'n\min\left\{\frac{\lambda_p^2}{K^4},1\right\}\right)$.
\end{lemma}
\begin{proof} It follows by putting Lemma 12 and Lemma 15 in \citet{loh2011high} together.
\end{proof}

\begin{lemma} \label{lem:angleExpectation}
Let $X_1$ and $X_2$ be independent random variables with distribution $\cN(0,1)$. For any positive constant $C > 0$, let event $\cE := \{C\cdot|X_2| \geq |X_1|\}$. Then we have

\noindent(a)
\[
\mathbb{E}\left[|X_1| ~\big|~\cE\right]\cdot \Pr(\cE) = \sqrt{\frac{2}{\pi}}\left[1 - \sqrt{\frac{1}{C^2+1}}\right].
\] 

\noindent(b)
\[
\mathbb{E}\left[|X_2| ~\big|~\cE\right]\cdot \Pr(\cE) = \sqrt{\frac{2}{\pi}}\frac{C}{\sqrt{1 + C^2}}.
\]

\noindent(c)
\[
\mathbb{E}\left[|X_1X_2|~\big|~ \cE\right]\cdot \Pr(\cE) = \frac{2C^2}{\pi(1+C^2)}.
\]
	
\end{lemma}
\begin{proof}
(a)
\begin{align*}
& \mathbb{E}\left[|X_1| ~\big|~\cE\right] \cdot\Pr(\cE) =  4\cdot \int_{0}^{\infty}\int_{0}^{uC}\frac{1}{2\pi}\exp(-\frac{1}{2}v^2)\exp(-\frac{u^2}{2})v\ddd v\ddd u = \sqrt{\frac{2}{\pi}}\left[1 - \sqrt{\frac{1}{C^2+1}}\right].
\end{align*}
(b)
\begin{align*}
& \mathbb{E}\left[|X_2|~\big|~\cE\right] \cdot\Pr(\cE) =  4\cdot \int_{0}^{\infty}\int_{v/C}^{\infty}\frac{1}{2\pi}\exp(-\frac{1}{2}v^2)\exp(-\frac{u^2}{2})u\ddd u\ddd v = \sqrt{\frac{2}{\pi}}\frac{C}{\sqrt{1 + C^2}}.
\end{align*}
(c)
\begin{align*}
& \mathbb{E}\left[|X_1X_2|~\big|~  \cE\right]\cdot \Pr( \cE)  = 4\cdot \int_{0}^{\infty}\int_{v/C}^{\infty}\frac{1}{2\pi}\exp(-\frac{u^2}{2})\exp(-\frac{v^2}{2})uv\ddd u\ddd v \\
& = \frac{2}{\pi}\int_{0}^{\infty}\exp(-\frac{C^2+1}{2}v^2)vdv  = \frac{2C^2}{\pi(1+C^2)}.	
\end{align*}
\end{proof}

\begin{lemma} \label{lem:key_expectation}
	Let $X \sim \cN(0,\sigma^2)$ and $Z$ be Rademacher random variable taking values in $\{-1,1\}$. Moreover, $X$ and $Z$ are independent. Function $f(x,z;a,\gamma)$ is defined as
	\[
	f(x,z;a,\gamma) = \frac{x+az}{1 + \exp(-\frac{2(1+\gamma)}{\sigma^2}a(x+az))}.
	\]
	Then for any $a \in \RR, \gamma \in \RR$, we have 
	\[
	 \left|\mathbb{E}\left[f(X,Z;a,\gamma)\right] - \frac{a}{2} \right| \leq \min\bigg\{ \frac{1}{2}|a\gamma|\exp(\frac{\gamma^2a^2-a^2}{2\sigma^2}),~\frac{\sigma}{\sqrt{2\pi}} + |a|\bigg\}.
	\]
	In the special case $\gamma = 0$, we have $\mathbb{E}\left[f(X,Z;a,\gamma)\right] = a/2$.
\end{lemma}
\begin{proof}
First note that
\begin{align*}
	\mathbb{E}\left[f(X,Z;a,\gamma)\right] & = \frac{1}{2}\mathbb{E}\left[\frac{X+a}{1 + \exp(-\frac{2(1+\gamma)}{\sigma^2}a(X+a))} + \frac{X-a}{1 + \exp(-\frac{2(1+\gamma)}{\sigma^2}a(X-a))}\right] \\
	& = \frac{1}{2}\mathbb{E}\left[\frac{X+a}{1 + \exp(-\frac{2(1+\gamma)}{\sigma^2}a(X+a))} + \frac{-X-a}{1 + \exp(-\frac{2(1+\gamma)}{\sigma^2}a(-X-a))}\right],
\end{align*}
where the first equality is from taking expectation of $Z$, the second equality is from the fact that the distribution of $X$ is symmetric around $0$. Let $X' = X + a$, then we have
\begin{align*}
\mathbb{E}\left[f(X,Z;a,\gamma)\right] & = \frac{1}{2}\mathbb{E}\left[\frac{X'}{1+\exp(-\frac{2(1+\gamma)}{\sigma^2}aX')} + \frac{-X'}{1+\exp(\frac{2(1+\gamma)}{\sigma^2}aX')} \right] \\ & = \frac{1}{2}\mathbb{E}\left[X' - 2\frac{\exp(-\frac{2(1+\gamma)}{\sigma^2}aX')X'}{1 + \exp(-\frac{2(1+\gamma)}{\sigma^2}aX')}\right].
\end{align*}
Using $\mathbb{E}\left[X'\right] = a$, we have
\begin{align} \label{eq:tmp8}
& \mathbb{E}\left[f(X,Z;a,\gamma)\right] - a/2  = \mathbb{E}\left[-\frac{\exp(-\frac{2(1+\gamma)}{\sigma^2}aX')X'}{1 + \exp(-\frac{2(1+\gamma)}{\sigma^2}aX')}\right] \notag\\
& = \int_{-\infty}^{\infty} \frac{\exp(-\frac{(x-a)^2}{2\sigma^2})}{\sqrt{2\pi}\sigma}\frac{-\exp(-\frac{2(1+\gamma)}{\sigma^2}ax)x}{1 + \exp(-\frac{2(1+\gamma)}{\sigma^2}ax)}\ddd x  = \int_{-\infty}^{\infty}\frac{\exp(-\frac{x^2 + a^2}{2\sigma^2})x}{\sqrt{2\pi}\sigma} \frac{-\exp(-\frac{\gamma ax}{\sigma^2})}{\exp(\frac{a(1+\gamma)x}{\sigma^2}) + \exp(\frac{-a(1+\gamma)x}{\sigma^2})}\ddd x  \notag\\
& = \int_{0}^{\infty} \frac{\exp(-\frac{x^2 + a^2}{2\sigma^2})x}{\sqrt{2\pi}\sigma}\frac{\exp(\frac{\gamma ax}{\sigma^2})-\exp(-\frac{\gamma ax}{\sigma^2})}{\exp(\frac{a(1+\gamma)x}{\sigma^2}) + \exp(\frac{-a(1+\gamma)x}{\sigma^2})}\ddd x
\end{align}
When $a\gamma \geq 0$, we have $\mathbb{E}\left[f(X,Z;a,\gamma)\right] - a/2 \geq 0$. Under this setting,  \eqref{eq:tmp8} yields that
\begin{align*}
& \mathbb{E}\left[f(X,Z;a,\gamma)\right] - a/2  \leq  \int_{0}^{\infty} \frac{\exp(-\frac{x^2 + a^2}{2\sigma^2})x}{2\sqrt{2\pi}\sigma}\left[\exp(\frac{\gamma ax}{\sigma^2}) - \exp(-\frac{\gamma ax}{\sigma^2})\right]\ddd x \\
& = \frac{1}{2}\exp(\frac{\gamma^2a^2-a^2}{2\sigma^2}) \int_{0}^{\infty} \frac{1}{\sqrt{2\pi}\sigma}\left[\exp\left(-\frac{(x-\gamma a)^2}{2\sigma^2}\right)- \exp\left(-\frac{(x+\gamma a)^2}{2\sigma^2}\right)\right]x\ddd x \\
& = \frac{1}{2}\exp(\frac{\gamma^2a^2-a^2}{2\sigma^2}) \int_{-\infty}^{\infty} \frac{1}{\sqrt{2\pi}\sigma}\exp\left(-\frac{(x-\gamma a)^2}{2\sigma^2}\right)x\ddd x = \frac{1}{2}\exp(\frac{\gamma^2a^2-a^2}{2\sigma^2})\gamma a,
\end{align*}
where the first inequality follows from the fact that $x + 1/x \geq 2$ for any $x > 0$, the second equality is from 
\[
-\int_{0}^{\infty} \exp\left(-\frac{(x+\gamma a)^2}{2\sigma^2}\right)x\ddd x =  \int_{-\infty}^{0} \exp\left(-\frac{(x-\gamma a)^2}{2\sigma^2}\right) x\ddd x.
\]
When $a\gamma \leq 0$, using similar proof, we have $\frac{1}{2}\exp(\frac{\gamma^2a^2-a^2}{2\sigma^2})\gamma a \leq \mathbb{E}\left[f(X,Z;a,\gamma)\right] - a/2 \leq 0$. Combining the two cases, we prove that
\begin{equation} \label{eq:tmp12}
\big|\mathbb{E}\left[f(X,Z;a,\gamma)\right] - a/2\big| \leq \frac{1}{2}|a\gamma|\exp(\frac{\gamma^2a^2-a^2}{2\sigma^2}).
\end{equation}
In the special case when $\gamma = 0$, we thus have $\mathbb{E}(f(X,Z;a,\gamma)) = a/2$.

Note that when $a\gamma \geq 0$, \eqref{eq:tmp8} also implies that
\begin{align*}
& \mathbb{E}\left[f(X,Z;a,\gamma)\right] - a/2  \leq  \int_{0}^{\infty} \frac{\exp(-\frac{x^2 + a^2}{2\sigma^2})x}{\sqrt{2\pi}\sigma}\frac{\exp(\frac{\gamma ax}{\sigma^2})}{\exp(\frac{a(1+\gamma)x}{\sigma^2})}\ddd x = \int_{0}^{\infty} \frac{\exp(-\frac{(x+a)^2}{2\sigma^2})x}{\sqrt{2\pi}\sigma}\ddd x \\
& = \int_{0}^{\infty} \frac{\exp(-\frac{(x+a)^2}{2\sigma^2})(x+a)}{\sqrt{2\pi}\sigma}\ddd x -  \int_{0}^{\infty} \frac{\exp(-\frac{(x+a)^2}{2\sigma^2})a}{\sqrt{2\pi}\sigma}\ddd x \leq \frac{\sigma}{\sqrt{2\pi}} + |a|.
\end{align*}
Similarly, when $a\gamma \leq 0$, we have
\begin{align*}
& \mathbb{E}\left[f(X,Z;a,\gamma)\right] - a/2  \geq  \int_{0}^{\infty} \frac{\exp(-\frac{x^2 + a^2}{2\sigma^2})x}{\sqrt{2\pi}\sigma}\frac{-\exp(\frac{-\gamma ax}{\sigma^2})}{\exp(\frac{-a(1+\gamma)x}{\sigma^2})}\ddd x = -\int_{0}^{\infty} \frac{\exp(-\frac{(x-a)^2}{2\sigma^2})x}{\sqrt{2\pi}\sigma}\ddd x \\
& = -\int_{0}^{\infty} \frac{\exp(-\frac{(x-a)^2}{2\sigma^2})(x-a)}{\sqrt{2\pi}\sigma}\ddd x -  \int_{0}^{\infty} \frac{\exp(-\frac{(x-a)^2}{2\sigma^2})a}{\sqrt{2\pi}\sigma}\ddd x \geq -\frac{\sigma}{\sqrt{2\pi}} - |a|.
\end{align*}
Therefore, we have that 
\begin{equation} \label{eq:tmp16}
\big|\mathbb{E}\left[f(X,Z;a,\gamma)\right] - a/2\big| \leq \frac{\sigma}{\sqrt{2\pi}} + |a|.
\end{equation}

Putting \eqref{eq:tmp12} and \eqref{eq:tmp16} together completes the proof.
\end{proof}
\end{document}